%% file: main.tex
\documentclass[11pt,twoside]{article}

\usepackage{fancyhdr}
\usepackage[colorlinks,citecolor=blue,urlcolor=blue,linkcolor=blue,bookmarks=false]{hyperref}
\usepackage{amsfonts,epsfig,graphicx}
\usepackage{afterpage}
\usepackage{amsmath,amssymb,amsthm} 
\usepackage{fullpage}
\usepackage[T1]{fontenc} 
\usepackage{graphicx} 
\usepackage{epsf} 
\usepackage{amsfonts,amsmath}
\usepackage[sort,numbers]{natbib} 
\usepackage{psfrag,xspace}
\usepackage{color,etoolbox}

\usepackage{url}
\usepackage[inline]{enumitem}
\usepackage{graphicx} 
\usepackage{pdfpages}
\usepackage{subfigure} 
\usepackage{verbatim}
\usepackage{booktabs}

\usepackage{algorithm}
\usepackage{algpseudocode}
\usepackage{wrapfig}
\usepackage{lipsum}

\usepackage{amsthm}
\usepackage{mysymbols}
\input{macros}

\newcommand{\true}{\theta^*}
\newcommand{\risk}{\mathcal{R}}
\newcommand{\samplerisk}{\mathcal{R}_n}
\newcommand{\graddist}{P_g^{\theta}}

\newcommand{\thetastar}{\theta^*}
\newcommand{\dataset}{\mathcal{D}_n}

\newcommand{\releff}{\text{\sc RelEff}}

\newcommand{\gmom}{\text{\sc Gmom}}

\newcommand{\poploss}{\overline{\mathcal{L}}}
\newcommand{\inprod}[2]{\ensuremath{\langle {#1}, \; {#2}
\rangle}}
\newcommand{\smconst}{\gamma}
\newcommand{\truesamp}{S_P}
\newcommand{\ngood}{n_P}

\newcommand{\tiln}{\widetilde{n}}
\newcommand{\tilp}{\widetilde{p}}
\newcommand{\widetildelta}{\widetilde{\delta}}
\newcommand{\tildelta}{\widetilde{\delta}}
\newcommand{\link}{\Phi}

\newcommand{\indic}[1]{ \mathbb{I} \tlprn {#1}}
\newcommand{\winner}{{\widehat{\theta}_{j^*}}}

\usepackage{csquotes}

\begin{document}

\begin{center} 
{\LARGE{\bf{Robust Estimation via Robust Gradient Estimation}}}

\vspace*{.3in}

{\large{
\begin{tabular}{ccccc}
Adarsh Prasad$^\ddagger$~~~Arun Sai Suggala$^\ddagger$~~~Sivaraman Balakrishnan$^\dagger$~~~Pradeep Ravikumar$^\ddagger$ \\
\end{tabular}

\vspace*{.1in}

\begin{tabular}{ccc}
Machine Learning Department$^{\ddagger}$ \\
Department of Statistics$^{\dagger}$ \\
\end{tabular}

\begin{tabular}{c}
Carnegie Mellon University, \\
Pittsburgh, PA 15213.
\end{tabular}

\vspace*{.2in}

}}

\begin{abstract}
We provide a new computationally-efficient class of estimators for risk minimization. We show that these estimators are robust for general statistical models, under varied robustness settings, including in the classical Huber $\epsilon$-contamination model, and in heavy-tailed settings. Our workhorse is a novel robust variant of gradient descent, and we provide conditions under which our gradient descent variant provides accurate estimators in a general convex risk minimization problem. We provide specific consequences of our theory for linear regression, logistic regression and for canonical parameter estimation in an exponential family. These results provide some of the first computationally tractable and provably robust estimators for these canonical statistical models. Finally, we study the empirical performance of our proposed methods on synthetic and real datasets, and find that our methods convincingly outperform a variety of baselines. 
\end{abstract}
\end{center}

\input{introduction}
\input{setup}
\input{gradientOracles}
\input{exps}
\input{theory}

\input{discussion}

\input{ack}

\begin{small}\bibliography{local}
\bibliographystyle{plainnat}
\end{small}

\appendix
\input{appendix}

\clearpage
\end{document}

%% file: macros.tex
\newcommand\emu{\widehat{\mu}}

\newcommand\len[1]{\text{length} \left( #1 \right)}

\newcommand{\op}{\text{op}}
\newcommand{\tilS}{\widetilde{S}}

\def \bal#1\eal{\begin{align}#1\end{align}}
\def \bals#1\eals{\begin{align*}#1\end{align*}}

\def\tparam{\theta^{*}}

\def\eparam{\widehat{\theta}}

\def\grad{\nabla}

\def\st{\text{s.t.}}
\def\real{\mathbb{R}}
\def\E{\mathbb{E}}
\def\P{\mathbb{P}}

\def \poprisk{R}
\def\stepSize{{\eta}}

%% file: introduction.tex
\section{Introduction}
In classical analyses of statistical estimators, statistical guarantees are derived under strong model assumptions, and in most cases these guarantees hold only in the absence of arbitrary outliers, and other deviations from the model assumptions. 
Strong model assumptions are rarely met in practice, and this has motivated the development of robust inferential procedures, and which has a rich history in statistics with seminal contributions due to Box \cite{box53}, Tukey \cite{tukey1975mathematics}, Huber \cite{Huber81}, Hampel \cite{hampel2011robust} and several others. These have led to rich statistical concepts such as the influence function, the breakdown point, and the Huber $\epsilon$-contamination 
model, to assess the robustness of estimators. Despite this progress however, 
the statistical methods with the strongest robustness guarantees are computationally intractable, for instance those 
based on non-convex $M$-estimators \cite{Huber81}, $\ell_1$ tournaments \cite{devroye1985nonparametric,yatracos1985,donoho1988automatic} and
notions of depth \cite{mizera2002depth,gao2017robust,chen2015robust}. 

In this paper, we present a general class of estimators that are computationally tractable, and have strong robustness guarantees. The estimators we propose are obtained by robustifying iterative updates of risk minimization, and are broadly applicable to a wide-range of parametric statistical models. In the risk minimization framework, the target parameter $\thetastar$ is defined as the solution to an optimization problem:
\begin{align}
\label{eqn:risk_intro}
\thetastar =  \argmin_{\theta \in \Theta} \mathcal{R}(\theta) \equiv  \argmin_{\theta \in \Theta} \mathbb{E}_{z \sim P}\left[\poploss(\theta; z)\right],
\end{align}
where $\poploss$ is an appropriate loss-function, $R$ is the population risk and $\Theta$ is the set of feasible parameters. The statistical inference problem within the risk minimization framework is then to compute an approximate minimizer to the above program when given access to samples $\dataset = \{z_1,\ldots,z_n\}$. A classical approach to do so is via \emph{empirical risk minimization} (ERM), where we substitute the empirical expectation given the samples for the population expectation in the specification of the risk objective. While most modern statistical estimators use the above empirical risk minimization framework, a standard assumption that is imposed on $\dataset$ is that the data has no outliers, and has no arbitrary deviations from model assumptions; i.e., it is typically assumed that each of the $z_i$'s are independent and identically distributed according to the distribution $P$. Moreover, many analyses of risk minimization further assume that $P$
follows a sub-gaussian distribution, or has otherwise well-controlled tails in order to appropriately control the deviation between the population risk and its empirical counterpart. Due in part to these caveats with ERM, the seminal work of $M$-estimation replaces the risk minimization objective with a robust counterpart, so that the minimizer of the empirical expectation of the robust counterpart is more robust than the ERM minimizer. As noted above, for strong statistical guarantees, these in turn require solving computationally intractable non-convex optimization programs.

In contrast to this classical work, we propose a class of estimators that have a shift in perspective: rather than specify a robust objective, we consider an algorithm, namely projected gradient descent, that directly optimizes the population risk objective in Eq. \eqref{eqn:risk_intro}, and focus on making this algorithm robust. Thus, in contrast to specifying the robust parameter estimate as the solution to an optimization program as in $M$-estimation, which in turn could be computationally intractable, we specify the robust parameter estimate as the limit of a sequence of iterative updates that are individually robust as well as computationally tractable. We find that this shift in perspective leads to estimators that are both computationally tractable as well with strong robustness guarantees, that are as broadly applicable as ERM or $M$-estimators, and moreover with a unified statistical treatment for varied statistical models.


In addition to being applicable to a variety of statistical models, our general results are also applicable to a variety of \emph{notions of robustness}. In this paper, we derive corollaries in particular for two canonical robustness settings:
\begin{enumerate}
\item[(a)] {\bf Robustness to arbitrary outliers: } In this setting, we focus on Huber's $\epsilon$-contamination model, where rather than observe samples directly from $P$ in~\eqref{eqn:risk_intro} we instead observe samples drawn from $P_{\epsilon}$ which for an \emph{arbitrary} distribution $Q$ is defined as:
\begin{align*}
P_{\epsilon} = (1-\epsilon) P + \epsilon Q. 
\end{align*}
The distribution $Q$ allows for arbitrary outliers, which may correspond to gross corruptions or more subtle deviations from the assumed model. This model can be equivalently viewed as model mis-specfication in the Total Variation (TV) metric. 
\item[(b)] {\bf Robustness to heavy-tails: } In this setting, we are interested in developing estimators under weak moment assumptions. We assume that the distribution $P$ from which we obtain samples only has finite low-order moments (see Section~\ref{sec:heavy} for a precise characterization). Such heavy tailed distributions arise frequently in the analysis of financial data and large-scale biological datasets (see for instance examples in \cite{fan16,zhou2017}). 
In contrast to classical analyses of empirical risk minimization \cite{VdG-book}, in this setting the empirical risk is not uniformly close to the population risk, and methods that directly minimize the empirical risk perform poorly (see Section~\ref{sec:experiments}).


\end{enumerate}
While we provide corollaries demonstrating robustness with respect to the above deviations, we emphasize that our framework is more general. Below, we provide an outline of our results and contributions.

\begin{enumerate}
\item \textbf{Estimators.} Our first contribution is to introduce a new class of robust estimators for risk minimization~\eqref{eqn:risk_intro}. These estimators are based on robustly estimating gradients of the population risk to then plug in to a projected gradient descent algorithm, and are computationally tractable by design. A crucial ingredient of our framework is the design of robust gradient estimators for the population risk in~\eqref{eqn:risk_intro}. Our main insight is that in this general risk minimization setting, the gradient of the population risk is simply a multivariate mean vector, and we can leverage prior work on mean estimation to design robust gradient estimators. Thus, for our two canonical robustness cases, we develop such robust gradient estimators building on prior work for robust mean estimation in the Huber model \cite{lai2016agnostic}, and in the heavy-tailed model \cite{minsker2015geometric}. Another perspective of our framework is that it significantly generalizes the applicability of \emph{mean estimation} methods to general parametric models.
\item \textbf{Empirical Investigations.} Our estimators are computationally practical, and accordingly, our second contribution is to conduct extensive numerical experiments on real and simulated data with our proposed class of estimators. We provide guidelines for tuning parameter selection, and compare the proposed estimators with several competitive baselines~\cite{bhatia2015robust,ransac81,Huber81}. 
We find that our estimators consistently perform well across different settings, and across various metrics.
\item \textbf{Statistical Guarantees.} Finally, we provide rigorous robustness guarantees for the estimators we propose for a variety of classical statistical models: linear regression, logistic regression, and exponential family models. Our contributions in this direction are two-fold: building on prior work~\cite{balakrishnan2017statistical} we provide a general result on the stability of gradient descent for risk minimization, and show that under certain conditions, gradient descent can be quite tolerant to inaccurate gradient estimates. Subsequently, in concrete settings, we provide a careful analysis of the quality of gradient estimation afforded by our proposed gradient estimators, and combine these results to obtain guarantees on our final parameter estimates. 
\end{enumerate}
Broadly, as we discuss in the sequel, our work suggests that our class of estimators based on robust gradient estimation offer a variety of practical, conceptual, statistical and computational advantages for robust estimation. They provide the general applicability of classical $M$-estimators, together with computational practicality  even for large-scale models, as well as  strong robustness guarantees.


\subsection{Related Work}
There has been extensive work in the broad area of robust statistics (see for instance \citep{hampel2011robust} and references therein); we focus this section on some lines of work that are most related to this paper. For the robustness setting of $\epsilon$-contaminated models, several classical estimators have been developed that are optimally robust for a variety of inferential tasks, including hypothesis testing \cite{huber1965robust}, mean estimation \cite{mizera2002depth}, general parametric estimation \cite{chen2016general,donoho1988automatic,yatracos1985}, and non-parametric estimation \cite{devroye1985nonparametric}. However, a major drawback with this classical line of work has been that most of the estimators with strong robustness guarantees are computationally intractable~\citep{tukey1975mathematics}, while the remaining ones use heuristics and are consequently not optimal \citep{hastings1947low}. A complementary line of recent research \citep{chen2015robust,gao2017robust} has focused on providing minimax upper and lower bounds on the performance of estimators under $\epsilon$-contamination model, without the constraint of computational tractability. Recently, there has been a flurry of research in theoretical computer science~\citep{diakonikolas2016robust,lai2016agnostic,du2017computationally,li17robust,charikar2016learning} designing provably robust estimators which are computationally tractable while achieving near-optimal contamination dependence, for special classes of problems such as computing means and covariances. Some of the proposed algorithms are however not computationally practical as they rely on the ellipsoid algorithm or require solving semi-definite programs \citep{diakonikolas2016robust,du2017computationally,li17robust,charikar2016learning} which can be slow for modern problem sizes. 

While in the general $\epsilon$-contamination setting, the contamination distribution could be arbitrary, there has been a lot of work in settings where the contamination distribution is restricted in various ways. For example, recent work in high-dimensional statistics (for instance \citep{candes2011robust,loh2011high, yi2016,loh2017,chen2013}) have studied problems like principal component analysis and linear regression under the assumption that the corruptions are evenly spread throughout the dataset.  

For the robustness setting of heavy tailed distributions, robust estimators aim to relax the sub-gaussian or sub-exponential distributional assumptions that are typically imposed on the target distribution, and allow it to be a heavy tailed distribution. Most approaches in this category substitute the \emph{empirical mean} of the risk objective in risk minimization with robust mean estimators such as \cite{lerasle2011robust,catoni2012challenging} that exhibit sub-gaussian type concentration around the true mean for distributions satisfying mild moment assumptions. The median-of-means estimator \cite{lerasle2011robust} and Catoni's mean estimator \cite{catoni2012challenging} are two popular examples of such robust mean estimators. In particular, \citet{hsu2016loss} use the median-of-means estimator to solve the corresponding robust variant of ERM. Although this estimator has strong theoretical guarantees, and is computationally tractable, as noted by the authors in \cite{hsu2016loss} it performs poorly in practice. In recent work \citet{brownlees2015empirical} use the Catoni's mean estimator to solve the corresponding robust variant of ERM. The authors provide risk bounds similar to bounds one can achieve under sub-gaussian distributional assumptions. However, their estimator is not easily computable and the authors do not provide a practical algorithm to compute the estimator. Other recent work by \citet{lerasle2011robust, lugosi2016risk} use similar ideas to derive estimators that perform well theoretically, in heavy-tailed situations. However, these approaches involve optimization of complex objectives for which no computationally tractable algorithms exist. We emphasize that in contrast to our work, these works focus on robustly estimating the population risk which does not directly lead to a computable estimator. In contrast, we consider robustly estimating the \emph{gradient} of the population risk, and embedding these estimates within the iterative algorithm of projected gradient descent, which leads naturally to a computionally practical estimator. 

\subsection{Outline}
We conclude this section with a brief outline of the remainder of the paper. In Section~\ref{sec:background}, we provide some background on risk minimization and the running robustness settings of Huber contamination, and heavy-tailed noise models. 
In Section~\ref{sec:gradientOracles}, we introduce our class of robust estimators, and provide concrete 
algorithms for the $\epsilon$-contaminated and heavy-tailed settings. In Section~\ref{sec:experiments} we study the empirical performance of our estimator on a variety of tasks and datasets.  We complement our empirical results with theoretical guarantees in Sections~\ref{sec:theory_prelims},~\ref{sec:consequences_huber} and~\ref{sec:consequences_heavy}. We defer technical details to the Appendix. Finally, we conclude in Section~\ref{sec:discussion} with a discussion of some open problems.



%% file: setup.tex

\section{Background and Problem Setup}
\label{sec:background}
In this section we provide the necessary background on risk minimization, gradient descent, and introduce two notions of robustness that we consider in this work.
\subsection{Risk Minimization and Parametric Estimation}
In the setting of risk minimization, we assume that we have access to a differentiable loss function $\poploss: \Theta \times \mathcal{Z} \mapsto \mathbb{R}$, where $\Theta$ is a convex subset of $\real^p$. Let $\mathcal{R}(\theta) = \mathbb{E}_{z \sim P} \left[ \poploss(\theta; z) \right]$ be the population loss, also known as the \emph{risk}, and let $\tparam$ be the minimizer of the population risk $\mathcal{R}(\theta)$, over the set $\Theta$:
\begin{align}
\label{eqn:risk}
\true = \argmin_{\theta \in \Theta} \mathcal{R}(\theta).
\end{align}  
The goal of risk minimization is to minimize the population risk $\mathcal{R}(\theta)$, given only $n$ samples $\dataset = \{z_i\}_{i = 1}^n$, in order to estimate the unknown parameter $\tparam$. 

In this work we assume that the population risk is convex to ensure tractable minimization. Moreover, in order to ensure identifiability of the parameter $\true$, we impose two standard regularity conditions  \cite{bubeck2015convex} on the population risk. These properties are defined in terms of the error of the first-order Taylor approximation of the population risk, i.e. defining, $\tau(\theta_1,\theta_2) := \risk(\theta_1) - \risk(\theta_2) - \inprod{\nabla \risk(\theta_2)}{\theta_1 - \theta_2}$, we assume that
\begin{align}
\label{eqn:sc}
\frac{\tau_\ell}{2} \| \theta_1 - \theta_2\|_2^2 \leq \tau(\theta_1,\theta_2) \leq \frac{\tau_u}{2} \|\theta_1 - \theta_2\|_2^2,
\end{align}
where the parameters $\tau_\ell, \tau_u > 0$ denote the strong-convexity and smoothness parameters respectively.

%
%

\subsection{Illustrative Examples of Risk Minimization}
\label{sec:illex}
The framework of risk minimization is a central paradigm of statistical estimation and is widely applicable. In this section, we provide illustrative examples that fall under this framework. 

\subsubsection{Linear Regression} 
Here we observe paired samples $\{(x_1, y_1), \dots (x_n, y_n)\}$, where each $(x_i, y_i) \in \mathbb{R}^p \times \mathbb{R}$. We assume that the $(x,y)$ pairs sampled from the true distribution $P$ are linked via a linear model:
\begin{align}\label{eq:linregress}
y = \inprod{x}{\true} + w,
\end{align}
where $w$ is drawn from a zero-mean distribution such as normal distribution with variance $\sigma^2$ ($\calN(0,\sigma^2)$) or a more heavy-tailed distribution such as student-t or Pareto distribution. 
We suppose that under $P$ the covariates $x \in \mathbb{R}^p$, have mean 0, and covariance $\Sigma$. 

For this setting we use the squared loss as our loss function, which induces the following population risk:
\[
\poploss(\theta; (x,y)) = \frac{1}{2}\left(y - \left\langle x, \theta\right\rangle\right)^2, \quad \text{and} \quad \mathcal{R}(\theta) = \frac{1}{2}(\theta - \theta^*)^T\Sigma(\theta - \theta^*).
\]
Note that the true parameter $\theta^*$ is the minimizer of the population risk $\mathcal{R}(\theta)$. The strong-convexity and smoothness assumptions from~\eqref{eqn:sc} in this setting require that $\tau_\ell \leq \lambda_{\min}(\Sigma) \leq \lambda_{\max}(\Sigma) \leq \tau_u$.

\subsubsection{Generalized Linear Models}
Here we observe paired samples $\{(x_1, y_1), \dots (x_n, y_n)\}$, where each $(x_i, y_i) \in \mathbb{R}^p \times \calY$. We suppose that the $(x,y)$ pairs sampled from the true distribution $P$ are linked via a linear model such that when conditioned on the covariates  $x$, the response variable has the distribution:
\begin{align}\label{eqn:glm}
P(y | x) \propto \exp \left( \frac{y \inprod{x}{\true} - \link(\inprod{x}{\tparam}) }{c(\sigma)} \right)
\end{align}
Here $c(\sigma)$ is a fixed and known scale parameter and $\link : \real \mapsto \real$ is the link function. We focus on the random design setting where
the covariates $x \in \mathbb{R}^p$, have mean 0, and covariance $\Sigma$. 
We use the negative conditional log-likelihood as our loss function, i.e.
\begin{align}\label{eqn:glmfisherloss}
\poploss(\theta; (x,y)) =  - y \inprod{x}{\theta} + \link(\inprod{x}{\theta}). 
\end{align}
Once again, the true parameter $\theta^*$ is the minimizer of the resulting population risk $\mathcal{R}(\theta)$. It is easy to see that Linear Regression with Gaussian Noise lies in the family of generalized linear models. A popular instance of such  GLMs is a logistic regression model.
\paragraph{Logistic Regression.}  In this case the $(x,y)$ pairs are linked as:
\begin{align}\label{eq:logisticregress}
y = \begin{cases}
1~~~\text{with probability}~~~\frac{1}{1 + \exp ( - \inprod{x}{\true})}, \\
0~~~\text{otherwise.}
\end{cases}
\end{align}
This corresponds to setting $\link(t) = \log(1 + \exp(t))$ and $c(\gamma) = 1$ in~\eqref{eqn:glm}. The hessian of the population risk is given by
\[
\nabla^2\mathcal{R}(\theta) = \mathbb{E}\left[\frac{\exp{\left\langle x, \theta\right\rangle}}{(1+\exp{\left\langle x, \theta\right\rangle})^2}xx^T\right].
\]
Note that as $\theta$ diverges, the minimum eigenvalue of the hessian approaches $0$ and the loss is no longer strongly convex. To prevent this, in this case 
we take the parameter space $\Theta$ to be bounded.

\subsubsection{Exponential Families and Canonical Parameters.} Finally we consider the case where the true distribution $P$ is in exponential family
with canonical parameters $\true \in \mathbb{R}^p$, and a vector of sufficient statistics obtained from the map $\phi: \mathcal{Z} \mapsto \mathbb{R}^p$. Note that while the linear and logistic regression models are indeed in an exponential family, our interest in those cases was not in the canonical parameters. 

In more details, we can write the true distribution $P$ in this case as
\begin{align*}
P(z) = h(z) \exp \left( \inprod{\phi(z)}{\true} - A(\true) \right),
\end{align*}
where $h(z)$ is some base measure. The negative log-likelihood gives us the following loss function:
\begin{align}\label{eq:expfamily}
\poploss(\theta; z) = - \inprod{\phi(z)}{\theta} + A(\theta).
\end{align}
The strong-convexity and smoothness assumptions require that there are constants $\tau_\ell, \tau_u$ such that 
$\tau_\ell \leq
\nabla^2 A(\theta) \leq \tau_u$, for $\theta \in \Theta$.

\subsection{Empirical Risk Minimization}
Given data $\dataset = \{z_i\}_{i = 1}^n$, empirical risk minimization (ERM) substitutes the empirical expectation of the risk for the population risk in the risk minimization objective: 
\begin{align*}
\widehat{\theta}_n = \argmin_{\theta \in \Theta} \samplerisk(\theta) := \frac{1}{n} \sum_{i=1}^n \poploss(\theta; z_i).
\end{align*}

Most modern statistical estimators follow this ERM recipe above. When the loss is the log-likelihood of the statistical model, this reduces to the classical Maximum Likelihood Estimation (MLE) principle. The empirical risk minimizer is however a poor estimator of $\theta^*$ in the presence of outliers in the data: since ERM depends on the sample mean, outliers in the data can effect the sample mean and lead ERM to sub-optimal estimates. This observation has led to a large body of research that focuses on developing robust M-estimators, where we substitute in the empirical expectation of a robust counterpart of the loss function $\poploss$; the resulting estimators have favorable statistical properties, but are often computationally intractable.

\subsection{Projected Gradient Descent}
A popular approach for solving the empirical risk minimization problem is projected gradient descent. Projected gradient descent generates a sequence of iterates $\{\theta^t\}_{t=0}^{\infty}$, by refining an initial parameter $\theta_0 \in \Theta$ via the update:
\begin{align*}
\theta^{t+1} = \mathcal{P}_{\Theta}\left(\theta^t - \eta \nabla \samplerisk(\theta^t) \right),
\end{align*}
where $\eta > 0$ is the step size and $\mathcal{P}_{\theta}$ is the projection operator onto $\Theta$. While this gradient descent method is simple, it is not however robust to various deviations, for general convex losses. Accordingly, we have a small shift in perspective: instead of performing gradient descent on the empirical risk, we perform gradient descent on the population risk. Our work then relies on the important observation that this gradient of the population risk ($\mathbb{E}_{z \sim P}\left[\nabla \poploss(\theta;z)\right]$) is simply a mean vector: one that can be estimated robustly by leveraging recent advances in robust mean estimation \cite{lai2016agnostic, minsker2015geometric}. This leads to a general method for risk minimization based on embedding robust gradient estimation within a projected gradient descent algorithm  (see Algorithm~\ref{algo:festimation}).

\subsection{Robust Estimation}
\label{sec:heavy}
One of the goals of this work is to develop general statistical estimation methods that are robust under varied robustness settings. We derive corollaries in particular for two robustness models: Huber's $\epsilon$-contamination model, and the heavy-tailed model. We now briefly review these two notions of robustness.
\begin{enumerate}
\item[(a)] {\bf Huber's $\epsilon$-contamination model: } 
Huber~\citep{huber1964robust,huber1965robust} proposed the $\epsilon$-contamination model where we observe samples that are obtained from a mixture of the form
\begin{align} \label{eqn:mixture}
 P_\epsilon = (1-\epsilon) P + \epsilon Q ,
\end{align}
where $P$ is the true distribution, $\epsilon$ is the expected fraction of outliers and $Q$ is an \emph{arbitrary} outlier distribution.  Given i.i.d. observations drawn from $P_{\epsilon}$, our objective is to estimate $\theta^*$, the minimizer of the population risk $\mathcal{R}(\theta) = \mathbb{E}_{z \sim P} \left[ \poploss(\theta; z) \right]$, robust to the contamination from $Q$.

\item[(b)] {\bf Heavy-tailed model: }
In the heavy-tailed model it is assumed that the data follows a heavy-tailed distribution (i.e, $P$ is heavy-tailed). 
While heavy-tailed distributions have various possible characterizations: in this paper we consider a characterization via gradients. For a fixed $\theta \in \Theta$ we let $\graddist$ denote the multivariate distribution of the gradient of population loss, i.e. $\nabla \poploss(\theta; z).$ We refer to a potentially heavy-tailed distribution as one for which our only assumption on $\graddist$ is that it has finite second moments for any $\theta \in \Theta$. As we illustrate in Section~\ref{sec:consequences_heavy}, in various concrete examples this translates to relatively weak low-order moment assumptions on the data distribution $P$. 

Given $n$ i.i.d observations from $P$, our objective is to estimate the minimizer of the population risk. From a conceptual standpoint, the classical analysis of risk-minimization which relies on uniform concentration of the empirical risk around the true risk, fails in the heavy-tailed setting necessitating new estimators and analyses \cite{lugosi2017sub,lugosi2016risk,brownlees2015empirical,hsu2016loss}.

\end{enumerate}

%% file: gradientOracles.tex
\section{Robust Gradient Descent via Gradient Estimation}
\label{sec:gradientOracles}
 
Gradient descent and its variants are at the heart of modern optimization and are well-studied in the literature. 
Suppose we have access to the true distribution $P_{\tparam}$. Then to minimize the population risk $\risk(\theta)$, we can use projected gradient descent, where starting at some initial $\theta^0$ and for an appropriately chosen step-size $\eta$, we update our estimate according to:
\begin{equation}
\label{eqn:pgd_poploss}
\theta^{t+1}  \leftarrow \mathcal{P}_{\Theta}(\theta^t - \eta \nabla \risk(\theta^t)).
\end{equation}
However, we only have access to $n$ samples $\dataset = \{z_i\}_{i = 1}^n$. The key technical challenges are then to estimate the gradient of $\risk(\theta)$ from samples $\dataset$, and to ensure that an appropriate modification of gradient descent is stable to the resulting estimation error. 

To address the first challenge we observe that the gradient of the population risk at any point $\theta$ is the mean of a multivariate distribution, \ie $\grad \risk (\theta) = E_{z \sim P} \left[ \grad \poploss(\theta;z) \right]$. Accordingly, the problem of gradient estimation can be reduced to a multivariate mean estimation problem, where our goal is to \emph{robustly} estimate the true mean $\nabla \risk(\theta)$ from $n$ samples $\{\nabla \poploss(\theta; z_i)\}_{i = 1}^n$.  
For a given sample-size $n$ and confidence parameter $\delta \in (0,1)$ we define a gradient estimator: 
\begin{definition}
A function $g(\theta; \dataset, \delta)$ is a gradient estimator, if for functions $\alpha$ and $\beta$, with probability at least $1-\delta$, at any fixed $\theta \in \Theta$, the estimator satisfies the following inequality: 
\begin{equation}
\label{eqn:grad_estimator}
\|g(\theta; \dataset, \delta) - \nabla \risk(\theta)\|_2 \leq \alpha(n, \delta)\|\theta - \theta^*\|_2 + \beta(n, \delta).
\end{equation}
\end{definition} 
\noindent In subsequent sections, we will develop conditions under which we can obtain gradient estimators with strong control on the functions $\alpha(n,\delta)$ and $\beta(n,\delta)$ in the Huber and heavy-tailed models. Furthermore, by investigating the stability of gradient descent we will develop sufficient conditions on these functions such that gradient descent with an inaccurate gradient estimator still returns an accurate estimate. 

To minimize $\risk(\theta)$, we replace $\nabla \risk(\theta)$ in equation \eqref{eqn:pgd_poploss} with the gradient estimator $g(\theta; \dataset, \delta)$ and perform projected gradient descent. 
In order to avoid complex statistical dependency issues that can arise in the analysis of gradient descent, for our theoretical results we consider a sample-splitting variant of the algorithm where each iteration is performed on a fresh batch of samples. We summarize the overall robust gradient descent algorithm via gradient estimation in Algorithm~\ref{algo:festimation}. In contrast to $M$-estimation where we use robust estimates of the overall loss function, here we use robust estimates of the gradient, a small shift in perspective, but which has strong statistical and computational consequences: we obtain a computationally practical algorithm, and moreover with strong robustness guarantees via careful statistical analyses of the stability of the resulting biased and inexact gradient descent iterates.

We further assume that the number of gradient iterations $T$ is specified a-priori, and accordingly we define:
\begin{align}
\widetilde{n} = \floor{\frac{n}{T}} ~~~~\text{and}~~~~\widetilde{\delta} = \frac{\delta}{T}.
\end{align} 
\noindent
We discuss methods for selecting $T$, and the impact of sample-splitting in later sections. As confirmed in our experiments (see Section~\ref{sec:experiments}), sample-splitting should be viewed as a device introduced for theoretical convenience which can likely be eliminated via more complex uniform arguments (see for instance the work~\citep{balakrishnan2017statistical}). 

\vskip0.1in
\noindent
It can be seen that the key ingredient in the robust gradient descent Algorithm~\ref{algo:festimation} is a robust estimator of the gradients. Next, we consider the two notions of robustness described in Section \ref{sec:background}, and derive specific gradient estimators for each of the models using the framework described above. Although we derive corollaries of our general results for these two settings of Huber contamination and heavy-tailed models, we emphasize that our class of estimators are more general and are not restricted to these two notions of robustness. 

\begin{figure}
    \begin{minipage}{\textwidth}
      \begin{algorithm}[H]
\caption{Robust Gradient Descent}
\label{algo:festimation}
        \begin{algorithmic}
\Function{RGD (Gradient Estimator $g(\cdot)$, Data $\{z_1,\ldots,z_n\},$ Step Size $\eta$, Number of Iterations $T$, Confidence $\delta$)}{}
\State Split samples into $T$ subsets $\{\calZ_t\}_{t=1}^T $ of size $\widetilde{n}$.
\For{$t=0$ to $T-1$} 
\State $\theta^{t+1} = \argmin_{\theta \in \Theta} 
\norm{\theta - \left(\theta^{t} - \eta \, g(\theta^t; \mathcal{Z}_t, \widetilde{\delta}) \right)}{2}^2.$
\EndFor
\EndFunction
\end{algorithmic}
      \end{algorithm}
    \end{minipage}
  \end{figure} 
  
\subsection{Gradient Estimation in Huber's $\epsilon$-contamination model}


There has been a flurry of recent interest~\citep{diakonikolas2016robust,lai2016agnostic,du2017computationally,li17robust,charikar2016learning} in designing mean estimators which, under the Huber contamination model, can robustly estimate the mean of a random vector. While some of these results are focused on the case where the uncorrupted distribution is Gaussian, or isotropic,  we are more interested in robust mean oracles for more general distributions. \citet{lai2016agnostic} proposed a robust mean estimator for general distributions, satisfying weak moment assumptions, and we leverage the existence of such an estimator to design a \emph{Huber gradient estimator} $g(\theta; \dataset, \delta)$ which works in the Huber contamination model.

The estimator builds upon the fact that with a single dimension, it is relatively easy to estimate the gradient robustly. In higher dimensions, the crucial insight of \citet{lai2016agnostic} is that the effect of the contamination distribution $Q$ on the mean of uncontaminated distribution $P$ is effectively one-dimensional provided we can accurately estimate the direction along which the mean is shifted. In our context, if we can compute the gradient shift direction, i.e. the direction of the difference between the sample (corrupted) mean gradient and the true (population) gradient, then the true gradient can be estimated by using a robust 1D-mean algorithm along the gradient-shift direction and a non-robust sample-gradient in the orthogonal direction since the contamination has no effect on the gradient in this orthogonal direction. In order to identify this gradient shift direction, we follow \citet{lai2016agnostic} and use a recursive Singular Value Decomposition (SVD) based algorithm. In each stage of the recursion, we first remove gross-outliers via a truncation algorithm (described in more detail in the Appendix, and termed $ \text{{\sc HuberOutlierGradientTruncation}}$ in Algorithm~\ref{algo:huber_mean_estimation}). We subsequently identify two subspaces using an SVD -- a clean subspace where the contamination has a small effect on the mean and another subspace where the contamination has a potentially larger effect. We use a simple sample-mean estimator in the clean subspace and recurse our computation on the other subspace. Building on the work of \citet{lai2016agnostic}, in Lemma~\ref{lem:lai2016agnostic} and Appendix~\ref{sec:proof:lem:lai2016agnostic} we provide a careful non-asymptotic analysis of this gradient estimator. 

Algorithm~\ref{algo:huber_mean_estimation} presents the overall \emph{Huber gradient estimator} $g(\theta; \dataset, \delta)$.



\begin{algorithm}[H]
\centering
\caption{Huber Gradient Estimator}
        \begin{algorithmic}
\Function{HuberGradientEstimator(Sample Gradients $S = \{\nabla \poploss(\theta;z_i)\}_{i = 1}^n$, Corruption Level $\epsilon$, Dimension $p$, $\delta$)}{}
\State $\widetilde{S} = \text{{\sc HuberOutlierGradientTruncation}}(S,\epsilon,p,\delta)$.
\If{p=1} 
\State \Return $\text{mean}(\widetilde{S})$
\Else
\State Let $\Sigma_{\widetilde{S}}$ be the covariance matrix of $\widetilde{S}$. 

\State Let $V$ be the span of the top $p/2$ principal components of $\Sigma_{\widetilde{S}}$ and $W$ be its complement.

\State Set $S_1 := P_V (\widetilde{S})$ where $P_V$ is the projection operation on to $V$.

\State Let  $\widehat{\mu}_V := \text{\sc{HuberGradientEstimator}}(S_1,\epsilon,p/2,\delta)$. 

\State Let $\widehat{\mu}_W := \text{mean}(P_W \widetilde{S})$.

\State Let $\widehat{\mu} \in \mathbb{R}^p$ be such that $P_V(\widehat{\mu}) = \widehat{\mu}_V$, and $P_W(\widehat{\mu}) = \widehat{\mu}_W$.

\Return $\widehat{\mu}$.
\EndIf

\EndFunction
\end{algorithmic}
\label{algo:huber_mean_estimation}
\end{algorithm}
\subsection{Gradient Estimation in the Heavy-Tailed model}
To design gradient estimators for the heavy-tailed model, we leverage recent work on designing robust \emph{mean} estimators in this setting. These robust mean estimators build on the classical work of \citet{alon96,nemirovski1983problem} and \citet{jerrum86} on the so-called median-of-means estimator. For the problem of one-dimensional mean estimation, \citet{lerasle2011robust, catoni2012challenging} propose robust mean estimators that achieve exponential concentration around the true mean for any distribution with bounded second moment. In this work we require mean estimators for multivariate distributions. Several recent works (\cite{hsu2016loss,minsker2015geometric,lugosi2017sub}) extend the median-of-means estimator of to general metric spaces. In this paper we use the geometric median-of-means estimator (\gmom), which was originally proposed and analyzed by \citet{minsker2015geometric}, to design the gradient estimator $g(\theta; \dataset, \delta)$. 

The basic idea behind the \gmom~estimator is to first split the samples into non-overlapping
subsamples and estimate the sample mean of each of the subsamples. Then the \gmom~estimator is given by the median-of-means of the subsamples.  
Formally, let $\{x_i \dots x_n \}\in \mathbb{R}$ be $n$ i.i.d random variables sampled from a distribution $P$.  Then the \gmom~estimator for estimating the mean of $P$ can be described as follows.  Partition the $n$ samples into $b$ blocks $B_1, \dots B_b$ each of size $\floor{n/b}$. Let $\{\widehat{\mu}_1, \dots, \widehat{\mu}_b\}$ be the sample means in each block, where $\widehat{\mu}_i = \frac{1}{|B_i|}\sum_{x_j \in B_i}x_j$.  Then the \gmom~estimator is given by $\text{median}\{\widehat{\mu}_1, \dots \widehat{\mu}_b\}$. In high dimensions where different notions of the median have been considered \citet{minsker2015geometric} uses geometric median:
\[
\widehat{\mu} = \argmin_{\mu} \sum_{i = 1}^b\|\mu - \widehat{\mu}_i\|_2.
\]
Algorithm~\ref{algo:heavy_tailed_estimation} presents the gradient estimator $g(\theta; \dataset, \delta)$ obtained using \gmom~as the mean estimator.
\begin{algorithm}[H]
\centering
\caption{Heavy Tailed Gradient Estimator}
\label{algo:heavy_tailed_estimation}
        \begin{algorithmic}
\Function{HeavyTailedGradientEstimator(Sample Gradients $S = \{\nabla \poploss(\theta;z_i)\}_{i = 1}^n$, $\delta$)}{}
\State Define number of buckets $b = 1 + \lfloor 3.5\log{1/\delta}\rfloor$.

\State Partition $S$ into $b$ blocks $B_1, \dots B_b$ each of size $\floor{n/b}$.

\For{$i = 1 \dots n$}
 \State $\widehat{\mu}_i =  \frac{1}{|B_i|}\displaystyle\sum_{s \in B_i}s$.
\EndFor
\State Let $\widehat{\mu} = \displaystyle \argmin_{\mu} \sum_{i = 1}^b\|\mu - \widehat{\mu}_i\|_2$.

\Return $\widehat{\mu}$.
\EndFunction
\end{algorithmic}
\end{algorithm}

\subsection{Choice of Hyper-Parameters}

In this section, we discuss how to tune the hyperparameters for our algorithms. In particular, note that the gradient estimators described in Algorithms~\ref{algo:huber_mean_estimation}, \ref{algo:heavy_tailed_estimation} depend on corruption level $\epsilon$, and on confidence  $\delta$, which are not known in advance. 

Since the standard hyper-parameter selection techniques such as cross validation, hold-out validation, pick hyper-parameters that minimize the empirical mean of the loss on validation data, they can't be used in the presence of outliers in the data. One criteria we could use in such cases is to choose hyper-parameters that minimize a robust estimate of the population risk on validation data. However, we can't use any of the existing robust mean estimators to estimate the population risk because they themselves depend on hyper-parameters such as corruption level $\epsilon$.

\paragraph{Huber Contamination.} We now consider the Huber contamination model and propose a heuristic based on Scheffe's tournament estimator~\citep{yatracos1985,devroye2012combinatorial} for hyper-parameter selection. In particular, we consider the gradient descent procedure described in Algorithm~\ref{algo:huber_mean_estimation} and explain our technique for choosing $\epsilon, \delta$ using hold out cross validation.  Note that our goal is to pick hyper-parameters that minimize the population risk $\mathcal{R}(\theta)$. Under the assumption of strong convexity of $\mathcal{R}(\theta)$, this is equivalent to picking hyper-parameters that minimize the parameter error $\|\theta - \theta^*\|_2$. 

We begin with the problem of density estimation, where we are given $n$ i.i.d samples $\{z_i\}_{i = 1}^n$ from $(1-\epsilon) P_{\theta^*} + \epsilon Q$, where $P_{\theta^*}$  belongs to the class of distributions $\{P_{\theta}: \theta \in \Theta\}$,  and $Q$ is an arbitrary distribution. Our goal is to estimate $\theta^* \in \Theta$ from the samples.
Suppose $\left\{ P_{\eparam_1},P_{\eparam_2},\ldots, P_{\eparam_m} \right\}$ are the candidate solutions returned by Algorithm~\ref{algo:huber_mean_estimation} for different settings of $\epsilon, \delta$. Consider the following pairwise test function:
\begin{align} 
\phi_{jk} = \mathbb{I} \left\{ \abs{ \frac{1}{n_{val}} \sum \limits_{i=1}^{n_{val}} \indic{p_{\eparam_j}(z'_i) > {p_{\eparam_k}(z'_i)}} - P_{\eparam_j} \paren{p_{\eparam_j}(z) > p_{\eparam_k}(z) } } >  \right. \nonumber \\ \left. \abs{ \frac{1}{n_{val}} \sum \limits_{i=1}^{n_{val}} \indic{p_{\eparam_j}(z'_i) > {p_{\eparam_k}(z'_i)}} - P_{\eparam_k} \paren{p_{\eparam_j}(z) > p_{\eparam_k}(z) } } \right\},
\end{align}
where $p_{\eparam_j}$ is the probability density of $P_{\eparam_j}$, $\{z'_i\}_{i = 1}^{n_{val}}$ is the validation data and $\indic{.}$ is the indicator function.
When $\phi_{jk}=1$, then $\eparam_k$ is favored over $\eparam_j$ and when $\phi_{jk} =0$, then $\eparam_j$ is favored over $\eparam_k$. Then, the final estimate $P_\winner$ is given by
\[ j^* = \argmin \limits_{j \in [m]} \sum \limits_{ \substack{{k=1} \\ { k \neq j} } }^m \phi_{jk}   \]
Following \citep{devroye2012combinatorial}, it can be shown that the above procedure picks a $j^*$ such that $P_\winner$ is close to $P_{\theta^*}$ in TV metric. For distributions $\{P_{\theta}: \theta \in \Theta\}$ whose TV metric is roughly equivalent to the parameter error, the above procedure results in hyper-parameters which minimize the parameter error.
This procedure can be extended to supervised learning problems such as regression and classification.

\paragraph{Heavy-Tailed Distribution.} For the Heavy-Tailed setting we experimented with (a) empirical mean of the risk on validation data: $\frac{1}{n_{val}}\sum_{i=1}^{n_{val}} \poploss(\theta; z'_i)$ where $\{z'_i\}_{i=1}^{n_{val}}$ is the validation data, which does not require any tuning parameters, as well as (b) median of means based mean of the risk on validation data, for various confidence levels $\delta$. However, both the techniques in the context of hold-out validation resulted in models with similar performance. So, in our experiments with heavy tailed distributions, we present results obtained using the empirical risk as in (a) on hold-out validation data.

%% file: exps.tex
\section{Experiments}\label{sec:experiments}
In this section we demonstrate our proposed methods for the Huber contamination and heavy-tailed models, on a variety of simulated and real data examples. 

\subsection{Huber Contamination}
\label{sec:exp_huber}
We first consider the Huber contamination model and demonstrate the practical utility of gradient-descent based robust estimator described in Algorithms~\ref{algo:festimation} and \ref{algo:huber_mean_estimation}.
\subsubsection{Synthetic Experiments: Linear Regression}
\begin{figure}[!ht]
\centering
      \subfigure[\label{fig:regression_p} \small{Parameter error vs $p$ for $\epsilon = 0.1$} ]{\includegraphics[width=0.32\textwidth]{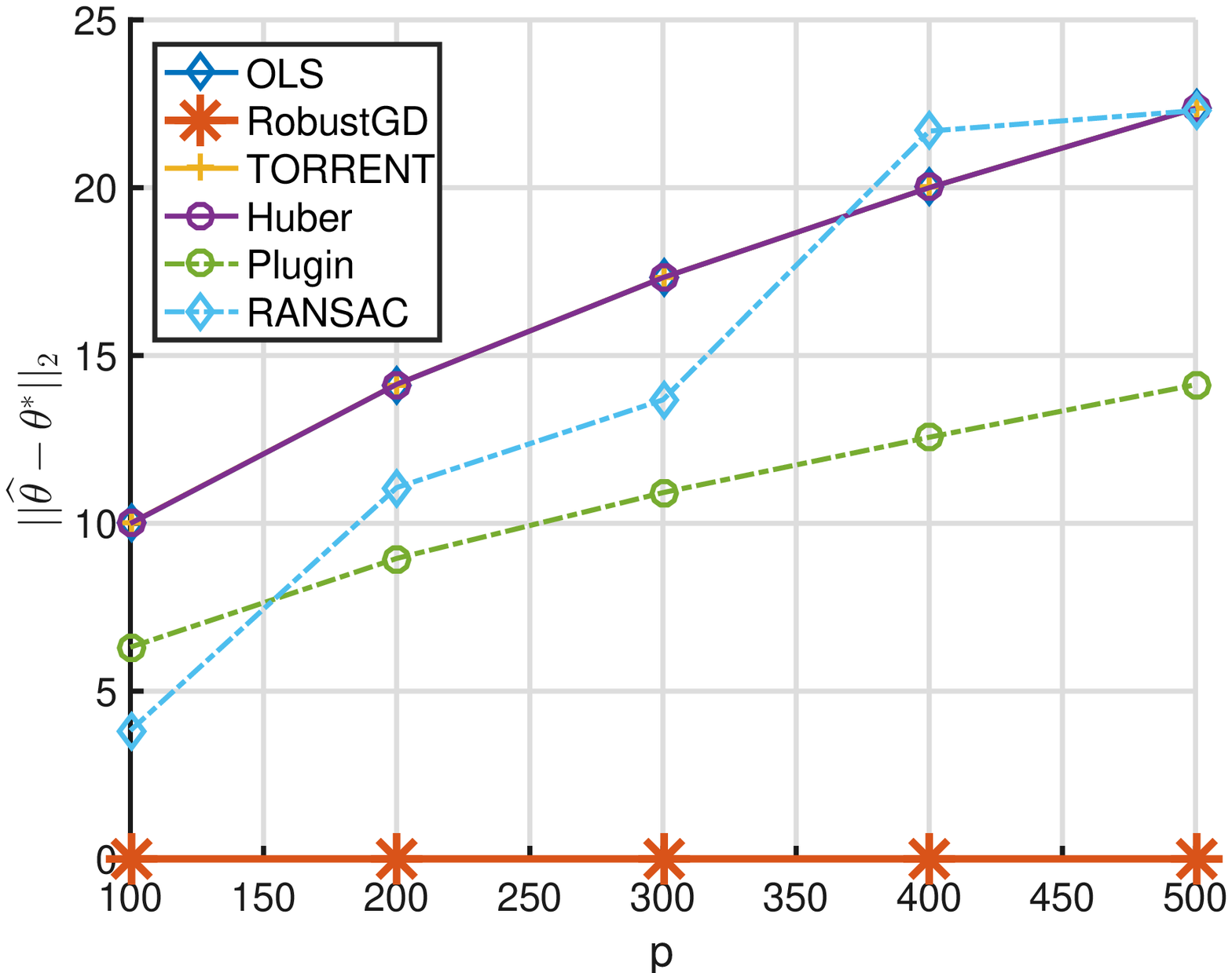}}   
      \subfigure[\label{fig:regression_iter_LRV_p} \small{Parameter error vs $\epsilon$}]{\includegraphics[width=0.32\textwidth]{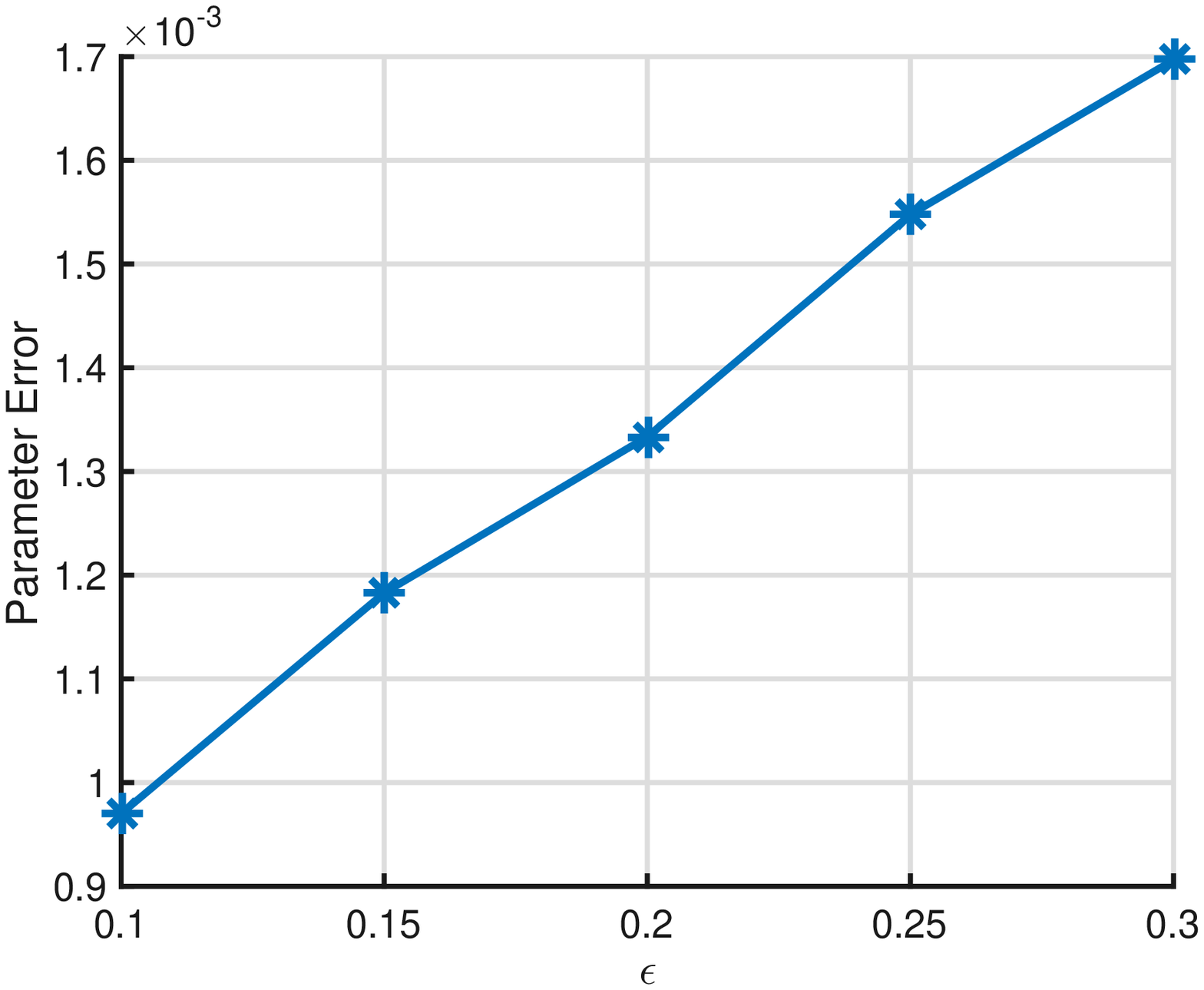}}
       \subfigure[\label{fig:regression_iter_LRV_GD}  \small{$\log(\|\theta^t - \theta^*\|_2)$ vs $t$ for different $\epsilon$.}]{\includegraphics[width=0.32\textwidth]{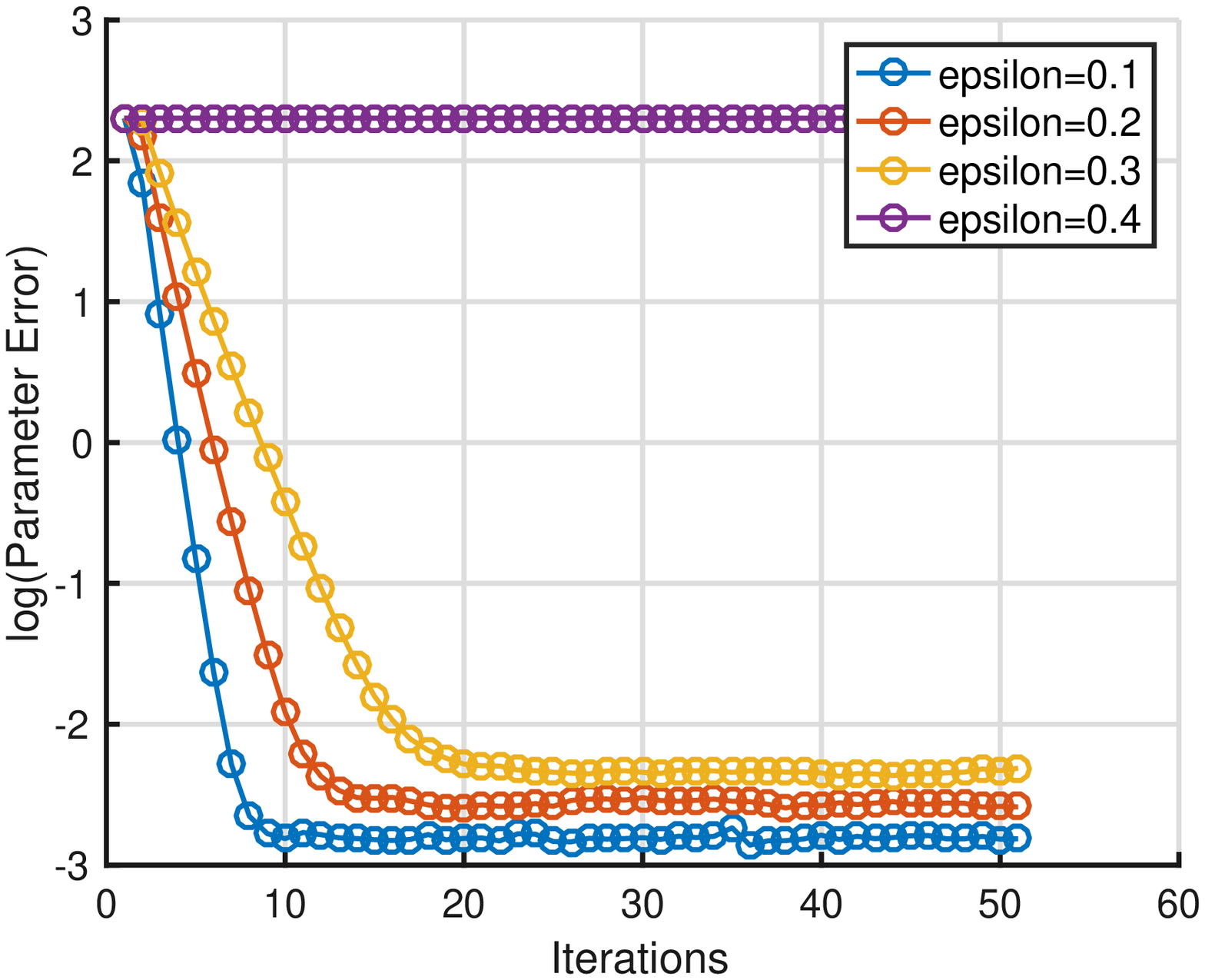}}
       
       \subfigure[\label{fig:regression_val_LRV_epsilon} \small{Hyperparameter Tuning vs $p$}]{\includegraphics[width=0.32\textwidth]{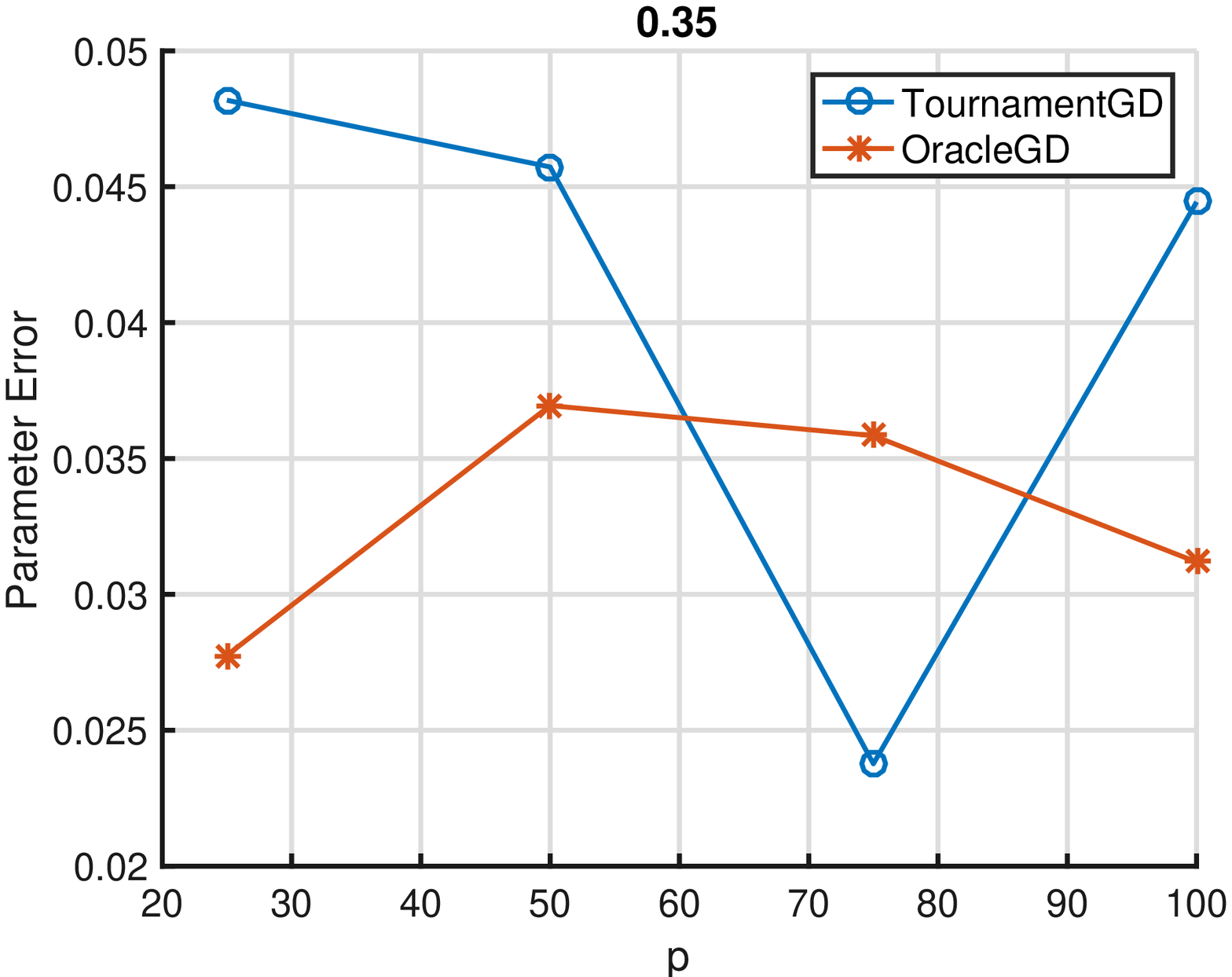}}
       \subfigure[\label{fig:regression_val_LRV_p}  \small{Hyperparameter Tuning vs $\epsilon$}]{\includegraphics[width=0.32\textwidth]{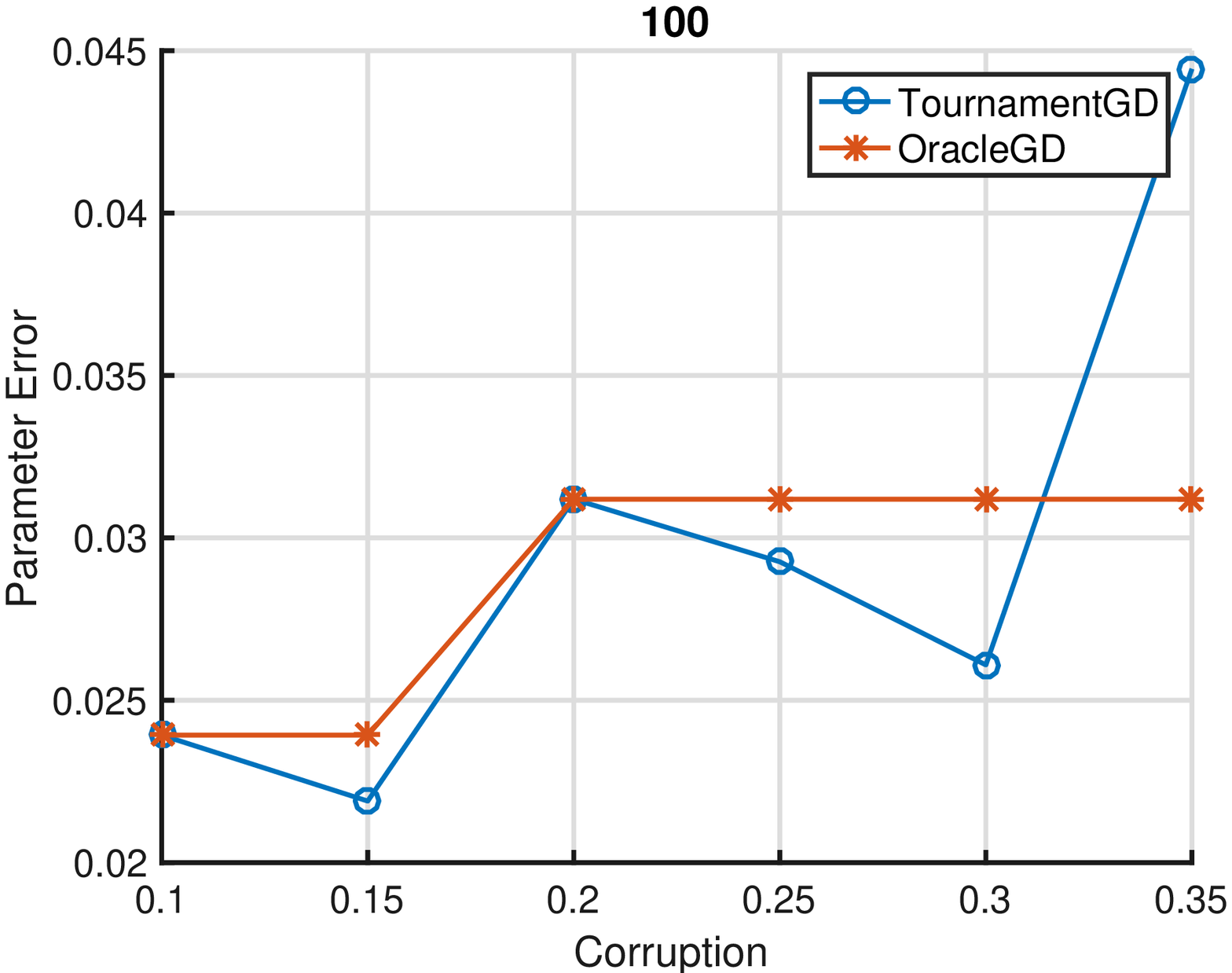}}
       
\caption{Robust Linear Regression.}
\label{fig:linear_regression}
\end{figure}
Recall the linear regression model described in~\eqref{eq:linregress} where we observe paired samples $\{(x_i, y_i)\}_{i = 1}^n$. We assume that the $(x,y)$ pairs sampled from the true distribution $P$ are linked via a linear model:
\begin{align*}
y = \inprod{x}{\true} + w.
\end{align*}
We now describe the experiment setup, the data model and present the results.
\paragraph{Setup.} We fix the contamination level $\epsilon = 0.1$ and $\sigma^2 = 0.1$. Next, we generate $(1 - \epsilon)n$ \emph{clean} covariates from a multivariate Gaussian $x \sim \calN(0,\calI_p)$, and we generate the corresponding clean responses using $y = \iprod{x}{\tparam}+w$ where $\tparam = [1,\ldots,1]^T$ and $w \sim \calN(0,\sigma^2)$. We simulate an outlier distribution by drawing the covariates from $\calN(0,p^2\calI_p)$, and setting the responses  to $0$. The total number of samples is set to be $(10\frac{p}{\epsilon^2})$. 
We note that the sample size we choose increases with the dimension. This scaling is used to ensure that the statistical (minimax) error, in the absence of any contamination, is roughly 0.001. An optimally robust method should have error close to 0.1 (roughly equal to corruption level), which ours does (see Figure~\ref{fig:linear_regression}).

\paragraph{Metric.} We measure the parameter error in $\ell_2$-norm. 
We also study the convergence properties of our proposed method, for different contamination levels $\epsilon$. We use code provided by \citet{lai2016agnostic} 
to implement our gradient estimator.

\paragraph{Baselines.} As our baselines, we use OLS, TORRENT~\cite{bhatia2015robust}, the Huber-loss M-estimator, RANSAC and a plugin estimator (detailed further in Section~{sec:linregtheory}, and which in a nutshell robustly estimates the sufficient statistics required for the OLS estimator). TORRENT is an iterative hard-thresholding based alternating minimization algorithm, where in one step, it calculates an active set of examples by keeping only $(1-\epsilon)n$ samples which have the smallest absolute values of residual $r = y - \iprod{x}{\theta^t}$, and in the other step it updates the current estimates by solving OLS on the active set. \citet{bhatia2015robust} have shown the superiority of TORRENT over a variety of other convex-penalty based outlier techniques, hence, we do not compare against those methods. The plugin estimator is implemented using 
Algorithm~\ref{algo:huber_mean_estimation} to estimate both the mean vector $\frac{1}{n} \sum_{i=1}^n y_i x_i$ and the covariance matrix $\frac{1}{n} \sum_{i=1}^n x_i x_i^T$, which are the required sufficient statistics for the OLS estimator.

\paragraph{Results.} We summarize our main findings here.
\begin{itemize}[leftmargin=*]
\item \textbf{Error vs dimension $p$:} All estimators except our proposed algorithm perform poorly with increasing dimension, as shown in Figure~\ref{fig:regression_p}. Note that the TORRENT algorithm has strong guarantees when 
only the response $y$ is corrupted but performs poorly in the Huber contamination model where both $x$ and $y$ may be contaminated. 
We find that the error for the robust plugin estimator increases with dimension. We investigate this theoretically in Section~\ref{sec:linregtheory}, where we find that the error of the plugin estimator grows with the norm of $\tparam$. In our experiments, we choose $\norm{\tparam}{2} = \sqrt{p}$, and thus Figure~\ref{fig:regression_p} corroborates Corollary~\ref{cor:linregress_plugin} in Section~\ref{sec:linregtheory}.
\item \textbf{Error vs $\epsilon$:} In Figure~\ref{fig:regression_iter_LRV_p} we find that the parameter error $\|\widehat{\theta} - \theta^*\|_2$ increases linearly with the contamination rate $\epsilon$ and we study this further in Section~\ref{sec:linregtheory}.
\item \textbf{Error vs iteration $t$:} Finally, Figure~\ref{fig:regression_iter_LRV_GD} shows that the convergence rate decreases with increasing contamination $\epsilon$ and after $\epsilon$ is high enough, the algorithm remains stuck at $\theta_0$, corroborating Lemma~\ref{lem:lowerBoundGradient} (in the Appendix). 
\item \textbf{Hyper-parameter Tuning:} In Figures~\ref{fig:regression_val_LRV_epsilon} and~\ref{fig:regression_val_LRV_p}, we find the final solution chosen by our tournament based heuristic for hyper-parameter selection (TournamentGD) has roughly the same performance as the algorithm which knows the true value of $\epsilon$ (OracleGD). In particular, our final error does not scale with $p$. 
\end{itemize}
Next, we study the performance of our proposed method in the context of classification. 
\subsubsection{Synthetic Experiments: Logistic Regression}
%
\begin{figure}[!ht]
\centering
      \subfigure[\label{fig:logistic_p} \small{0-1 Error vs $p$ at $\epsilon = 0.1$} ]{\includegraphics[width=0.32\textwidth]{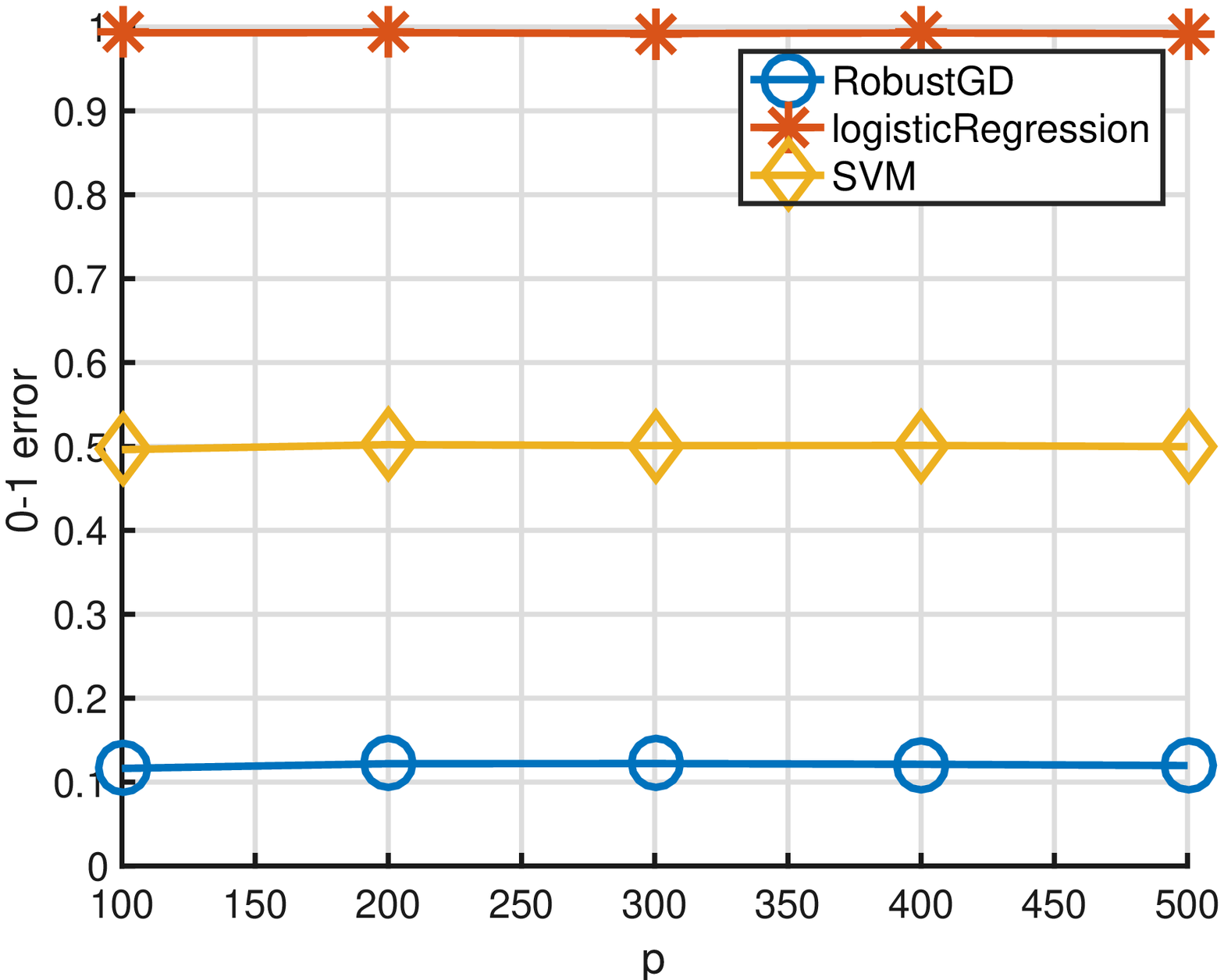}}   
      \subfigure[\label{fig:logistic_iter_LRV_p} \small{0-1 error vs $\epsilon$}]{\includegraphics[width=0.32\textwidth]{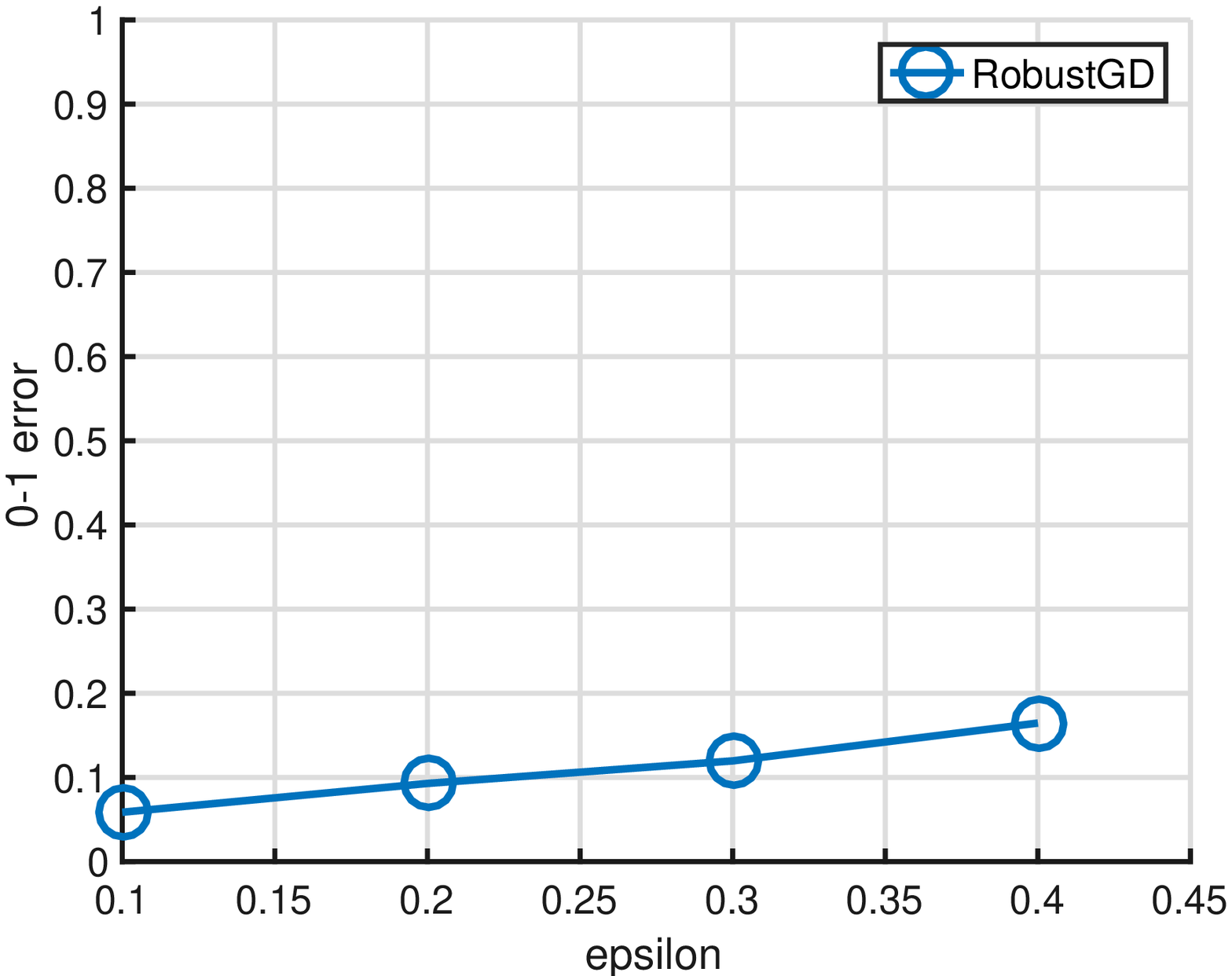}}
       \subfigure[\label{fig:logistic_iter_LRV_GD}  \small{$\log(\|\theta^t - \theta^*\|_2)$ vs $t$ for different $\epsilon$}]{\includegraphics[width=0.32\textwidth]{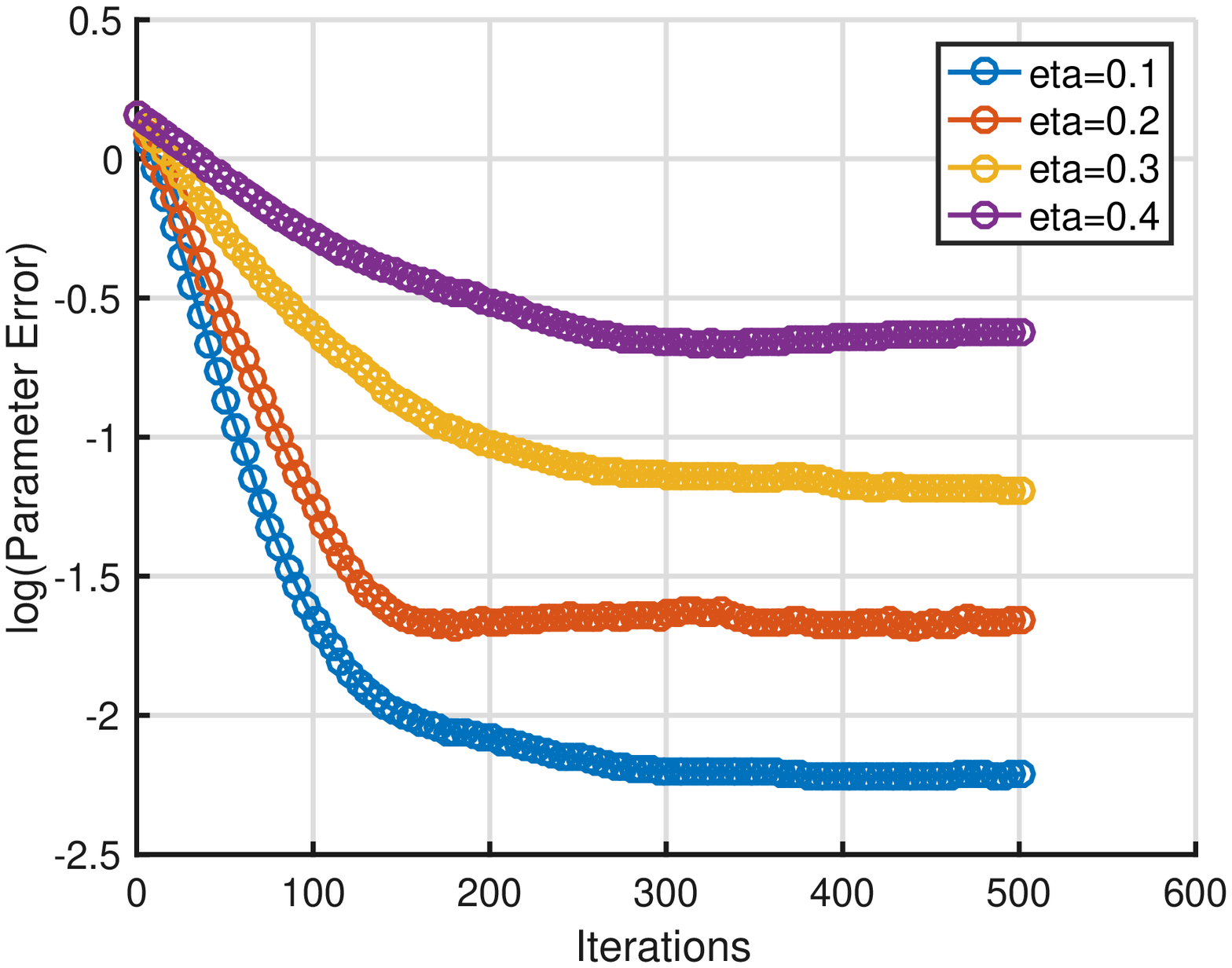}}
\caption{Robust Logistic Regression.}
\end{figure}
\paragraph{Setup.} We simulate a linearly separable classification problem, where the \emph{clean} covariates are sampled from $\calN(0,\calI_p)$, the corresponding clean responses are computed as $y = \text{sign}(\iprod{x}{\tparam})$ where $\tparam = [1/\sqrt{p},\ldots,1/\sqrt{p}]^T$. We simulate the outlier distribution by adding asymmetric noise, \ie we flip the labels of one class, and increase the variance of the corresponding covariates by multiplying them by $p^2$. The total number of samples are set to be $(10\frac{p}{\epsilon^2})$.
\paragraph{Metric.} We measure the 0-1 classification error on a held-out (clean) test set. We study how the 0-1 error changes with $p$ and $\epsilon$ and the parameter estimation error of our proposed method for different contamination levels $\epsilon$. 
\paragraph{Baselines.} We use the logistic regression MLE and the linear Support Vector Machine (SVM) as our baselines.
\paragraph{Results.} We summarize our main findings below:
\begin{itemize}[leftmargin=*]
\item \textbf{0/1 Error vs dimension $p$:} In Figure~\ref{fig:logistic_p} we observe that both the SVM and logistic regression MLE perform poorly with increasing dimension. The logistic regression MLE 
completely flips the labels and has a 0-1 error close to 1, 
whereas the linear SVM outputs a random hyperplane classifier that flips the label for roughly 
half of the dataset.

\item \textbf{0/1 Error vs $\epsilon$ and $t$:} Figures~\ref{fig:logistic_iter_LRV_p} and~\ref{fig:logistic_iter_LRV_GD} show qualitatively similar results to the linear regression setting, i.e. that the error of our proposed estimator degrades gracefully (and grows linearly) with the contamination level $\epsilon$ and that the gradient descent iterates converge linearly. 

\end{itemize}

\subsubsection{Robust Face Reconstruction}

{\tiny\begin{table}
\centering
\caption{Fitting to original image error.}
\label{table:rmse_res}
\begin{tabular}{@{}llllll@{}}
\toprule
          & Best~Possible & Proposed & TORRENT & OLS  & SCRRR \\ \midrule
Mean RMSE & 0.05         & 0.09      & 0.175   & 0.21 & 0.13 \\ \bottomrule
\end{tabular}
\end{table}}
\paragraph{Setup.} In this experiment, we show the efficacy of our algorithm by attempting to reconstruct face images that have been corrupted with heavy occlusion, where the occluding pixels play the role the outliers. We use the data from the Cropped Yale Dataset \cite{lee2005acquiring} . The dataset contains 38 subjects, and each image has $192 \times 168$ pixels. Following the methodology of \citet{wang2015self}, we choose 8 face images per subject, taken under mild illumination conditions and computed an eigenface set with 20 eigenfaces. Then given a new corrupted face image of a subject, the goal is to get the best reconstruction/approximation of the true face. To remove scaling effects, we normalized all images to $[0, 1]$ range. One image per person was used to test reconstruction. Occlusions were simulated by randomly placing $10$ blocks of size $30\times30$. We repeated this 10 times for each test image. In this setting, each image of a subject corresponds to a different task; \ie $X$ is a common fixed eigenface basis, ${y}$ is an observed(occluded) image, and the goal is to reconstruct(de-noise) the given image using the given basis. Note that in this example, we use a linear regression model as the uncontaminated statistical model, which is almost certainly not an exact match for the unknown ground truth distribution. Despite this model misspecification, as our results show, that robust mean based gradient algorithms do well.


\paragraph{Metric.} We use Root Mean Square Error (RMSE) between the original and reconstructed image to evaluate the performance of the algorithms.  We also compute the best possible reconstruction of the original face image by using the 20 eigenfaces.

\paragraph{Methods.} \citet{wang2015self} implemented popular robust estimators such as RANSAC, Huber Loss \etc and showed their poor performance. \citet{wang2015self} then  proposed an alternate robust regression algorithm called Self Scaled Regularized Robust Regression(SCRRR). Hence, use TORRENT, SCRRR and OLS as baselines. We also compare against the best possible RMSE obtained by reconstructing the un-occluded image using the eigenfaces.

\paragraph{Results.} Table~\ref{table:rmse_res} shows that the mean RMSE is best for our proposed gradient descent based method and that the recovered images are in most cases closer to the un-occluded original image.~(Figure~\ref{fig:eigen}). Figure~\ref{fig:eigen3} shows a case when none of the methods succeed in reconstruction.

\begin{figure}[H]
\label{fig:eigen}
\centering
      \subfigure[\label{fig:eigen1} Successful Reconstruction]{\includegraphics[width=0.32\textwidth]{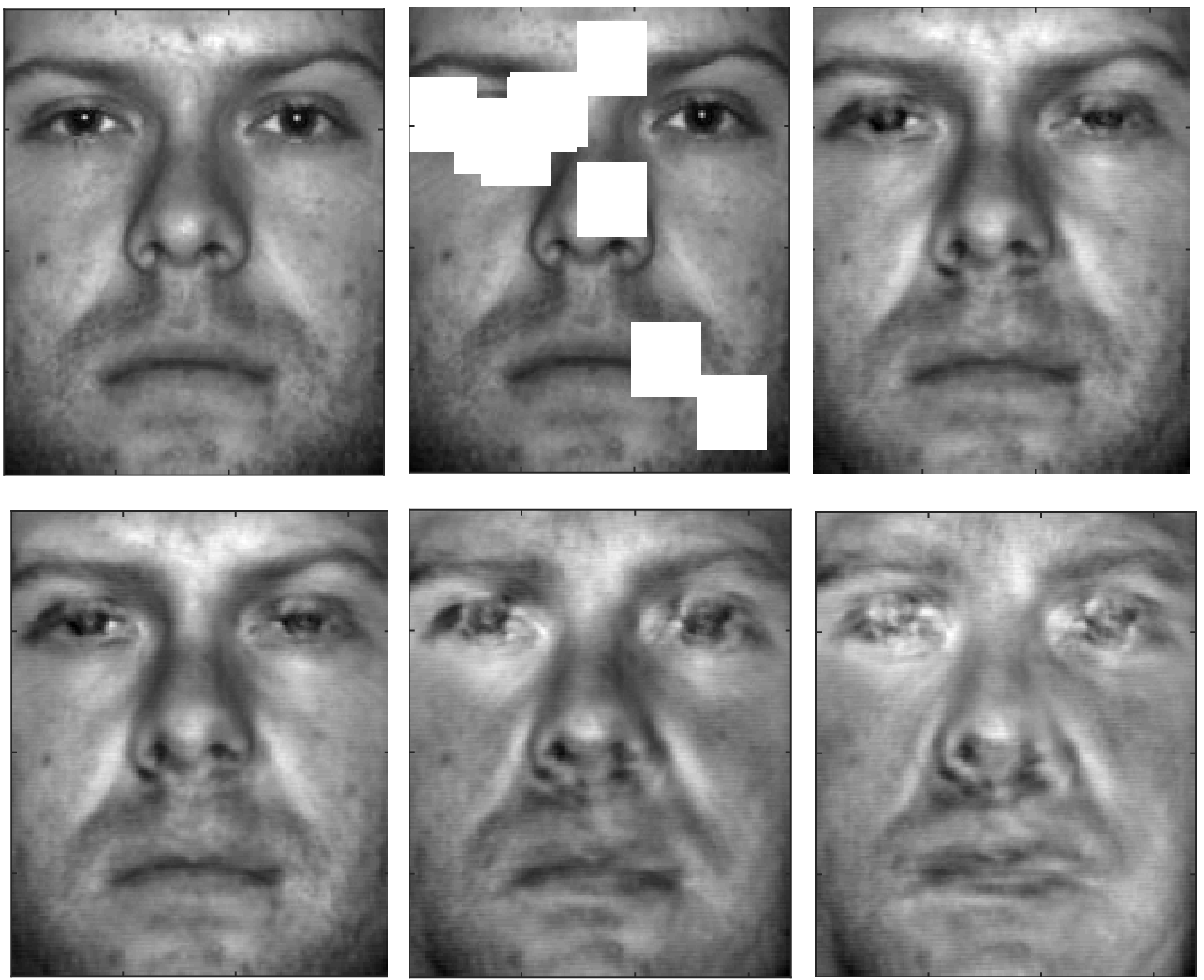}}   
      \hspace{1pt}
      \subfigure[\label{fig:eigen2} Successful Reconstruction]{\includegraphics[width=0.32\textwidth]{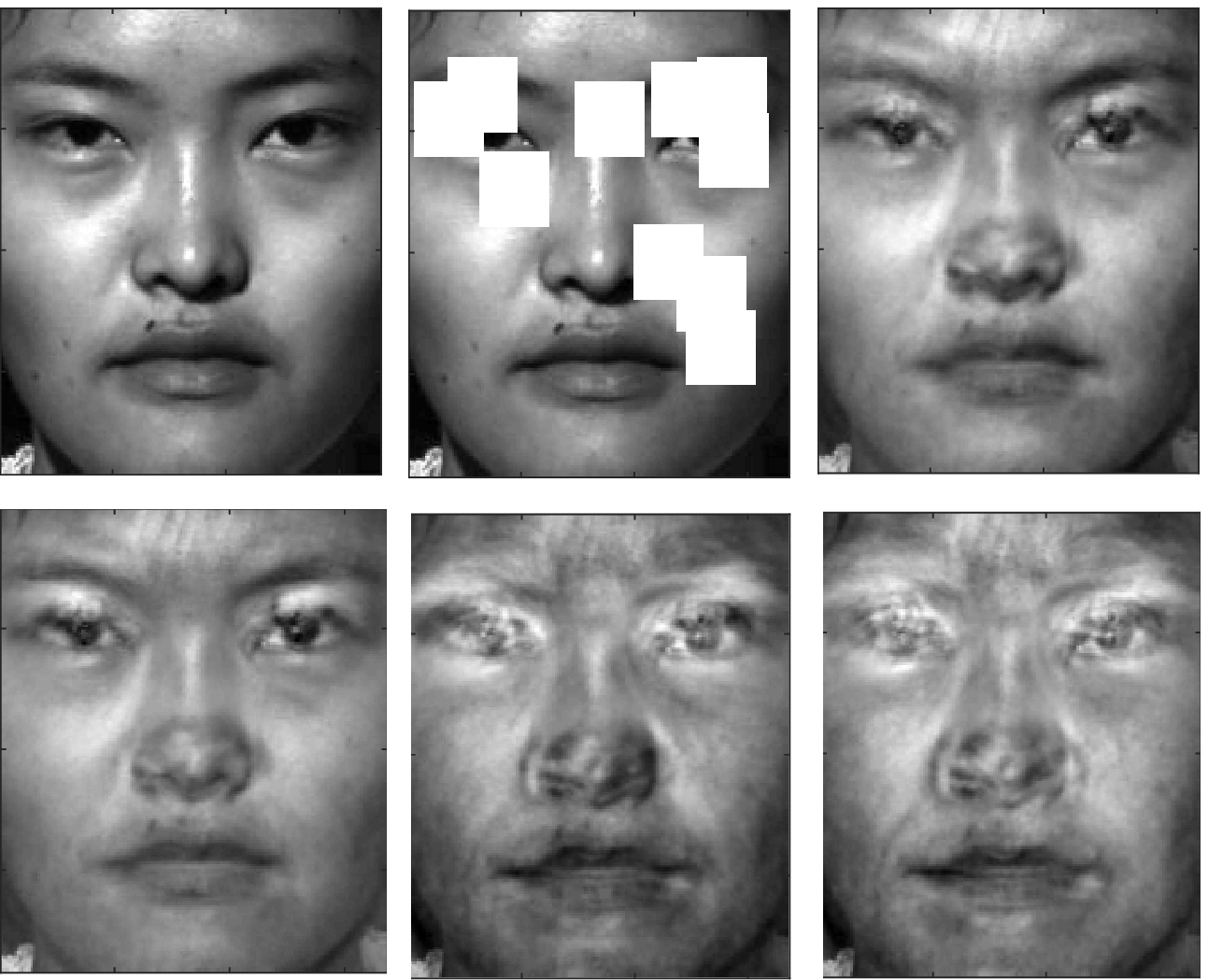}}
            \hspace{1pt}
       \subfigure[\label{fig:eigen3}  Failed Reconstruction]{\includegraphics[width=0.32\textwidth]{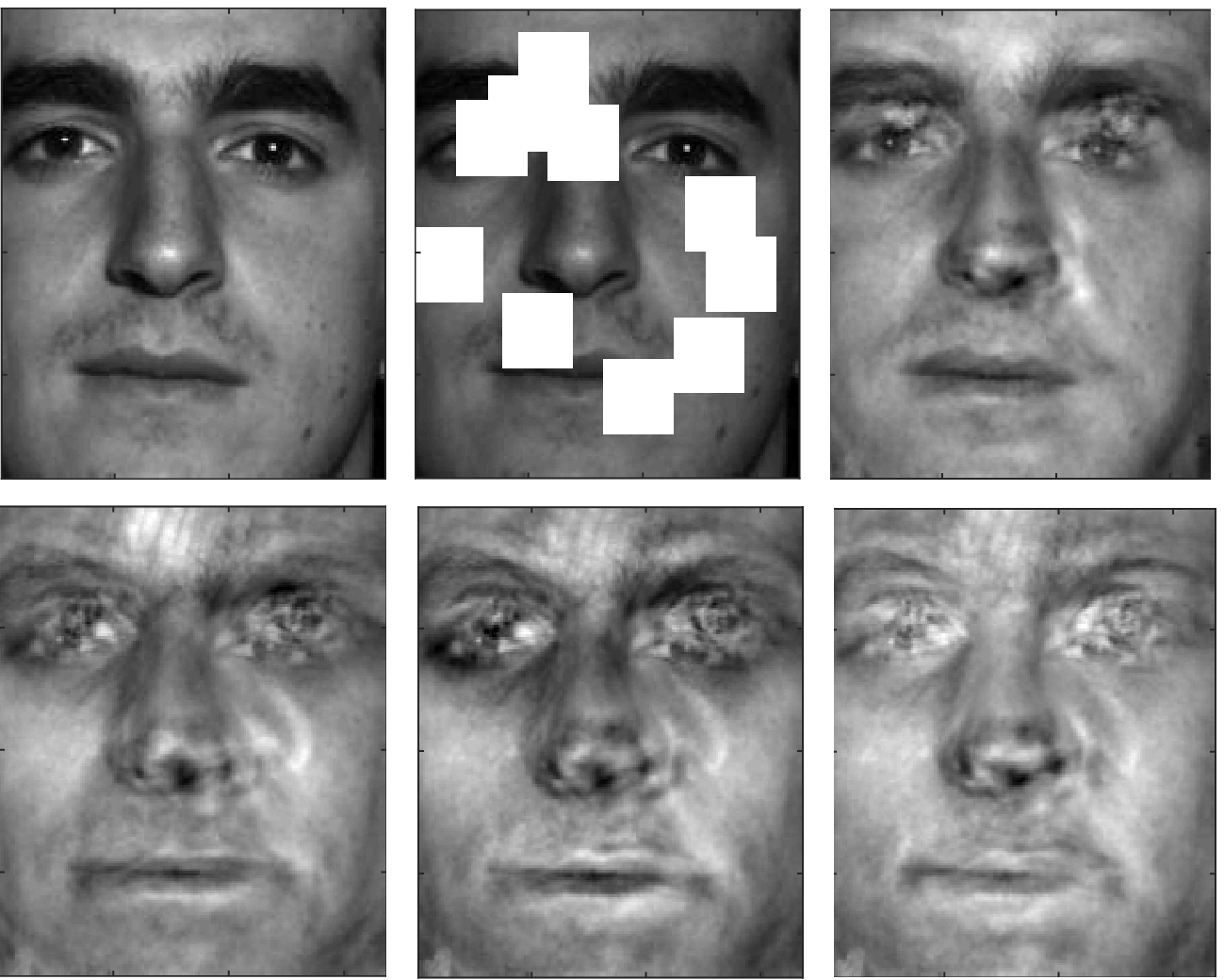}}
       
\caption{Robust Face recovery results: Top; in order from L to R: original image, occluded image, best possible recovery with given basis.  Bottom; in order from L to R: Reconstructions using our proposed algorithm, TORRENT and ordinary least squares (OLS).}
\end{figure}

\subsection{Heavy-tailed Estimation}
We now consider the heavy-tailed model and present experimental results on synthetic datasets comparing the gradient descent based robust estimator described in Algorithms~\ref{algo:festimation} and \ref{algo:heavy_tailed_estimation} (which we call RobustGD) with ERM and several other recent proposals. In these experiments we focus on the problem of linear regression which is described in Section~\ref{sec:exp_huber} and work with heavy-tailed noise distributions.

\subsubsection{Synthetic Experiments: Simple Linear Regression}
\paragraph{Setup.} The covariate $x \in \mathbb{R}^p$ is sampled from a zero-mean isotropic gaussian distribution. We set each entry of $\theta^*$ to $1/\sqrt{p}$. The noise $w$ is sampled from a Pareto distribution, with mean zero, variance $\sigma^2$ and tail parameter $\beta$. The tail parameter $\beta$ determines the moments of the Pareto random variable. More specifically, the moment of order $k$ exists only if $k < \beta$, hence, smaller the $\beta$ the more heavy-tailed the distribution. In this setup, we keep the dimension $p$ fixed to 128 and vary $n$, $\sigma$ and $\beta$. We always maintain the sample-size $n$ to be at least $4p$.

\paragraph{Methods.} We compare RobustGD with several baselines. Since we are always in the low-dimensional ($n \geq p$) setting, the solution to ERM has a closed form expression and is simply the OLS solution. We also study OLS-GD, which performs a gradient descent on ERM and is equivalent to using empirical mean as the gradient oracle in our framework.  We also compare against the robust estimation techniques of \citet{hsu2016loss} and \citet{duchi2016variance}, which we refer to as RobustHS, RobustDN and two classical techniques namely the LASSO \cite{tibshirani1996regression} and ridge regression.  In our experiments, all the iterative techniques are run until convergence.


\paragraph{Metrics.}
We use two metrics to compare the performance of various approaches: a) parameter error which is defined as $\|\theta - \theta^*\|_2$ and
b) to compare the performance of two estimators $\eparam_1$, $\eparam_2$, we define the notion of relative efficiency:
\[ \releff(\eparam_1,\eparam_2) = \frac{\|\eparam_2 - \theta^*\|_2 - \|\eparam_1 -\theta^*\|_2}{\|\eparam_1 -\theta^*\|_2}.\]
Roughly, this corresponds to the percentage improvement in the parameter error obtained using $\eparam_1$ over $\eparam_2$. Whenever $\releff(\eparam_1,\eparam_2) > 0$, $\eparam_1$ has a lower parameter error, and higher the value, the more the fractional improvement.

\paragraph{Results.} To reduce the variance in the plots presented here, we averaged results over $20$ repetitions. Figure~\ref{fig:GD1} shows the benefits of using RobustGD over other baselines. 

\begin{itemize}[leftmargin=*]
\item \textbf{Error vs number of iterations:} In Figures~\ref{fig:1a},~\ref{fig:1b} we plot the excess risk of various approaches against the number of iterations (for OLS, LASSO, ridge regression and the method of \citet{hsu2016loss} we only plot the excess risk of the final iterate). We see that upon convergence RobustGD has a much lower parameter error. As expected, OLS-GD converges to OLS.
%
\item \textbf{Error vs number of samples:} 
Next, in Figures~\ref{fig:1c},~\ref{fig:1d} we plot the parameter error as $n/p$ increases. We see that RobustGD is always better than other baselines, even when the number of samples is 12 times the dimension $p$.
\item \textbf{Relative Efficiency vs $\beta$, and $\sigma$:} In Figure~\ref{fig:1e}, we plot the relative efficiency against $\beta$, the moment bound of Pareto distribution. This shows that the percentage improvement in the excess risk by RobustGD decreases as the moment bound $\beta$ increases. This behavior is expected because as we increase the moment bound the tails of the noise distribution become lighter. This shows that there is more benefit in using RobustGD in the heavy tailed setting.  We do a similar study to see the relative efficiency against the variance of the noise distribution. Figure~\ref{fig:1f} plots relative efficiency against standard deviation of the noise distribution.
\end{itemize}
\begin{figure}[!ht]
        \centering
         \subfigure[\label{fig:1a}\small{$n = 512, p = 128, \sigma = 0.75, \beta = 3$}]{\includegraphics[width=0.32\textwidth]{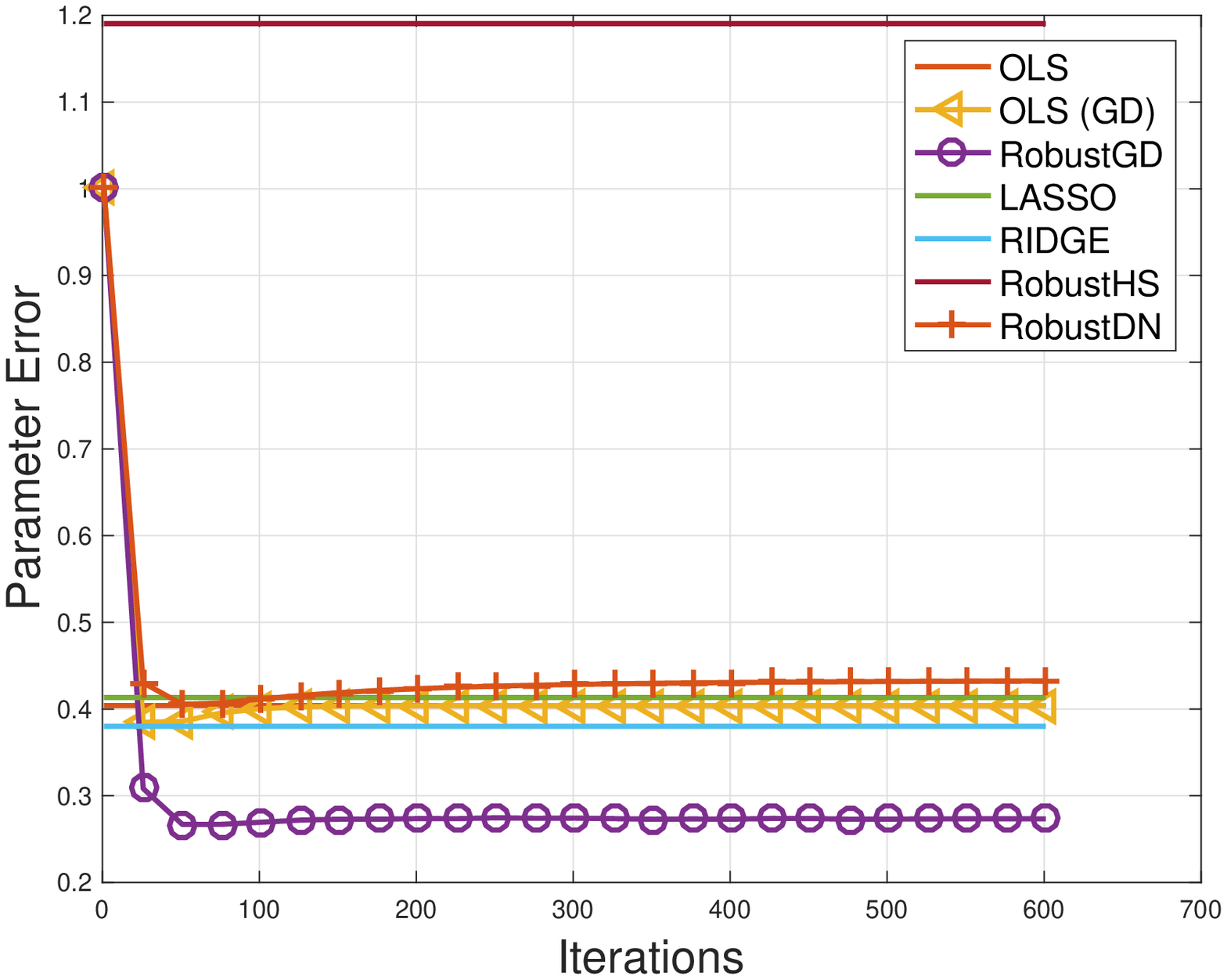}}
        \subfigure[\label{fig:1b}\small{$n = 1024, p = 128, \sigma = 0.75, \beta = 3$}]{\includegraphics[width=0.32\textwidth]{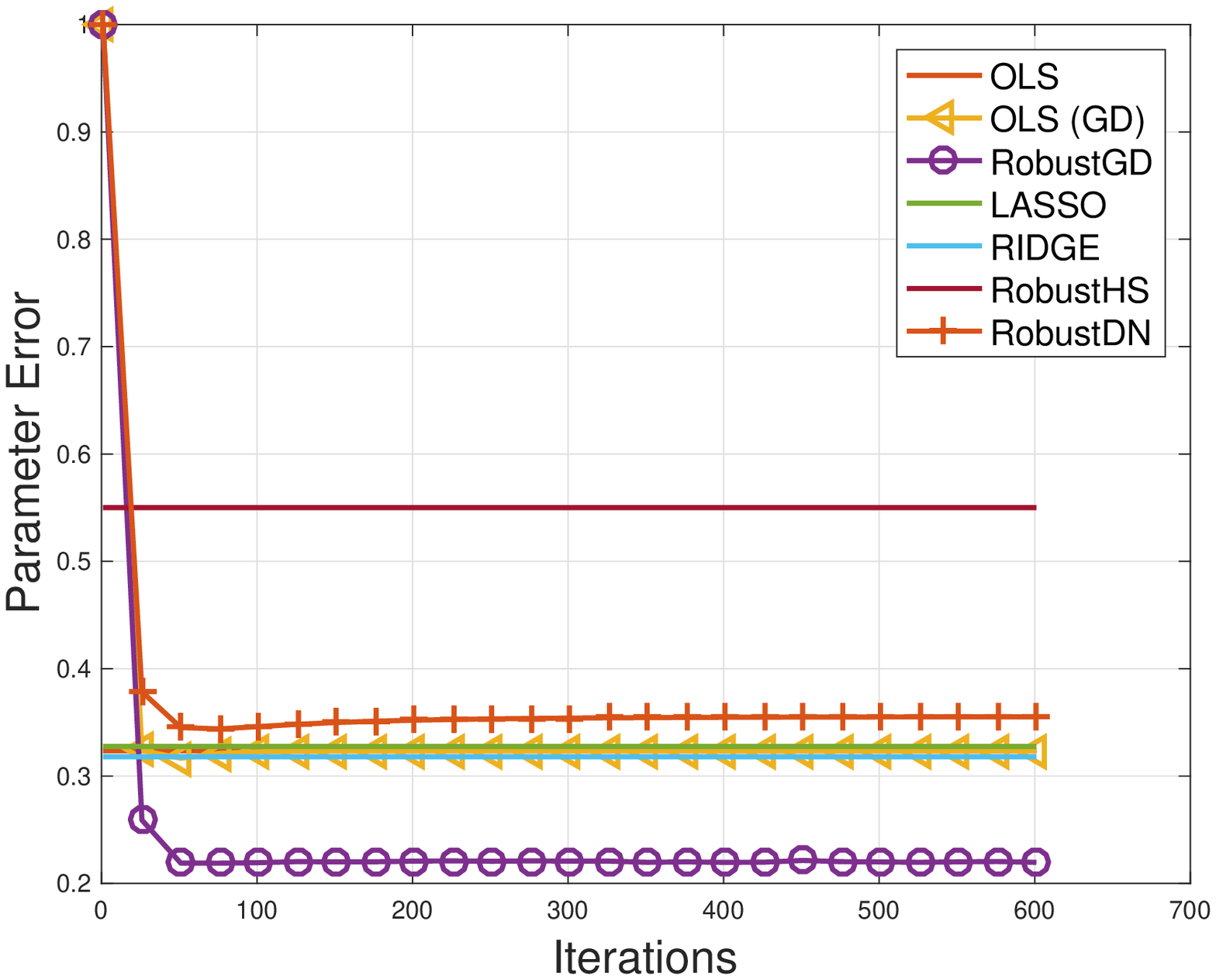}}
        \subfigure[\label{fig:1c}\small{$\sigma = 0.75, \beta = 3$}]{\includegraphics[width=0.32\textwidth]{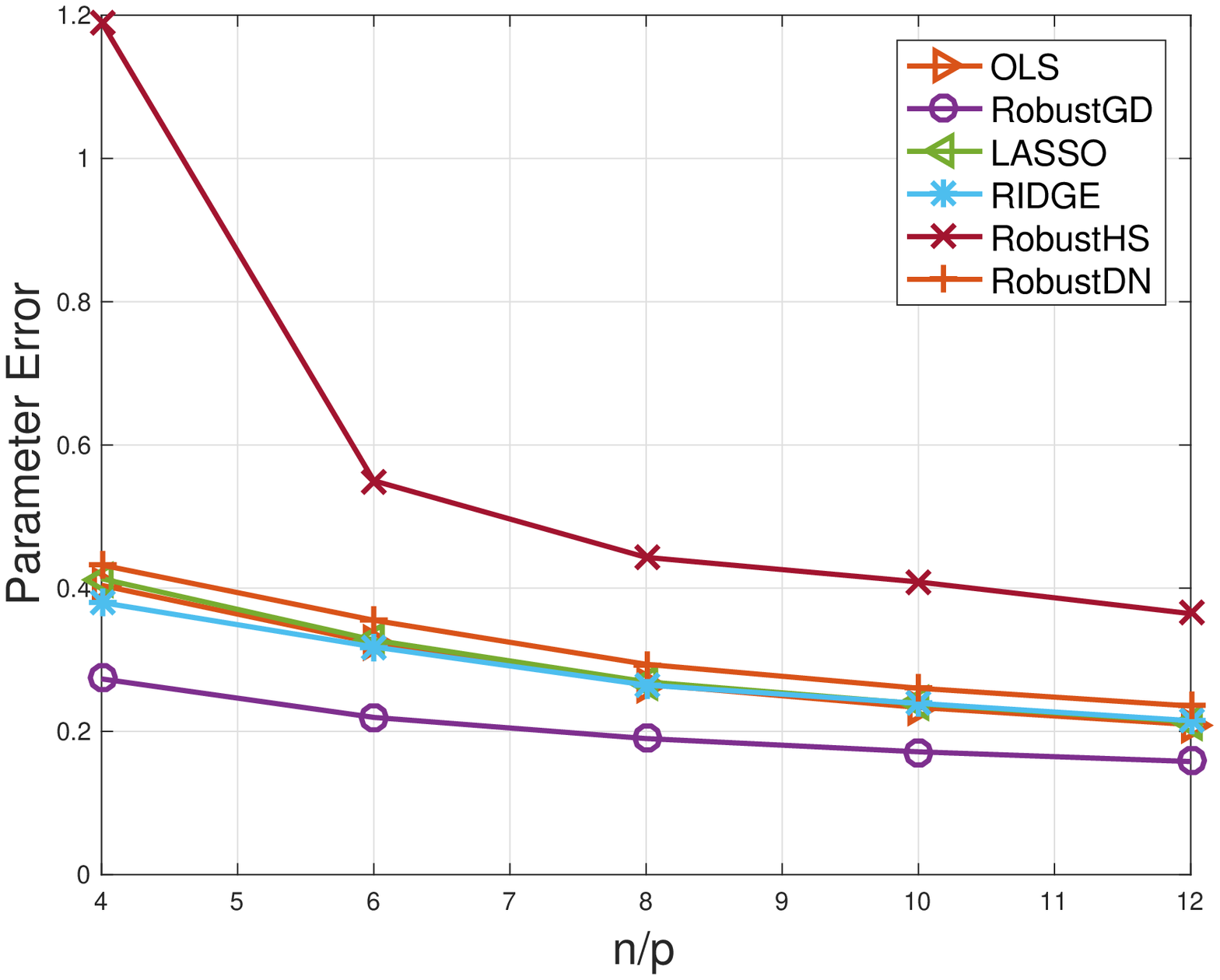}}
        \subfigure[\label{fig:1d}\small{$\sigma = 1, \beta = 3$}]{\includegraphics[width=0.32\textwidth]{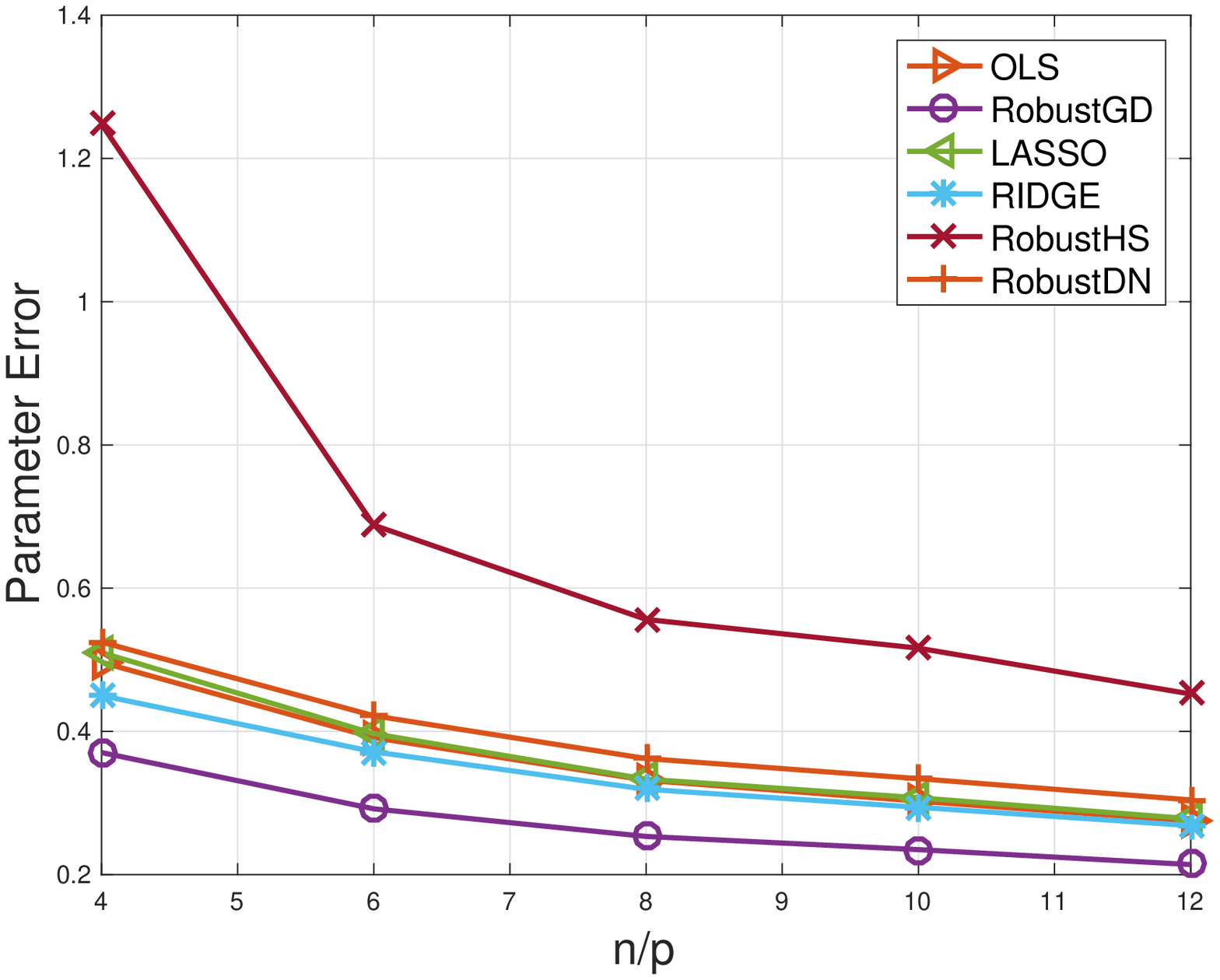}}
        \subfigure[\label{fig:1e}\small{$n = 512, p = 128, \sigma = 0.75$}]{\includegraphics[width=0.32\textwidth]{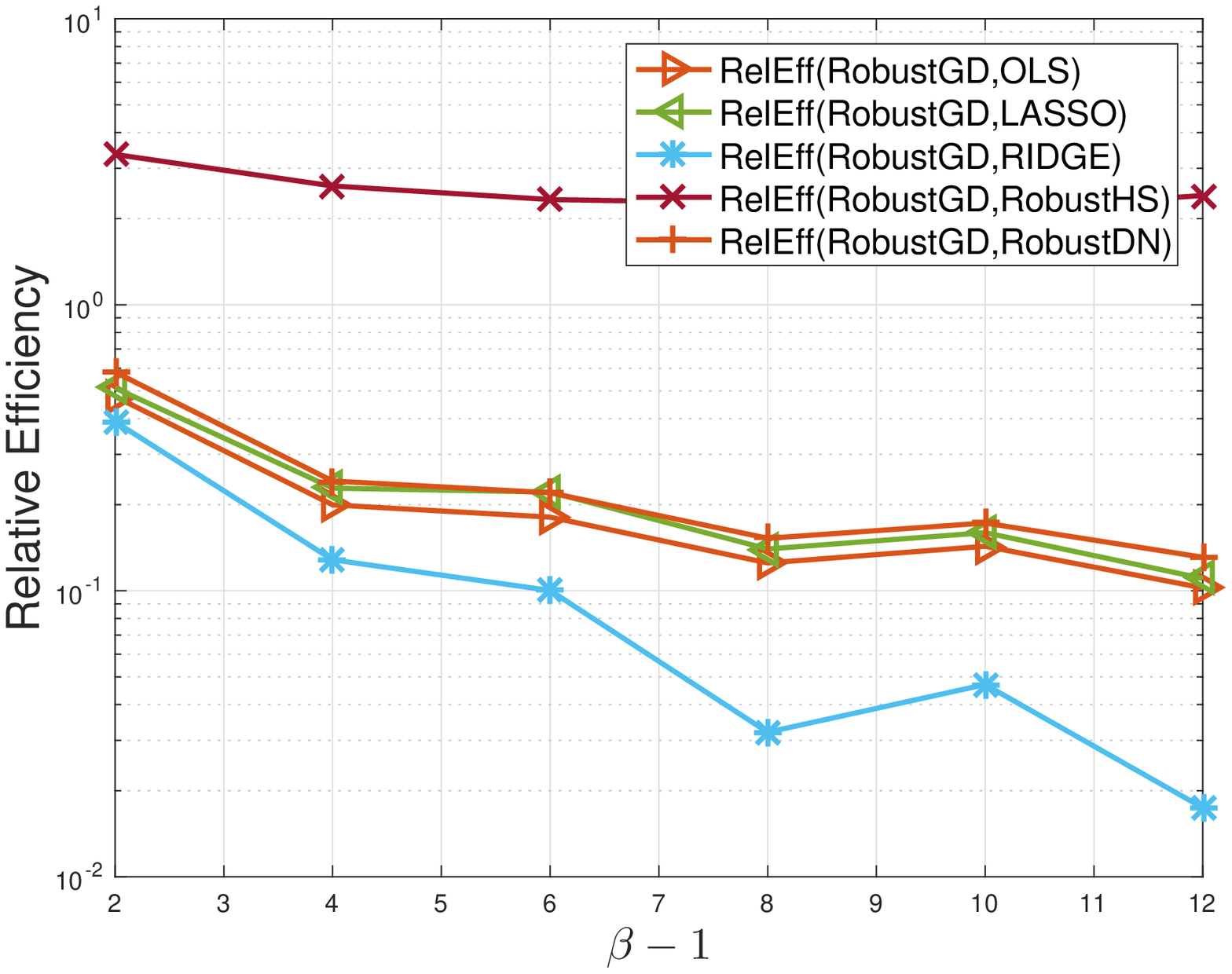}}
        \subfigure[\label{fig:1f}\small{$n = 512, p = 128, \beta = 3$}]{\includegraphics[width=0.32\textwidth]{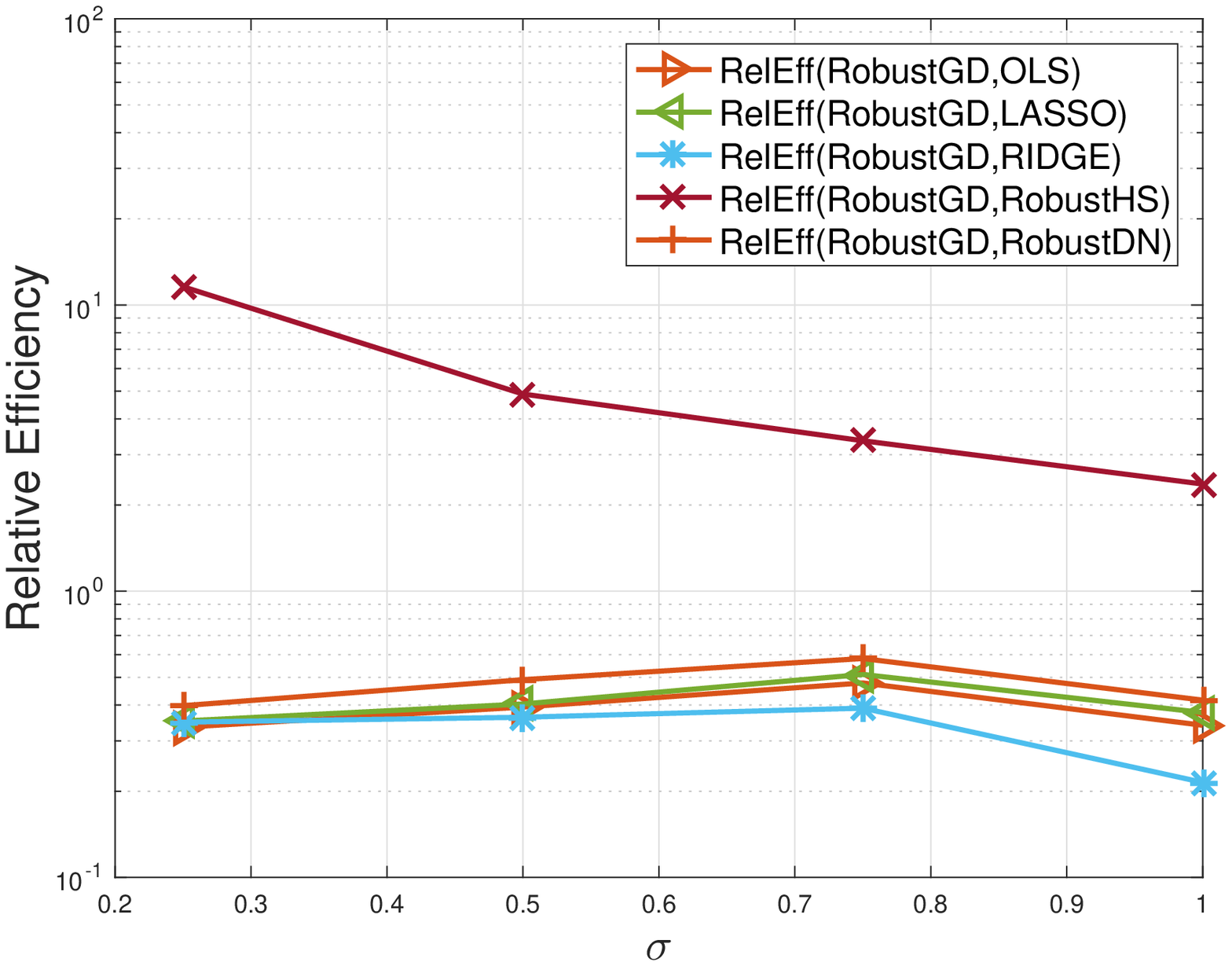}}
\caption{Linear Regression: Performance comparison of RobustGD against baselines.}
\label{fig:GD1}
\end{figure}

%% file: theory.tex
\section{Theoretical Preliminaries}
\label{sec:theory_prelims}
In this section we develop some theoretical preliminaries. 
We first develop a general theory on convergence of projected gradient descent in Section~\ref{sec:grad}. Next we analyze the
gradient estimators defined in Algorithms~\ref{algo:huber_mean_estimation} and \ref{algo:heavy_tailed_estimation} in Sections~\ref{sec:huber} and~\ref{sec:heavy} respectively.
Finally in Sections \ref{sec:consequences_huber} and \ref{sec:consequences_heavy} we present consequences of our general theory for the canonical examples of risk minimization described in Section~\ref{sec:illex}, under Huber contamination and heavy-tailed models.

For some of our examples, we will assume certain mild moment conditions. Concretely, for a random vector $x \in \mathbb{R}^p$, let $\mu = \mathbb{E}[x]$ and $\Sigma$ be the covariance matrix. Then $x$ has bounded $2k^{\text{th}}$ moments if there exists a constant $C_{2k}$ such that for every unit vector $v$ we have that
\begin{equation}
\label{eqn:bounded_moment}
\mathbb{E} \left[ \inprod{x - \mu}{v}^{2k} \right] \leq C_{2k} \left(\mathbb{E} \left[ \inprod{x - \mu}{v}^2\right] \right)^k.
\end{equation}


\subsection{Stability of Gradient Descent}
\label{sec:grad}
In this section we develop a general theory for the convergence of the projected gradient descent described in Algorithm~\ref{algo:festimation}.
Note that our gradient estimators could be biased and are not guaranteed to be consistent estimators of the true gradient $\grad \mathcal{R}(\theta)$. This is especially true in the Huber contamination model where it is impossible to obtain consistent estimators of the gradient of the risk because of the non-vanishing bias caused by the contaminated samples. Hence, we turn our attention to understanding the behavior of projected gradient descent with a biased, inexact, gradient estimator of the form in~\eqnref{eqn:grad_estimator}.
Before we present our main result, we define the notion of stability of a gradient estimator, which plays a key role in the convergence of gradient descent.
\begin{definition}[Stability]
A gradient estimator is \emph{stable} for a given risk function $\mathcal{R}: \Theta \mapsto \mathbb{R}$ if for some $\phi \in [0,\tau_\ell)$,
\begin{align*}
\alpha(\tiln,\tilde{\delta}) < \tau_\ell - \phi.
\end{align*}
\end{definition}
\noindent We denote by $\kappa$ the following contraction parameter:
\begin{align*}
\kappa :=  \sqrt{1- \frac{2\stepSize \tau_\ell \tau_u}{\tau_\ell+\tau_u}} + \stepSize \alpha(\widetilde{n},\widetilde{\delta}), 
\end{align*}
and note that $\kappa < 1$.
With these definitions in place we state our main result on the stability of gradient descent: 
\begin{theorem}
\label{thm:main}
Suppose that the gradient estimator satisfies the condition in~\eqref{eqn:grad_estimator} and
is stable for the risk function $\mathcal{R}: \Theta \mapsto \mathbb{R}$. Then Algorithm~\ref{algo:festimation} initialized at $\theta^0$ with step-size $\eta = 2/(\tau_\ell + \tau_u)$, returns iterates $\{\widehat{\theta}^t\}_{t=1}^{T}$ such that with probability at least $1 - \delta$ for the contraction parameter $\kappa$ above we have that,
\begin{align}\label{eq:main_bounds}
\|\widehat{\theta}^t - \true\|_2 \leq \kappa^t \|\theta^0 - \true\|_2 + \frac{1}{1 - \kappa} \beta(\widetilde{n},\widetilde{\delta}).
\end{align}
\end{theorem}
\noindent We defer a proof of this result to the Appendix. For the bound (\ref{eq:main_bounds}), the first term is decreasing in $T$, while the second term is increasing in $T$. This suggests that for a given $n$ and $\delta$, we need to run just enough iterations for the first term to be bounded by the second. Hence, we can fix the number of iterations $T^*$ as the smallest positive integer such that:
\[ T \geq \log_{1/\kappa} \frac{(1-\kappa) \norm{\theta^0 - \tparam}{2}}{\beta(\widetilde{n},\widetilde{\delta})}. \]
Since we obtain linear convergence, i.e. $\kappa < 1$, typically a logarithmic number of iterations suffice to obtain an accurate estimate.

Theorem~\ref{thm:main} provides a general result for risk minimization and parameter estimation, and requires bounds on $\alpha(\widetilde{n},\widetilde{\delta}),\beta(\widetilde{n},\widetilde{\delta})$ which capture the the error suffered by the gradient estimator for a given risk minimization problem. In any concrete instantiation for a given gradient estimator, risk pair, we first estimate these gradient estimator error bounds by studying the distribution of the gradient of the risk, and then apply Theorem~\ref{thm:main}. In the next two sections, we provide some general analyses of the gradient estimator in Algorithm~\ref{algo:huber_mean_estimation} for the Huber contamination model, and the gradient estimator in Algorithm~\ref{algo:heavy_tailed_estimation} for the heavy-tailed model, and which apply to any risk minimization problem. In Sections~\ref{sec:consequences_huber},\ref{sec:consequences_heavy} we then instantiate these gradient estimator error results for various illustrative statistical models such as linear regression, logistic regression, and general exponential families. Plugging these into Theorem~\ref{thm:main}, we then get consequences of our robustness guarantees for various statistical model, robustness setting pairs.

\subsection{General Analysis of Huber Contamination Gradient Estimator in Algorithm~\ref{algo:huber_mean_estimation}}
\label{sec:huber}
We now analyze the gradient estimator described in Algorithm~\ref{algo:huber_mean_estimation} for Huber contamination model and study the error suffered by it. 
As stated before, Algorithm~\ref{algo:huber_mean_estimation} uses the robust mean estimator of \citet{lai2016agnostic}. Hence, while our proof strategy mimics that of \citet{lai2016agnostic}, we present a different result which is obtained by a more careful non-asymptotic analysis of the algorithm.


We define:
\begin{align}
\label{eqn:gamma}
\gamma(n,p,\delta,\epsilon) := \Big({\frac{p \log p \log \big(n/(p\delta) \big)}{n}}\Big)^{3/8} + \Big(\frac{\epsilon p^2 \log p \log \big( \frac{p \log (p)}{ \delta} \big)}{n}  \Big)^{1/4},
\end{align}
and with this definition in place we have the following result:
\begin{lemma}
\label{lem:lai2016agnostic}
Let $P$ be the true probability distribution of $z$ and let $P_{\theta}$ be the true distribution of the gradients $\grad \poploss (\theta;z)$ on $\mathbb{R}^p$ with mean $\mu_{\theta} = \nabla \mathcal{R}(\theta)$, covariance $\Sigma_{\theta}$, and bounded fourth moments. 
There exists a positive constant $C_1> 0$, 
such that 
given $n$ samples from the distribution in~\eqref{eqn:mixture}, the Huber Gradient Estimator described in Algorithm~\ref{algo:huber_mean_estimation} when instantiated with the contamination level $\epsilon$, with probability at least $1 - \delta $, returns an estimate $\widehat{\mu}$ of $\mu_{\theta}$ such that,
\begin{align*}
\norm{\widehat{\mu} - \mu_{\theta}}{2} \leq C_1 \Big( \sqrt{\epsilon} + \gamma(n,p,\delta,\epsilon) \Big)\norm{\Sigma_{\theta}}{2}^{\half} \sqrt{\log p}. 
\end{align*}
\end{lemma}
\noindent We note in particular, if $n \rightarrow \infty$ (with other parameters held fixed) then $\gamma(n,p,\delta,\epsilon) \rightarrow 0$ and the 
error of our gradient estimator satisfies 
\begin{align*}
\norm{\widehat{\mu} - \mu_{\theta}}{2} \leq C \sqrt{ \norm{\Sigma_{\theta}}{2} \epsilon \log p},
\end{align*} 
and has only 
a weak dependence on the dimension $p$.

\subsection{General Analysis of Heavy-tailed Model Gradient Estimator in Algorithm~\ref{algo:heavy_tailed_estimation}}
\label{sec:heavy}
In this section we analyze the gradient estimator for heavy-tailed setting, described in Algorithm~\ref{algo:heavy_tailed_estimation}.
The following result shows that the gradient estimate has exponential concentration around the true gradient, under the mild assumption that the gradient distribution has bounded second moment. Its proof follows directly from the analysis of geometric median-of-means estimator of \citet{minsker2015geometric}. We use $\trace{\Sigma_{\theta}}$ to denote the trace of the matrix $\Sigma_{\theta}$.
\begin{lemma}
\label{lem:gmom_concentration}
Let $P$ be the probability distribution of $z$ and $P_{\theta}$ be the distribution of the gradients $\grad \poploss (\theta;z)$ on $\mathbb{R}^p$ with mean $\mu_{\theta} = \nabla \mathcal{R}(\theta)$, covariance $\Sigma_{\theta}$. Then the heavy tailed gradient estimator described in Algorithm~\ref{algo:heavy_tailed_estimation} returns an estimate $\widehat{\mu}$ that satisfies the following exponential concentration inequality, with probability at least $1-\delta$:
\[
\|\widehat{\mu} - \mu_{\theta}\|_2 \leq 11\sqrt{\frac{\trace{\Sigma_{\theta}}\log{(1.4/\delta)}}{n}}.
\]
\end{lemma}
\noindent The results of the Lemmas~\ref{lem:lai2016agnostic} and~\ref{lem:gmom_concentration} effectively ensure that under relatively mild moment assumptions we can robustly estimate multivariate mean vectors and in subsequent sections we show how to leverage these strong guarantees for robust parametric estimation.
\section{Consequences for Estimation under  $\epsilon$-Contaminated Model}
\label{sec:consequences_huber}
We now turn our attention to the examples introduced earlier, and present specific applications of Theorem~\ref{thm:main}, for parametric estimation under Huber contamination model. As shown in Lemma~\ref{lem:lai2016agnostic}, we need the added assumption that the true gradient distribution has bounded fourth moments, which suggests the need for additional assumptions. We make our assumptions explicit and defer the technical details to the Appendix. 
\subsection{Linear Regression}
\label{sec:linregtheory}

We assume that the covariates $x \in \real^p$ have bounded $8^{th}$-moments and the noise $w$ has bounded $4^{th}$ moments.  
\begin{theorem}[Robust Linear Regression]
\label{corr:linregress}
Consider the statistical model in equation~\eqref{eq:linregress}, and suppose that the number of samples $n$ is large enough such that $\gamma(\tiln,p,{\widetildelta}) < \frac{C_1\tau_\ell}{\norm{\Sigma}{2} \sqrt{\log p }}$ and the contamination level is such that 
\begin{align*}
\epsilon < \left(\frac{C_2\tau_\ell}{\norm{\Sigma}{2} \sqrt{\log p }}  - \gamma(\tiln,p,{\widetildelta}) \right)^2,
\end{align*} 
for some constants $C_1$ and $C_2$. Then, there are universal constants $C_3,C_4$, such that if Algorithm~\ref{algo:festimation} is initialized at $\theta^0$ with stepsize $\eta = 2/(\tau_u+\tau_\ell)$ and Algorithm~\ref{algo:huber_mean_estimation} as gradient estimator, then it returns iterates $\{\eparam^t\}_{t=1}^{T}$ such that for a contraction parameter $\kappa < 1$, with probability at least $1 - \delta$,
\begin{align}\label{eq:bound:regress}
\norm{\eparam^t - \tparam}{2}\leq \kappa^t \norm{\theta^0 - \tparam}{2} + \frac{C_3 \sigma \sqrt{\norm{\Sigma}{2} \log p}} {{1-\kappa}}  \Big( \epsilon^{\half} +  \gamma(\tiln,p,{\widetildelta}) \Big).  
\end{align}

\end{theorem}
\noindent In the asymptotic setting when the number of samples $n \rightarrow \infty$ (and other parameters are held fixed), we see that for the Huber Gradient Estimator, the corresponding maximum allowed contamination level is 
\begin{align*}
\epsilon < \frac{C_1\tau_\ell^2}{\tau_u^2 \log p},
\end{align*} i.e. the better conditioned the covariance matrix $\Sigma$, the higher the contamination level we can tolerate. 

\paragraph{Plugin Estimation.} For linear regression, the true parameter can be written in closed form as $\tparam = \Exp[xx^T]\inv \Exp[x y]$. A non-iterative way to estimate $\tparam$ is to separately estimate $\Exp[xx^T]$ and  $\Exp[x y]$ using robust covariance and mean oracles respectively. Under the assumption that $x \sim \calN(0,\calI_p)$, one can reduce the problem to robustly estimating $\Exp[x y]$. Under this setting, we now present a result using our robust mean estimator (from Lemma~\ref{lem:lai2016agnostic}) to directly estimate $\Exp[x y]$. Recall, the definition of $\gamma$ in~\eqref{eqn:gamma}. We have the following result: 
\begin{corollary}\label{cor:linregress_plugin}
Consider the model in~\eqref{eq:linregress} with the covariates drawn from $\calN(0,\calI_p)$ and $w \in \calN(0,1)$, then there are universal constants $C_1,C_2$ such that if $\epsilon < C_1$, the plugin estimator $\eparam$ of  $\E[xy]$ described above with probability at least $1 - \delta$ satisfies:
\begin{align}\label{eq:bound:regress_plugin}
\norm{\eparam - \tparam}{2} & \leq   C_2 \sqrt{(1 + 2 \norm{\tparam}{2}^2) \log p} \Big( \epsilon^{\half} + \gamma(n,p,\delta,\epsilon) \Big). \end{align}
\end{corollary}
\noindent Comparing bounds \eqref{eq:bound:regress} and \eqref{eq:bound:regress_plugin}, we see that even when it does not have to estimate the covariance matrix, the error of the plugin estimator depends on $\norm{\tparam}{2}$, which would make the estimator vacuous if $\norm{\tparam}{2}$ scales with the dimension $p$. On the other hand, the asymptotic rate of our robust gradient estimator is independent of $\norm{\tparam}{2}$. This disadvantage of plugin estimation is inescapable, and is seen for instance in known minimax results for robust mean estimation \citep{chen2015robust} that show that the dependence on $\norm{\tparam}{2}$ is unavoidable for any oracle which estimates the mean of $x y$ in the $\epsilon$-contaminated setting. Next, we apply our estimator to generalized linear models.

\subsection{Generalized Linear Models}
\label{sec:glm}
Here we assume that the covariates have bounded 8$^{\text{th}}$ moments. Additionally, we assume smoothness of $\link'(\cdot)$ around $\tparam$. In particular, we assume that there exist universal constants $L_{\link,2k}$, $B_{2k}$ such that
\[ \E_{x} \left[ \left| \link'(\inprod{x}{\theta}) - \link'(\inprod{x}{\tparam}) \right|^{2k} \right] \leq L_{\link,2k} \norm{\true - \theta}{2}^{2k} + B_{\link,2k}, \ \ \text{for } k=1,2,4 \]
We also assume that $\E_{x}[\left| \link^{(t)}(\iprod{x}{\tparam}) \right| ^{k}] \leq M_{\link,t,k}$ where $\link^{(t)}(\cdot)$ is the $t^{th}$-derivative of $\link(\cdot)$. 
\begin{theorem}[Robust Generalized Linear Models]
\label{thm:glm}
Consider the statistical model in equation~\eqref{eqn:glm}, and suppose that the number of samples $n$ is large enough such that 
\begin{align*}
\gamma(\tiln,p,{\widetildelta}) <  \frac{ C_1 \tau_\ell}{\sqrt{\log p } \norm{\Sigma}{2}^\half [ L_{\link,4}^{\frac{1}{4}} + L_{\link,2}^{\half}]},
\end{align*} 
and the contamination level is such that,
\begin{align*}
\epsilon <  \left(  \frac{C_2 \tau_\ell}{\sqrt{\log p } \norm{\Sigma}{2}^\half [ L_{\link,4}^{\frac{1}{4}} + L_{\link,2}^{\half}]} - \gamma(\tiln,p,{\widetildelta}) \right)^2,
\end{align*} for some constants $C_1$ and $C_2$. Then, there are universal constants $C_3,C_4$, such that if Algorithm~\ref{algo:festimation} is initialized at $\theta^0$ with stepsize $\eta = 2/(\tau_u+\tau_\ell)$ and Algorithm~\ref{algo:huber_mean_estimation} as gradient estimator, then it returns iterates $\{\eparam^t\}_{t=1}^{T}$  such that with probability at least $1 - \delta$
\begin{align}\label{eq:bound:glm}
\norm{\eparam^t - \tparam}{2} \leq & \kappa^t \norm{\theta^0 - \tparam}{2} \nonumber \\
& + \frac{C_3 \sqrt{\log p } \norm{\Sigma}{2}^\half [ B_{\link,4}^{\frac{1}{4}} + B_{\link,2}^{\half} + c(\sigma)^\half M_{\link,2,2}^{\frac{1}{4}} + c(\sigma)^{\frac{3}{4}} M_{\link,4,1}^{\frac{1}{4}}]}{1 - \kappa} \Big( \epsilon^{\half} + \gamma(\tiln,p,{\widetildelta}) \Big),    
\end{align}
for some contraction parameter $\kappa < 1$.
\end{theorem}
\noindent Note that for the case of linear regression with gaussian noise, it is relatively straightforward to see that $L_{\link,2k} = C_{2k} \norm{\Sigma}{2}^{k}$, $B_{\link,2k} = 0$,   $M_{\link,t,k} = 1 \ \ \forall (t \geq 2,k \in \calN)$ and $M_{\link,t,k} = 0 \ \ \forall (t \geq 3,k \in \calN)$ under the assumption of bounded $8^{th}$ moments of the covariates; which essentially leads to an equivalence between Theorem~\ref{corr:linregress} and Theorem~\ref{thm:glm} for this setting. In the following section, we instantiate the above Theorem for logistic regression and compare and contrast our results to other existing methods.

\subsubsection{Logistic Regression} 
\label{sec:huber_logistic}
By observing that $\link^{(t)}(\cdot)$ is bounded for logistic regression for all $t \geq 1$, we can see that $L_{\link,2k} = 0$, and that there exists a universal constant $C>0$ such that $B_{\link,2k} < C$ and $M_{\link,t,k} < C \ \ \forall (t \geq 1, k \in \calN)$.

\begin{corollary}[Robust Logistic Regression]
\label{cor:logistic}
Consider the model in equation\eqref{eq:logisticregress}, then there are universal constants $C_1,C_2$, such that if $\epsilon < C_1$, then Algorithm~\ref{algo:festimation} initialized at $\theta^0$ with stepsize $\eta = 2/(\tau_u+\tau_\ell)$ and Algorithm~\ref{algo:huber_mean_estimation} as gradient estimator, returns iterates $\{\eparam^t\}_{t=1}^{T}$,  such that with probability at least $1 - \delta$
\begin{align}
\norm{\eparam^t - \tparam}{2}\leq \kappa^t \norm{\theta^0 - \tparam}{2} + \frac{C_2 \sqrt{\norm{\Sigma}{2} \log p}} {{1-\kappa}}  \Big( \epsilon^{\half} +  \gamma(\tiln,p,{\widetildelta}) \Big),    
\end{align}   
for some contraction parameter $\kappa < 1$.
\end{corollary}
\noindent Under the restrictive assumption that $x \sim \calN(0,\calI_p)$,  \citet{du2017computationally} exploited Stein's trick to derive a plugin estimator for logistic regression. However, similar to the linear regression, the error of the plugin estimator scales with $\norm{\tparam}{2}$, which is avoided in our robust gradient descent algorithm. We also note that our algorithm extends to general covariate distributions.

\subsection{Exponential Family}
Here we assume that the random vector $\phi(z), z \sim P$ has bounded 4$^{\text{th}}$ moments.

\begin{theorem}[Robust Exponential Family]
\label{cor:expFamily_Iterative}
 Consider the model in equation\eqref{eq:expfamily}, then there are universal constants $C_1,C_2$, such that if $\epsilon < C_1$, then Algorithm~\ref{algo:festimation} initialized at $\theta^0$ with stepsize $\eta = 2/(\tau_u+\tau_\ell)$ and Algorithm~\ref{algo:huber_mean_estimation} as gradient oracle, returns iterates $\{\eparam^t\}_{t=1}^{T}$,  such that with probability at least $1 - \delta$
\begin{align}
 \norm{\eparam^t - \tparam}{2}\leq \kappa^t \norm{\theta^0 - \tparam}{2}  + \frac{C_2  \sqrt{\tau_u \log p }}{1-\kappa}  \Big( \epsilon^{\half} + \gamma(\tiln,p,\widetildelta) \Big),   
 \end{align}
for some contraction parameter $\kappa < 1$.
\end{theorem}
\paragraph{Plugin Estimation.} Since the true parameter $\theta^*$ is the minimizer of the negative log-likelihood, we know that $\Exp[\grad \poploss (\tparam)] = 0$, which implies that $\grad A(\tparam) =   \Exp_{\tparam}[\phi(Z)]$. This shows that the true parameter $\tparam$ can be obtained by inverting the $\grad A$ operator, whenever possible. In the robust estimation framework, we can use a robust mean of the sufficient statistics to estimate $\Exp_{\tparam}[\phi(Z)]$. We instantiate this estimator using the mean estimator of \citep{lai2016agnostic} to estimate $\Exp_{\tparam}[\phi(Z)]$:

\begin{corollary}\label{cor:expFamily_Plugin}
Consider the model in equation\eqref{eq:expfamily}, then there are universal constants $C_1,C_2$ such that if $\epsilon < C_1$, then \citep{lai2016agnostic} returns an estimate $\emu$ of  $\E[\phi(z)]$, such that with probability at least $1 - \delta$
\begin{align}\label{eq:bound:expFamily_plugin}
\norm{\calP_{\Theta} \left[ (\grad A)\inv \emu \right] - \tparam}{2} & \leq   C_2 \frac{\sqrt{\tau_u \log p}}{\tau_\ell} \Big( \epsilon^{\half} + \gamma(n,p,\delta,\epsilon) \Big),
\end{align} 
where $\calP_{\Theta} \left[ \theta \right] =  \argmin_{y \in \Theta} \norm{y - \theta}{2}^2$ is the projection operator onto the feasible set $\Theta$.
\end{corollary}

\subsection{Discussion and Limitations}

In the asymptotic setting of $n \rightarrow \infty$, Algorithm~\ref{algo:festimation} with Algorithm~\ref{algo:huber_mean_estimation} as gradient estimator converges to a point $\eparam$ such that $\norm{\eparam - \tparam}{2} = O(\sqrt{\epsilon \log p})$. Hence, our error scales only logarithmically with the dimension $p$. This dependency on the dimension $p$ is a facet of using the estimator from \citet{lai2016agnostic} for gradient estimation. Using better oracles will only improve our performance. Next, we would like to point to the difference in the maximum allowed contamination $\epsilon^*$ between the three models. For logistic regression and exponential family, $\epsilon^* < C_1$, while for linear regression, $\epsilon^* < \frac{C_1\tau_\ell^2}{\tau_u^2 \log p}$. These differences are in large part due to differing variances of the gradients, which naturally depend on the underlying risk function. 
This scaling of the variance of gradients for linear regression also provides insights into the limitations of our robust gradient descent approach in Algorithm~\ref{algo:festimation}. In the Appendix, we provide an upper bound for the contamination level $\epsilon$ based on the initialization point $\theta^0$, above which, Algorithm~\ref{algo:festimation} would not work for any gradient estimator.

\section{Consequences for Heavy-Tailed Estimation}
\label{sec:consequences_heavy}
In this section we present specific applications of Theorem~\ref{thm:main} for parametric estimation, under heavy tailed setting.  The proofs of the results can be found in the Appendix. 
\subsection{Linear Regression}
We first consider the linear regression model described in Equation~\eqref{eq:linregress}. We assume that the covariates $x \in \real^p$ have bounded $4^{th}$-moments and the noise $w$ has bounded $2^{nd}$ moments. This assumption is needed to bound the error in the gradient estimator (see Lemma~\ref{lem:gmom_concentration}).

\begin{theorem}[Heavy Tailed Linear Regression]
\label{corr:heavy_linregress}  Consider the statistical model in equation~\eqref{eq:linregress}.  There are universal constants $C_1, C_2 >0$ such that if $$\tiln > \frac{C_1\tau_u^2}{\tau_l^2}p\log(1/\tildelta)$$ and if Algorithm~\ref{algo:festimation} is initialized at $\theta^0$ with stepsize $\eta = 2/(\tau_u+\tau_\ell)$ and Algorithm~\ref{algo:heavy_tailed_estimation} as gradient estimator, then it returns iterates $\{\eparam^t\}_{t=1}^{T}$  such that with probability at least $1 - \delta$
\begin{align}\label{eq:bound:heavy_regress}
\norm{\eparam^t - \tparam}{2}\leq \kappa^t \norm{\theta^0 - \tparam}{2} + \frac{C_2 \sigma \sqrt{\norm{\Sigma}{2}}} {{1-\kappa}}  \left(\sqrt{\frac{p\log(1/\tildelta)}{\tiln}} \right),    
\end{align}
for some contraction parameter $\kappa < 1$.
\end{theorem}
\subsection{Generalized Linear Models}
In this section we consider generalized linear models described in Equation~\eqref{eqn:glm}, where the covariate $x$ is allowed to have a heavy tailed distribution.  Here we assume that the covariates have bounded 4$^{\text{th}}$ moment. Additionally, we assume smoothness of $\link'(\cdot)$ around $\tparam$. Specifically, we assume that there exist universal constants $L_{\link,2k}$, $B_{2k}$ such that
\[ \E_{x} \left[ \left| \link'(\inprod{x}{\theta}) - \link'(\inprod{x}{\tparam}) \right|^{2k} \right] \leq L_{\link,2k} \norm{\true - \theta}{2}^{2k} + B_{\link,2k}, \ \ \text{for } k=1,2 \]
We also assume that $\E_{x}[\left| \link^{(t)}(\iprod{x}{\tparam}) \right| ^{k}] \leq M_{\link,t,k}$ for $t \in \{1, 2, 4\}$, where $\link^{(t)}(\cdot)$ is the $t^{th}$-derivative of $\link(\cdot)$.

\begin{theorem}[Heavy Tailed Generalized Linear Models]
\label{thm:heavy_glm}
Consider the statistical model in equation~\eqref{eqn:glm}. There are universal constants $C_1, C_2 >0$ such that if
$$\tiln > \frac{C_1\norm{\Sigma}{2}\left(\sqrt{L_{\link,4}} + L_{\link,2}\right)}{\tau_l^2} p\log{1/\tildelta},$$ and if Algorithm~\ref{algo:festimation} is initialized at $\theta^0$ with stepsize $\eta = 2/(\tau_u+\tau_\ell)$ and Algorithm~\ref{algo:heavy_tailed_estimation} as gradient estimator, it returns iterates $\{\eparam^t\}_{t=1}^{T}$  such that with probability at least $1 - \delta$
\begin{align}\label{eq:bound:heavy_glm}
\norm{\eparam^t - \tparam}{2} \leq & \kappa^t \norm{\theta^0 - \tparam}{2} \nonumber \\
& + \frac{C_2\norm{\Sigma}{2}^\half \left[ B_{\link,4}^{\frac{1}{4}} + B_{\link,2}^{\half} + c(\sigma)^\half M_{\link,2,2}^{\frac{1}{4}} + c(\sigma)^{\frac{3}{4}} M_{\link,4,1}^{\frac{1}{4}}\right]}{1 - \kappa} \left(  \sqrt{\frac{p\log(1/\tildelta)}{\tiln}}\right),    
\end{align}
for some contraction parameter $\kappa < 1$.
\end{theorem}
We now instantiate the above Theorem for logistic regression model.
\begin{corollary}[Heavy Tailed Logistic Regression]
\label{cor:heavy_logistic}  
Consider the model in equation\eqref{eq:logisticregress}. There are universal constants $C_1, C_2 >0$ such that if
$$\tiln > \frac{C_1^2\norm{\Sigma}{2}}{\tau_l^2} p\log{1/\tildelta}.$$ and if Algorithm~\ref{algo:festimation} initialized at $\theta^0$ with stepsize $\eta = 2/(\tau_u+\tau_\ell)$ and Algorithm~\ref{algo:heavy_tailed_estimation} as gradient estimator, it returns iterates $\{\eparam^t\}_{t=1}^{T}$  such that with probability at least $1 - \delta$
\begin{align}
\norm{\eparam^t - \tparam}{2}\leq \kappa^t \norm{\theta^0 - \tparam}{2} + \frac{C_2 \sqrt{\norm{\Sigma}{2}}} {{1-\kappa}}  \left(  \sqrt{\frac{p\log(1/\tildelta)}{\tiln}}\right),   
\end{align}   
for some contraction parameter $\kappa < 1$.
\end{corollary}

\subsection{Exponential Family}
We now instantiate Theorem~\ref{thm:main} for parameter estimation in heavy-tailed exponential family distributions. Here we assume that the random vector $\phi(z), z \sim P$ has bounded 2$^{\text{nd}}$ moments, and we obtain the following result:
\begin{theorem}[Heavy Tailed Exponential Family]
\label{cor:heavy_expFamily}
 Consider the model in equation\eqref{eq:expfamily}. If Algorithm~\ref{algo:festimation} is initialized at $\theta^0$ with stepsize $\eta = 2/(\tau_u+\tau_\ell)$ and Algorithm~\ref{algo:heavy_tailed_estimation} as gradient estimator, it returns iterates $\{\eparam^t\}_{t=1}^{T}$,  such that with probability at least $1 - \delta$
\begin{align}
 \norm{\eparam^t - \tparam}{2}\leq \kappa^t \norm{\theta^0 - \tparam}{2}  + \frac{1}{1-\kappa}  C\sqrt{\frac{\norm{\grad^2 A(\tparam)}{2} p\log{1/\tildelta}}{\tiln}},   
 \end{align}
for some contraction parameter $\kappa < 1$ and universal constant $C$.
\end{theorem}


%% file: discussion.tex
\section{Discussion}
\label{sec:discussion}
In this paper we introduced a broad class of robust estimators, that leverage the inherent robustness of gradient descent, together with the observation that for risk minimization in most statistical models, the gradient of the risk takes the form of a simple multivariate mean, which can be robustly estimated using recent work on robust mean estimation. In contrast to classical $M$-estimators that use robust estimates of the risk, our class of estimators employ a shift in perspective, and use robust estimates of gradients of the risk instead, which can then be embedded into a simple projected gradient descent iterative algorithm. Our class of robust gradient descent estimators work well in practice and in many cases outperform other robust (and non-robust) estimators. We also show that these estimators have strong robustness guarantees under varied robustness settings, including Huber's $\epsilon$-contamination model and for heavy-tailed distributions.

There are several avenues for future work, including a better understanding of robust mean estimation, any improvement 
in which would immediately translate to improved guarantees for our robust gradient descent estimators. Finally, it would also be of interest to understand the extent to which we could replace gradient descent with other optimization methods such 
as accelerated gradient descent or Newton's method. We note however, that although these methods may have faster rates of convergence in the classical risk minimization settings, in our setup their stability to using inexact gradients is far more crucial and warrants further investigation.


%% file: ack.tex
\section{Acknowledgements}
A.P., A.S., P.R. acknowledge the support of PNC, and NSF via IIS-1149803, IIS-1664720, DMS-1264033. S.B. acknowledges the support of NSF via DMS-1713003. We thank Larry Wasserman and Ankit Pensia for helpful comments on the paper.

%% file: appendix.tex
\section{Proof of Theorem~\ref{thm:main}}
In this section, we present the proof of our main result on projected gradient descent with an inexact gradient estimator. To ease the notation we will often omit $\{D_n , \delta\}$ from $g(\theta; D_n, \delta)$.
\begin{proof}
At any iteration step $t \in \{1,2,\ldots,T\}$, by assumption we have that with probability at least $1 - \frac{\delta}{T}$,
\begin{equation} \label{eq:1}
 \norm{ g(\theta^t; D_n, \delta/T) - \grad \mathcal{R}(\theta^t) }{2} \leq \alpha(n/T,\delta/T) \norm{\theta - \tparam}{2} + \beta(n/T,\delta/T). 
\end{equation}
Taking union bound, \eqref{eq:1} holds over all iteration steps $t \in \{1 \dots T\}$, with probability at least $1 - \delta$. For the remainder of the analysis, we assume this event to be true. 
\paragraph{Notation.} Let $g(\theta^k) = \grad \mathcal{R}(\theta^k) + e_k$ be the noisy gradient. Let $\alpha = \alpha(n/T,\delta/T)$ and $\beta = \beta(n/T,\delta/T)$ for brevity.

We have the following Lemma from \citet{bubeck2015convex}.
\begin{lemma}\label{lem:bubeck3.11}[Lemma 3.11 \cite{bubeck2015convex}] 
Let $f$ be $M$-smooth and $m$-strongly convex, then for all $x,y \in \real^p$, we have:
\[ \iprod{\grad f(x) - \grad f(y)}{x - y} \geq \frac{mM}{m+M} \norm{x-y}{2}^2 + \frac{1}{m+M} \norm{\grad f(y) - \grad f(x)}{2}^2. \]
\end{lemma}

 By assumptions we have that: $ \norm{\grad \mathcal{R}(\theta^k) - g(\theta^k)}{2} = \norm{e_k}{2} \leq \alpha \norm{\theta^k - \tparam}{2} + \beta$. Our update rule is  $\theta^{k+1} = \mathbb{P}_\Theta \left[ \theta^k - \stepSize g(\theta^k) \right]$.  Then we have that:
 
 \begin{align}
\norm{\theta^{k+1} - \tparam}{2}^2 &= \norm{ \mathbb{P}_\Theta [\theta^k - \stepSize g(\theta^k)]  - \tparam}{2}^2 = \norm{\mathbb{P}_\Theta [\theta^k - \stepSize g(\theta^k)] - \mathbb{P}_\Theta [\tparam - \stepSize \grad R(\tparam)]}{2}^2 \nonumber \\
&\leq  \norm{\theta^k - \stepSize g(\theta^k) - (\tparam - \stepSize \grad R(\tparam))}{2}^2 \label{eqn:projectionContraction} \\
& = \norm{\theta^k - \tparam -  \stepSize (\grad \mathcal{R}(\theta^k) - \grad R(\tparam)) - \eta e_k }{2}^2 \nonumber \\
& \leq \norm{\theta^k - \tparam -  \stepSize (\grad \mathcal{R}(\theta^k) - \grad R(\tparam)) }{2}^2 + \eta^2 \norm{e_k}{2}^2 \nonumber \\
& ~~~~~~~~ + 2 \norm{e_k}{2} \norm{\theta^k - \tparam -  \stepSize (\grad \mathcal{R}(\theta^k) - \grad R(\tparam)) }{2},  \label{eqn:47}
 \end{align}
where Equation~\eqref{eqn:projectionContraction} follows from contraction property of projections. Now, we can write $\norm{\theta^k  - \tparam - \eta(\grad \mathcal{R}(\theta^k) - \grad R(\tparam))}{2}$ as
\small
\begin{align}
 & \norm{\theta^k  - \tparam - \eta(\grad \mathcal{R}(\theta^k) - \grad R(\tparam))}{2}^2 = \norm{\theta^k - \tparam}{2}^2 + \stepSize^2 \norm{\grad \mathcal{R}(\theta^k) - \grad R(\tparam)}{2}^2 \nonumber - 2\stepSize \iprod{\grad \mathcal{R}(\theta^k) - \grad R(\tparam)}{\theta^k - \tparam} \\
 & \leq \norm{\theta^k - \tparam}{2}^2 + \stepSize^2 \norm{\grad \mathcal{R}(\theta^k) - \grad R(\tparam)}{2}^2 -2 \stepSize \left(  \frac{\tau_\ell \tau_u}{\tau_\ell+\tau_u} \norm{\theta^k - \tparam}{2}^2 + \frac{1}{\tau_\ell+\tau_u} \norm{\grad \mathcal{R}(\theta^k) - \grad R(\tparam)}{2}^2 \right) \\
 &= \norm{\theta^k - \tparam}{2}^2 ( 1 - 2\stepSize \tau_\ell \tau_u/(\tau_\ell+\tau_u)) + \stepSize \norm{\grad \mathcal{R}(\theta^k) - \grad R(\tparam)}{2}^2 (\stepSize - 2/(\tau_u+\tau_\ell)) \\
 & = \norm{\theta^k - \tparam}{2}^2 ( 1 - 2\stepSize \tau_\ell \tau_u/(\tau_\ell+\tau_u)) \label{eqn:48},
 \end{align}
 \normalsize
where the second step follows from Lemma~\ref{lem:bubeck3.11} and the last step follows from the step size $\stepSize = 2/(\tau_\ell+\tau_u) $. \\
Now, combining Equations~\eqref{eqn:47} and \eqref{eqn:48}, and using our assumption that $\norm{e_k}{2} \leq \alpha \norm{\theta^k - \tparam}{2} + \beta$, we get:
\[ \norm{\theta^{k+1} - \tparam}{2}^2 \leq \left( \norm{\theta^k - \tparam}{2} \sqrt{( 1-2\stepSize \tau_\ell \tau_u/(\tau_\ell+\tau_u))} + \stepSize \norm{e_k}{2} \right)^2 \]
\[ \norm{\theta^{k+1} - \tparam}{2} \leq \left[ \sqrt{1- 2\stepSize \tau_\ell \tau_u/(\tau_\ell+\tau_u)} +  \stepSize \alpha \right] \norm{\theta^k - \tparam}{2} + \stepSize \beta .\]
Let $\kappa = \sqrt{1- 2\stepSize \tau_\ell \tau_u/(\tau_\ell+\tau_u)} +  \stepSize \alpha$. By the assumption on stability we have $\alpha < \tau_\ell$. 
\begin{align}
\kappa &= \sqrt{1- 2\stepSize \tau_\ell \tau_u/(\tau_\ell+\tau_u)} + \stepSize \alpha \\
&< \sqrt{1- 2\stepSize \tau_\ell \tau_u/(\tau_\ell+\tau_u)} + \stepSize \tau_\ell.
\end{align}
Since $\eta = 2 / (\tau_\ell + \tau_u)$, we get that
\begin{align}
\kappa & < \sqrt{1 - 4\tau_u^2\tau_\ell^2/(\tau_\ell+\tau_u)^2} + 2\tau_\ell/(\tau_u + \tau_\ell) \\ 
\kappa & < \frac{\tau_u - \tau_\ell}{\tau_u + \tau_\ell} + 2\tau_\ell/(\tau_u + \tau_\ell) \\
\kappa & < 1
\end{align}

Therefore, we have that,
\[ \norm{\theta^{k+1} - \tparam}{2} \leq \kappa \norm{\theta^k - \tparam}{2} + \stepSize \beta. \]
for some $\kappa < 1$. Solving the induction,we get:
\[  \norm{\theta^{k} - \tparam}{2} \leq \kappa^k \norm{\theta^0 - \tparam}{2} + \frac{1}{1 - \kappa}\stepSize \beta.  \]
\end{proof}

\section{Proof of Theorem~\ref{corr:linregress}}
The proof of Theorem~\ref{corr:linregress} follows from  Theorem~\ref{thm:glm}, where we study Generalized Linear Models, which include linear regression as a special case. For the case of linear regression with gaussian noise, it is relatively straightforward to see that the smoothness parameters satisfy $L_{\link,2k} = C_{2k} \norm{\Sigma}{2}^{k}$, $B_{\link,2k} = 0$,   $M_{\link,t,k} = 1 \ \ \forall (t \geq 2,k \in \calN)$ and $M_{\link,t,k} = 0 \ \ \forall (t \geq 3,k \in \calN)$ under the assumption of bounded $8^{th}$ moments of the covariates. Substituting these values in Theorem~\ref{thm:glm} gives us the required result.
\section{Proof of Theorem~\ref{thm:glm}}
To prove our result on Robust Generalized Linear Models, we first study the distribution of gradients of the corresponding risk function.

\begin{lemma}\label{lem:glmDist}
Consider the model in Equation~\eqref{eqn:glm}, then there exist universal constants $C_1,C_2>0$ such that 
\begin{align*}
\norm{ \cov (\grad \poploss (\theta) }{2} \leq & C_1 \norm{\Delta}{2}^2 \norm{\Sigma}{2} \left( \sqrt{C_4} \sqrt{L_{\link,4}} + L_{\link,2}\right) \\
 & + C_2 \norm{\Sigma}{2} \left(B_{\link,2} + \sqrt{B_{\link,4}} + c(\sigma) \sqrt{3M_{\link,2,2}} +  \sqrt{c(\sigma)^3 M_{\link,4,1}} \right)  
\end{align*}
\begin{align*}
 \text{Bounded fourth moments} \  \ \Exp \left[ \left[ ( \grad \poploss(\theta)  - \Exp [\grad \poploss(\theta)])^T v  \right]^4 \right] \leq C_2( \var [\grad \poploss(\theta)^T v])^2 .
\end{align*}
\end{lemma}
\begin{proof}
The gradient $\grad \poploss (\theta)$ and it's expectation can be written as:
\begin{align*}
\grad \poploss (\theta) &= - y.x + u(\iprod{x}{\theta}).x \\
\Exp[\grad \poploss (\theta)] &= \Exp [x \left( u(x^T \theta) - u(x^T \tparam) \right) ] ,
\end{align*}
where $u(t) = \link'(t)$.
\begin{align*}
\norm{\Exp[\grad \poploss (\theta)] }{2} &= \sup_{y \in \mathbb{S}^{p-1}} y^T \Exp [\grad \poploss (\theta)]  \\
& \leq \sup_{y \in \mathbb{S}^{p-1}}  \Exp [(y^T x) \left( u(x^T \theta) - u(x^T \tparam) \right) ]  \\
& \leq \sup_{y \in \mathbb{S}^{p-1}} \sqrt{\Exp[(y^T x)^2]} \sqrt{ \Exp [\left( u(x^T \theta) - u(x^T \tparam)\right)^2} ] \\
& \leq C_1 \norm{\Sigma}{2}^\half \sqrt{L_{\link,2}\norm{\Delta}{2}^2 + B_{\link,2}}
\end{align*}
where the last line follows from our assumption of smoothness.

Now, to bound the maximum eigenvalue of the $\cov(\grad \poploss(\theta))$, 
\begin{align*}
\norm{\cov(\grad \poploss(\theta))}{2} &= \sup_{z \in \mathbb{S}^{p-1}} z^T \left( \E \left[ \grad \poploss(\theta) \grad \poploss (\theta)^T \right] - \Exp[\grad \poploss (\theta)]\Exp[\grad \poploss (\theta)]^T \right) z \\
& \leq \sup_{z \in \mathbb{S}^{p-1}} z^T \left( \E \left[ \grad \poploss(\theta) \grad\poploss (\theta)^T \right] \right) z +  \sup_{z\in \mathbb{S}^{p-1}} z^T \left( \Exp[\grad \poploss (\theta)]\Exp[\grad \poploss (\theta)]^T \right) z \\
& \leq \sup_{z \in \mathbb{S}^{p-1}} z^T \left( \E \left[ xx^T \left( u(x^T \theta) - y) \right)^2 \right] \right) z + \norm{\Exp[\grad \poploss (\theta)] }{2}^2 \\
& \leq \sup_{z \in \mathbb{S}^{p-1}} \E \left[ z^T \left(  xx^T \left( u(x^T \theta) - y \right)^2\right) z  \right] + \norm{\Exp[\grad \poploss (\theta)] }{2}^2 \\
& \leq \sup_{z \in \mathbb{S}^{p-1}} \sqrt{\E \left[ (z^T x)^4 \right]} \sqrt{ \E  \left[ \left( u(x^T \theta) - y \right)^4] \right] }+ \norm{\Exp[\grad \poploss (\theta)] }{2}^2
\end{align*}
To bound $\E  \left[ \left( u(x^T \theta) - y \right)^4 \right]$, we make use of the $C_r$ inequality. \\ \\ 
\textbf{${C_r}$ inequality}. If X and Y are random variables such that $\Exp |X|^r < \infty$ and $\Exp |Y|^4 < \infty$ where $r \geq 1$ then:
\[ \Exp |X + Y|^{r} \leq 2^{r-1} \left( \Exp |X|^r + \Exp |Y|^r \right)\]
Using the $C_r$ inequality, we have that
\begin{align*}
\E  \left[ \left( u(x^T \theta) - y \right)^4 \right] & \leq 8 \left( \E  \left[ \left( u(x^T \theta) - u(x^T \true) \right)^4 \right]  + \E  \left[ \left( u(x^T \true) - y \right)^4 \right]  \right) \\
& \leq C \left( L_{\link,4} \norm{\Delta}{2}^4 + B_{\link,4} + c(\sigma)^3 M_{\link,4,1} + 3 c(\sigma)^2 M_{\link,2,2}\right) 
\end{align*}
where the last line follows from our assumption that $P_{\tparam} (y | x)$ is in the exponential family, hence, the cumulants are higher order derivatives of the log-normalization function.
\begin{align*}
& \norm{\cov(\grad \poploss(\theta))}{2}  \leq \sqrt{C} \sqrt{C_4} \norm{\Sigma}{2} \left( \sqrt{L_{\link,4}} \norm{\Delta}{2}^2 + \sqrt{B_{\link,4}} + c(\sigma) \sqrt{3M_{\link,2,2}} + \sqrt{c(\sigma)^3 M_{\link,4,1}} \right) + \norm{\Exp[\grad \poploss (\theta)] }{2}^2 \\
& \leq \sqrt{C} \sqrt{C_4} \norm{\Sigma}{2} \left( \sqrt{L_{\link,4}} \norm{\Delta}{2}^2 + \sqrt{B_{\link,4}} + c(\sigma) \sqrt{3M_{\link,2,2}} + \sqrt{c(\sigma)^3 M_{\link,4,1}}\right) + C_1^2 \norm{\Sigma}{2} \left(  L_{\link,2}\norm{\Delta}{2}^2 + B_{\link,2} \right) \\
& \leq C \norm{\Delta}{2}^2 \norm{\Sigma}{2} \left( \sqrt{C_4} \sqrt{L_{\link,4}} + L_{\link,2}\right) + C_6 \norm{\Sigma}{2} \left(B_{\link,2} + \sqrt{B_{\link,4}} + c(\sigma) \sqrt{3M_{\link,2,2}} +  \sqrt{c(\sigma)^3 M_{\link,4,1}} \right)  
\end{align*}
\textbf{Bounded Fourth Moment.}
To show that the fourth moment of the gradient distribution is bounded, we have
\begin{align*}
 \Exp \left[ \left[ (\grad \poploss(\theta) - \Exp[\grad \poploss(\theta])^T v \right]^4 \right]  &\leq \Exp \left[ \left[ \left|(\grad \poploss(\theta) - \Exp[\grad \poploss(\theta)])^Tv \right| \right]^4  \right]  \\
 & \leq 8 \left[ \underbrace{\Exp [ | \grad \poploss(\theta])^T v |^4] }_{A}+ \underbrace{\Exp [ |\Exp[\grad \poploss(\theta)]^Tv |^4] }_{B} \right] .
\end{align*}
\textbf{Control of A.} 
\begin{align*}
\Exp [ | \grad \poploss(\theta])^T v |^4] &=  \Exp [ (x^T v)^4 (u(x^T \theta) - y)^4] \\
& \leq \sqrt{ \Exp [(x^T v)^8] } \sqrt{ \Exp[(u(x^T \theta) - y)^8]} \\
& \leq \sqrt{C_8} \norm{\Sigma}{2}^2 \sqrt{\E [(u(x^T \theta) - u(x^T \true))^8] + \E[(u(x^T \true) - y)^8]} \\
& \leq \sqrt{C_8} \norm{\Sigma}{2}^2 \sqrt{L_{\link,8} \norm{\Delta}{2}^8 + B_{\link,8} + \sum_{t , k =2}^8 g_{t,k} M_{\link,t,k}}\\
& \leq \sqrt{C} \norm{\Sigma}{2}^2 \sqrt{L_{\link,8}} \norm{\Delta}{2}^4  + \sqrt{B_{\link,8}} + \sqrt{\sum_{t , k =2}^8 g_{t,k} M_{\link,t,k}}
\end{align*}
where the last step follows from the fact that the 8th central moment can be written as a polynomial involving the lower cumulants, which in turn are the derivatives of the log-normalization function.

\textbf{Control of B.}
\begin{align*}
\Exp [ |\Exp[\grad \poploss(\theta)]^Tv |^4] \leq \norm{\E [\grad \poploss (\theta)}{2}^4 \leq C_1 \norm{\Sigma}{2}^2 \left( L_{\link,2}^2 \norm{\Delta}{2}^2 + B_{\link,2}^2 \right)
\end{align*}
By assumption $L_{\link,k}, B_{\link,k},M_{\link,t,k}$ are all bounded for $k,t \leq 8$, which implies that there exist constants $c_1,c_2>0$ such that 
\begin{align}
\Exp \left[ \left[ (\grad \poploss(\theta) - \Exp[\grad \poploss(\theta])^T v \right]^4 \right] \leq   c_1 \norm{\Sigma}{2}^2 \norm{\Delta}{2}^4 + c_2
\end{align}
Previously, we say that $\norm{\cov{\grad \poploss(\theta)}}{2} \leq  c_3 \norm{\Sigma}{2} \norm{\Delta}{2}^2 + c_4$, for some universal constants $c_3,c_4>0$, hence the gradient $\grad \poploss(\theta)$ has bounded fourth moments.
\end{proof}

Having studied the distribution of the gradients, we use Lemma~\ref{lem:lai2016agnostic} to characterize the stability of Huber Gradient estimator. Using Lemma~\ref{lem:lai2016agnostic}, we know that at any point $\theta$, the Huber Gradient Estimator $g(\theta,\delta/T)$ satisfies that with probability $1- \delta/T$, 
 \[ \norm{g(\theta,\delta/T) - \grad \mathcal{R}(\theta)}{2} \leq C_2 \Big( \epsilon^{\half} + \gamma(\widetilde{n},p,\widetilde{\delta}) \Big)\norm{\cov(\grad \poploss (\theta))}{2}^{\half} \sqrt{\log p}.  \]
Substituting the upper bound on $\norm{\cov(\grad \poploss (\theta))}{2}$ from Lemma~\ref{lem:glmDist}, we get that there are universal constants $C_1,C_2$ such that with probability at least 1 - $\delta/T$
\begin{align}\label{eqn:lllo}
 \norm{g(\theta) - \grad \mathcal{R}(\theta)}{2} \leq & \underbrace{C_1 \Big( \epsilon^{\half} + \gamma(\tiln,p,{\widetildelta}) \Big) \sqrt{\log p } \norm{\Sigma}{2}^\half [ L_{\link,4}^{\frac{1}{4}} + L_{\link,2}^{\half}]} _{\alpha(\tiln,\widetildelta)} \norm{\Delta}{2}  \\
 &  + \underbrace{C_2 \Big( \epsilon^{\half} + \gamma(\tiln,p,{\widetildelta}) \Big) \sqrt{\log p } \norm{\Sigma}{2}^\half [ B_{\link,4}^{\frac{1}{4}} + B_{\link,2}^{\half} + c(\sigma)^\half M_{\link,2,2}^{\frac{1}{4}} + c(\sigma)^{\frac{3}{4}} M_{\link,4,1}^{\frac{1}{4}}]}_{\beta(\tiln,\widetildelta)} 
 \end{align}
To ensure stability of gradient descent, we need that $\alpha(\tiln,\widetildelta) < \tau_\ell$. Using Equation~\eqref{eqn:lllo}, we get that gradient descent is stable as long as the number of samples $n$ is large enough such that $\gamma(\tiln,p,{\widetildelta}) <  \frac{ C_1 \tau_\ell}{\sqrt{\log p } \norm{\Sigma}{2}^\half [ L_{\link,4}^{\frac{1}{4}} + L_{\link,2}^{\half}]}$, and the contamination level is such that  \\ $\epsilon <  \left(  \frac{C_2 \tau_\ell}{\sqrt{\log p } \norm{\Sigma}{2}^\half [ L_{\link,4}^{\frac{1}{4}} + L_{\link,2}^{\half}]} - \gamma(\tiln,p,{\widetildelta}) \right)^2$ for some constants $C_1$ and $C_2$. Plugging the corresponding $\epsilon$ and $\beta(\tiln,\widetildelta)$ into Theorem~\ref{thm:main}, we get back the result of Theorem~\ref{thm:glm}.


\section{Proof of Corollary~\ref{cor:linregress_plugin}}
We begin by studying the distribution of the random variable $xy = xx^T \tparam + x.w$.

\begin{lemma}\label{lem:regressionDist_identity}
Consider the model in Equation~\eqref{eq:linregress}, with $x \sim \calN(0,\calI_p)$ and $w \sim \calN(0,1)$ then there exist universal constants $C_1,C_2$ such that
\begin{align*}
\Exp [ x y ]  &= \tparam \\
\norm{ \cov (xy) }{2} &=  1 + 2 \norm{\tparam}{2}^2 \\
\text{Bounded fourth moments}~~~~\Exp & \left[ \left[ ( xy   - \Exp [xy])^T v  \right]^4 \right]
 \leq C_2 ( \var [(xy)^T v])^2.
 \end{align*}
\end{lemma}

\begin{proof} 
\textbf{Mean.}
\begin{align}
xy &= xx^T \tparam + x.w \\
\Exp[xy] &= \Exp[xx^T \tparam + x.w] \\
\Exp[xy]  = \tparam.
\end{align}

\paragraph{Covariance.}
\begin{align}
\cov(xy) &=  \Exp [ (xx^T - I) \tparam + x.w) ((xx^T - I) \tparam + x.w)^T) ] \\
 \cov(xy)&=  \Exp[(xx^T - I)\tparam {\tparam}^T (xx^T - I)] + I_p .
\end{align}
Now, $Z = (xx^T - I)\tparam$ can be written as:
\[  (xx^T - I)\tparam =  \begin{bmatrix}
(x_1^2 - 1)  & x_1x_2 & \ldots & x_1 x_p \\
x_1 x_2 & (x_2^2 - 1) & \ldots & x_2 x_p \\
\vdots &  \vdots & \vdots & \vdots  \\
x_1 x_p & x_2 x_p & \ldots & (x_p^2 - 1)
\end{bmatrix} 
\begin{bmatrix}
\tparam_1 \\ \tparam_2 \\ \vdots \\ \tparam_p
\end{bmatrix}  = \begin{bmatrix}
\tparam_1(x_1^2 - 1 ) + x_1 x_2 \tparam_2 + \ldots + x_1 x_p \tparam_p \\ 
x_1 x_2 \tparam_1 + (x_2^2 - 1) \tparam_2  + \ldots + x_2 x_p \tparam_p \\ 
\vdots \\ 
x_1 x_p \tparam_1 + x_2 x_p \tparam_2  + \ldots + (x_p^2 - 1) \tparam_p 
\end{bmatrix} . \]
Then,
\[ \Exp \left[ ZZ^T \right] =  \begin{bmatrix}
2 {\tparam_1}^2 + {\tparam_2}^2 + \ldots + {\tparam_p}^2  &  \tparam_1 \tparam_2 & \ldots & \tparam_1 \tparam_p \\
\tparam_1\tparam_2 &  {\tparam_1}^2 + 2 {\tparam_2}^2 + \ldots + {\tparam_p}^2 & \ldots & \tparam_2 \tparam_p \\
\vdots & \vdots & \ddots & \vdots \\
\tparam_p \tparam_1 & \tparam_2 \tparam_p &  \ldots &  {\tparam_1}^2 + {\tparam_2}^2 + \ldots + 2{\tparam_p}^2
\end{bmatrix} .\]
Hence the covariance matrix can be written as:
\[\cov(xy) = I_p(1 + \norm{\tparam}{2}^2) + \tparam {\tparam}^T. \]
Therefore $ \norm{\cov (xy)}{2} = 1+2 \norm{\tparam}{2}^2$.

\paragraph{Bounded Fourth Moment.} We start from the LHS

\begin{align}
 \Exp \left[ \left[ (xy - \Exp[xy])^T v \right]^4 \right] & \leq \Exp \left[ \left[ \left|(xy - \Exp[xy])^Tv \right| \right]^4  \right] \\
&=  \Exp \left[ \left|( (xx^T - I)\tparam + w x)^T v \right| \right]^4 \\
& =  \Exp \left[ \left| ({\tparam}^T x) (x^T v) - {\tparam}^T v +  w v^T x \right| \right]^4 \\
& \leq 8 \left[ 8 \left[ \underbrace{\Exp \left| ({\tparam}^Tx) (x^T v) \right|^4}_{A} +  \underbrace{\Exp \left| {\tparam}^T v\right|^4}_{B} \right] + \Exp \underbrace{\left| w(x^Tv) \right|^4}_{C} . \right] 
\end{align}
The last line follows from two applications of the following inequality:\\ \\
\textbf{${C_r}$ inequality}. If X and Y are random variables such that $\Exp |X|^r < \infty$ and $\Exp |Y|^4 < \infty$ where $r \geq 1$ then:
\[ \Exp |X + Y|^{r} \leq 2^{r-1} \left( \Exp |X|^r + \Exp |Y|^r \right).\]
 
 Now to control each term on:
\bit
\item \textbf{Control of $A$}. Using Cauchy Schwartz, and normality of 1D projections of normal distribution  
\begin{align}
A & \leq \sqrt{\Exp[|{\tparam}^T x|^8]} \sqrt{\Exp[|x^T v|^8]} \\
 & \precsim \norm{\tparam}{2}^4.
\end{align}
\item \textbf{Control of $B$}, $B \leq \norm{\tparam}{2}^4$.
\item  \textbf{Control of $C$}, $C  = O(1)$, using independence of $w$ and normality of 1D projections of normal distribution. 
\eit
Therefore the $ \Exp \left[ \left[ (xy - \Exp[xy])^T v \right]^4 \right] \precsim c + \norm{\tparam}{2}^4$. \\

For the RHS:
\[ \var((xy)^T v)^2 = (v^T \cov(xy) v)^2 \leq \norm{\cov(xy)}{2}^2  .\]
We saw that the $\norm{\cov(xy)}{2} \precsim c + \norm{\tparam}{2}^2$, so both the LHS and RHS scale with $\norm{\tparam}{2}^4$. Hence, $xy$ has bounded fourth moments.
\end{proof}
Now that we've established that $xy$ has bounded fourth moments implies that we can use \citep{lai2016agnostic} as a mean estimation oracle. Using Theorem~1.3 \citep{lai2016agnostic}, we know that the oracle of \citep{lai2016agnostic} outputs an estimate $\eparam$ of $\Exp[xy]$ such that with probability at least $1 - 1/p^{C_1}$, we have:
\[ \norm{\eparam - \tparam}{2} \leq   C_2 \sqrt{ \norm{\cov(xy)}{2}  \log p} \Big( \epsilon^{\half} + \gamma(n,p,\delta,\epsilon) \Big) \]
 Using Lemma~\ref{lem:regressionDist_identity} to subsitute $\norm{\cov(xy)}{2} \leq 1 + 2 \norm{\tparam}{2}^2)$, we recover the statement of Corollary~\ref{cor:linregress_plugin}.

 \section{Proof of Theorem~\ref{cor:expFamily_Iterative}}
 To prove our result on Robust Exponential Family, we first study the distribution of gradients of the corresponding risk function. 

\begin{lemma}\label{lem:expFamilyGrad}
Consider the model in Equation~\eqref{eq:expfamily}, then there exists a universal constant $C_1$ such that
\begin{align}
& \Exp[\grad \poploss(\theta)] = \grad A (\theta) - \grad A(\tparam) \\
& \norm{\cov[\grad \poploss(\theta)]}{2} = \norm{\grad^2 A(\tparam)}{2}  \\
\text{Bounded fourth moments} \  \ & \Exp \left[ \left[ ( \grad \poploss(\theta)  - \Exp [\grad \poploss(\theta)])^T v  \right]^4 \right] \leq C_1( \var [\grad \poploss(\theta)^T v])^2 .
\end{align}
\end{lemma}

 \begin{proof}
By Fisher Consistency of the negative log-likelihood, we know that
 \begin{align}
 \Exp_{\tparam}[\grad \poploss(\tparam)] =  0  \\
 \implies \grad A(\tparam) - \Exp_{\tparam}[\phi(z)]  = 0 \\
 \implies \grad A(\tparam)  = \Exp_{\tparam}[\phi(z)] .
\end{align}  
For the mean, 
\begin{align}
\grad \poploss (\theta) & = \grad A(\theta) - \phi(z) \\
\Exp[\grad \poploss(\theta)] & = \grad A(\theta) - \Exp_{\tparam}[\phi(z)]  \\
\Exp[\grad \poploss(\theta)] & = \grad A(\theta) - \grad A(\tparam).
\end{align}
Now, for the covariance:
\begin{align*}
\cov_{\tparam}  [\grad \poploss(\theta)] & = \Exp_{\tparam} \left[ \left(\grad \poploss(\theta) - \Exp_{\tparam}[\grad \poploss(\theta)] \right) \left(\grad \poploss(\theta) - \Exp_{\tparam}[\grad(\poploss(\theta)] \right)^T \right] \\
&= \Exp_{\tparam} \left[ \left( \grad A(\tparam) - \phi(z) \right) \left( \grad A(\tparam) - \phi(z) \right)^T \right] \\
&= \cov_{\tparam} \left[ \grad \poploss (\tparam) \right] = \grad^2 A (\tparam).
\end{align*}
Bounded moments follows from our assumption that the sufficient statistics have bounded 4th moments.
\end{proof}

Having studied the distribution of the gradients, we use Lemma~\ref{lem:lai2016agnostic} to characterize the stability of Huber Gradient estimator. Using Lemma~\ref{lem:lai2016agnostic}, we know that at any point $\theta$, the Huber Gradient Estimator $g(\theta,\delta/T)$ satisfies that with probability $1- \delta/T$, 
 \[ \norm{g(\theta,\delta/T) - \grad \mathcal{R}(\theta)}{2} \leq C_2 \Big( \epsilon^{\half} + \gamma(\widetilde{n},p,\widetilde{\delta}) \Big)\norm{\cov(\grad \poploss (\theta))}{2}^{\half} \sqrt{\log p}.  \]
Substituting the upper bound on $\norm{\cov(\grad \poploss (\theta))}{2}$ from Lemma~\ref{lem:expFamilyGrad}, we get that there are universal constants $C_1,C_2$ such that 
 \begin{align*}
 \norm{g(\theta) - \grad \mathcal{R}(\theta)}{2} \leq \underbrace{C_1 \Big( \epsilon^{\half} + \gamma(\tiln,p,\widetildelta) \Big) \sqrt{\log p} \sqrt{\tau_u}}_{\beta(\tiln,\widetildelta)} .
 \end{align*}
 
 In this case we have that $\alpha(\tiln,\widetildelta) = 0 < \tau_\ell$ by assumption. Therefore we just have that $\epsilon < C_1$ for some universal constant $C_1$. Plugging the corresponding $\epsilon$ and $\beta(\tiln,\widetildelta)$ into Theorem~\ref{thm:main}, we get back the result of Corollary~\ref{cor:expFamily_Iterative}.
 
 \section{Proof of Corollary~\ref{cor:expFamily_Plugin}}
Using the contraction property of projections, we know that
\[ \norm{\calP_{\Theta} \left[ (\grad A)\inv \emu \right] - \tparam}{2} =  \norm{\calP_{\Theta} \left[ (\grad A)\inv \emu \right] - \calP_{\Theta} \left[ \tparam \right] }{2} \leq  \norm{(\grad A)\inv \emu - \tparam}{2}.\]

By Fisher Consistency of the negative log-likelihood, we know that
\[ \grad A(\tparam) =   \Exp_{\tparam}[\phi(z)]. \]
The true parameter $\tparam$ can be obtained by inverting the $\grad A$ operator whenever possible. 
\begin{align}
\norm{(\grad A)\inv \emu - \tparam}{2} &= \norm{(\grad A)\inv \emu - (\grad A)\inv\Exp_{\tparam}[\phi(z)] }{2} \\
& = \norm{\grad A^* \emu - \grad A^*\Exp_{\tparam}[\phi(z)] }{2}.
\end{align}{
where $A^*$ is the convex conjugate of $A$. We can use the following result to control the Lipschitz smoothness $A^*$.

\begin{theorem}\label{thm:convexDuality}(Strong/Smooth Duality)
Assume $f(\cdot)$ is closed and convex. Then $f(\cdot)$ is smooth with
parameter $M$ if and only if its convex conjugate $f(\cdot)$ is strongly convex with parameter $m = \frac{1}{M}$.
\end{theorem}
A proof of the above theorem can be found in \cite{kakade2009applications}. Hence, we have that:
\begin{equation}\label{eq:111}
\norm{\calP_{\Theta} \left[ (\grad A)\inv \emu \right] - \tparam}{2} \leq \frac{1}{\tau_\ell} \norm{\emu - \Exp_{\tparam}[\phi(z)] }{2} 
\end{equation} 
By assumption, we have that the fourth moments of the sufficient statistics are bounded. We also know that $\cov(\phi(z) = \grad^2 A(\tparam)$ which implies that we can use \citep{lai2016agnostic} as our oracle. Using Lemma~\ref{lem:lai2016agnostic}, we get that, there exists universal constants $C_1,C_2$ such that with probability at least $ 1 - 1/p^{C_1}$, 
\[ \norm{\emu - \Exp_{\tparam}[\phi(z)] }{2} \leq  C_2 \sqrt{\tau_u \log p} \Big( \epsilon^{\half} + \gamma(n,p,\delta,\epsilon) \Big).\]
Combining the above with Equation~\eqref{eq:111} recovers the result of Corollary~\ref{cor:expFamily_Plugin}.

\section{Proof of Theorem~\ref{corr:heavy_linregress}}
Before we present the proof of Theorem~\ref{corr:heavy_linregress}, we first study the distribution of gradients of the loss function. This will help us bound the error in the gradient estimator.

\begin{lemma}\label{lem:heavyTailedRegressionDist}
Consider the model in Equation~\eqref{eq:linregress}. Suppose the covariates $x \in \real^p$ have bounded $4^{th}$-moments and the noise $w$ has bounded $2^{th}$ moments. Then there exist universal constants $C_1,C_2$ such that
\begin{align*}
\Exp [ \grad \poploss (\theta) ]  &=  \Sigma \Delta \\
\norm{ \cov (\grad \poploss(\theta) }{2} &\leq  \sigma^2 \norm{\Sigma}{2} + C_1 \norm{\Delta}{2}^2\norm{\Sigma}{2}^2,
 \end{align*}
where $\Delta = \theta - \tparam$ and $E[xx^T] = \Sigma$.
\end{lemma}
\begin{proof} We start by deriving the results for $\Exp [ \grad \poploss (\theta) ]$.
\begin{align}
\poploss(\theta) =  \half (y - x^T \theta)^2 = \half (x^T(\Delta) - w)^2 \\
\grad \poploss(\theta) =  xx^T \Delta -  x.w \\
 \Exp [\grad \poploss (\theta) ] =  \Sigma \Delta .
\end{align}
Next, we bound the operator norm of the covariance of the gradients $\grad \poploss (\theta)$ at any point $\theta$.
\paragraph{Covariance.}
\begin{align}
\cov(\grad \poploss(\theta)) &=  \Exp [ (xx^T - \Sigma) \Delta - x.w) ((xx^T - \Sigma) \Delta - x.w)^T) ] \\
 \cov(\grad \poploss(\theta))&=  \Exp[(xx^T - \Sigma)\Delta \Delta^T (xx^T - \Sigma)] + \sigma^2 \Sigma .
\end{align}
Now, we want to bound $\norm{\cov(\grad \poploss(\theta))}{2} = \lambda_{\max}( \cov(\grad \poploss(\theta)))$.
\begin{align}
\lambda_{\max}( \cov(\grad \poploss(\theta))) &\leq \sigma^2 \lambda_{\max}(\Sigma) + \lambda_{\max}\left(\Exp[(xx^T - \Sigma)\Delta \Delta^T (xx^T - \Sigma)]\right)\\
& \leq \sigma^2 \lambda_{\max}(\Sigma) + \sup_{y \in \mathbb{S}^{p-1} } y^T\left(\Exp[(xx^T - \Sigma)\Delta \Delta^T (xx^T - \Sigma)]\right) y \\
& \leq \sigma^2 \lambda_{\max}(\Sigma) + \sup_{y \in \mathbb{S}^{p-1} } y^T\left(\Exp[(xx^T - \Sigma)\Delta \Delta^T (xx^T - \Sigma)]\right) y \\
& \leq \sigma^2 \lambda_{\max}(\Sigma) + \norm{\Delta}{2}^2 \sup \limits_{y,z \in \mathbb{S}^{p-1}} \Exp \left[ (y^T (xx^T - \Sigma) z)^2 \right] \\
& \leq \sigma^2 \lambda_{\max}(\Sigma) + \norm{\Delta}{2}^2 \sup \limits_{y,z \in \mathbb{S}^{p-1}} \left( \Exp \left[ 2 (y^T x)^2 (x^T z) ^2  + 2 (y^T \Sigma z)^2 \right] \right) \\
&\leq \sigma^2 \lambda_{\max}(\Sigma) +2 \norm{\Delta}{2}^2 \sup \limits_{y,z \in \mathbb{S}^{p-1}} \left( \Exp \left[ (y^T x)^2 (x^T z) ^2 \right]  +  \norm{\Sigma}{2}^2 \right) \\
&\leq \sigma^2 \lambda_{\max}(\Sigma) +2 \norm{\Delta}{2}^2 \sup \limits_{y,z \in \mathbb{S}^{p-1}} \left( \Exp \left[ (y^T x)^2 (x^T z) ^2 \right]  +  \norm{\Sigma}{2}^2 \right) \\
& \leq \sigma^2 \lambda_{\max}(\Sigma) +2 \norm{\Delta}{2}^2 \sup \limits_{y,z \in \mathbb{S}^{p-1}} \left( \sqrt{\Exp \left[ (y^T x)^4 \right] }  \sqrt{\Exp \left[ (z^T x)^4 \right] }  +  \norm{\Sigma}{2}^2 \right) \\
& \leq \sigma^2 \norm{\Sigma}{2} + 2\norm{\Delta}{2}^2 ( \norm{\Sigma}{2}^2 + C_4 \norm{\Sigma}{2}^2) ,
\end{align}
where the second last step follows from Cauchy-Schwartz and the last step follows from our assumption of bounded 4$^{th}$ moments (see Equation~\eqref{eqn:bounded_moment}).
\end{proof}%
We now proceed to the proof of Theorem \ref{corr:heavy_linregress}. From Lemma~\ref{lem:gmom_concentration}, we know that at any point $\theta$, the gradient estimator described in Algorithm~\ref{algo:heavy_tailed_estimation}, $g(\theta; D_{\tiln}, \tildelta)$,  satisfies the following with probability at least $1-\delta$, 
\begin{equation*}
 \begin{array}{lll}
 \norm{g(\theta; D_{\tiln}, \tildelta) - \grad \mathcal{R}(\theta)}{2} &\leq& C\sqrt{\frac{\text{tr}(\text{Cov}(\nabla \poploss (\theta)))\log{1/\tildelta}}{\tiln}}.
 \end{array}
 \end{equation*}
We substitute the upper bound  for $\norm{\cov(\grad \poploss(\theta))}{2}$ from Lemma~\ref{lem:heavyTailedRegressionDist} in the above equation
 \begin{equation*}
 \begin{array}{lll}
 \norm{g(\theta; D_{\tiln}, \tildelta) - \grad \mathcal{R}(\theta)}{2} &\leq& C\sqrt{\frac{\text{tr}(\text{Cov}(\nabla \poploss (\theta)))\log{1/\tildelta}}{\tiln}} \vspace{0.1in}\\
 &\leq& C\sqrt{\frac{p\left(\sigma^2 \norm{\Sigma}{2} + 2\norm{\Delta}{2}^2 ( \norm{\Sigma}{2}^2 + C_4 \norm{\Sigma}{2}^2)\right)\log{1/\tildelta}}{\tiln}} \vspace{0.1in}\\
 &\leq& \underbrace{C_1\sqrt{\frac{\norm{\Sigma}{2}^2 p\log{1/\tildelta}}{\tiln}}}_{\alpha(\tiln,\tildelta)} \norm{\theta - \tparam}{2} \\
  &&+ \underbrace{C_2  \sigma \sqrt{\frac{\norm{\Sigma}{2} p\log{1/\tildelta}}{\tiln}}}_{\beta(\tiln,\tildelta)}.
 \end{array}
 \end{equation*}
To complete the proof of this theorem, we use the results from Theorem~\ref{thm:main}.
Note that the  gradient estimator satisfies the stability condition if $\alpha(\tiln,\tildelta) < \tau_l$. This holds when $$\tiln > \frac{C_1^2\tau_u^2}{\tau_l^2} p\log{1/\tildelta}.$$ Now suppose $\tiln$ satisfies the above condition, then plugging $\beta(\tiln,\tildelta)$ into Theorem~\ref{thm:main} gives us the required result. 

\section{Proof of Theorem~\ref{thm:heavy_glm}}
To prove the Theorem we use the result from Lemma \ref{lem:glmDist}, where we derived the following expression for covariance of $\nabla \poploss (\theta)$
\begin{align*}
\norm{ \cov (\grad \poploss (\theta) }{2} \leq & C_1 \norm{\Delta}{2}^2 \norm{\Sigma}{2} \left( \sqrt{C_4} \sqrt{L_{\link,4}} + L_{\link,2}\right) \\
 & + C_2 \norm{\Sigma}{2} \left(B_{\link,2} + \sqrt{B_{\link,4}} + c(\sigma) \sqrt{3M_{\link,2,2}} +  \sqrt{c(\sigma)^3 M_{\link,4,1}} \right)  
\end{align*}
From Lemma~\ref{lem:gmom_concentration}, we know that at any point $\theta$, the gradient estimator described in Algorithm~\ref{algo:heavy_tailed_estimation}, $g(\theta; D_{\tiln}, \tildelta)$,  satisfies the following with probability at least $1-\delta$, 
\begin{equation*}
 \begin{array}{lll}
 \norm{g(\theta; D_{\tiln}, \tildelta) - \grad \mathcal{R}(\theta)}{2} &\leq& C\sqrt{\frac{\text{tr}(\text{Cov}(\nabla \poploss (\theta)))\log{1/\tildelta}}{\tiln}}.
 \end{array}
 \end{equation*}
Substituting the upper bound  for $\norm{\cov(\grad \poploss(\theta))}{2}$ in the above equation, we get
 \begin{equation*}
 \begin{array}{lll}
 \norm{g(\theta; D_{\tiln}, \tildelta) - \grad \mathcal{R}(\theta)}{2} &\leq& C\sqrt{\frac{\text{tr}(\text{Cov}(\nabla \poploss (\theta)))\log{1/\tildelta}}{\tiln}} \vspace{0.1in}\\
 &\leq& \underbrace{C_1\sqrt{\frac{\norm{\Sigma}{2} \left( \sqrt{C_4} \sqrt{L_{\link,4}} + L_{\link,2}\right) p\log{1/\tildelta}}{\tiln}}}_{\alpha(\tiln,\tildelta)} \norm{\theta - \tparam}{2} \\
  &&+ \underbrace{C_2  \sqrt{\frac{\norm{\Sigma}{2} \left(B_{\link,2} + \sqrt{B_{\link,4}} + c(\sigma) \sqrt{3M_{\link,2,2}} +  \sqrt{c(\sigma)^3 M_{\link,4,1}} \right)   p\log{1/\tildelta}}{\tiln}}}_{\beta(\tiln,\tildelta)}.
 \end{array}
 \end{equation*}
 We now use the results from Theorem~\ref{thm:main}.  The  gradient estimator satisfies the stability condition if $\alpha(\tiln,\tildelta) < \tau_l$. This holds when $$\tiln > \frac{C_1^2\norm{\Sigma}{2}\left( \sqrt{C_4} \sqrt{L_{\link,4}} + L_{\link,2}\right)}{\tau_l^2} p\log{1/\tildelta}.$$ Now suppose $\tiln$ satisfies the above condition, then plugging $\beta(\tiln,\tildelta)$ into Theorem~\ref{thm:main} gives us the required result.
\section{Proof of Theorem~\ref{cor:heavy_expFamily}}
The proof proceeds along similar lines as the proof of Theorem~\ref{thm:heavy_glm}. To prove the Theorem we utilize the result of Lemma~\ref{lem:expFamilyGrad}, where we showed that
$\norm{\cov[\grad \poploss(\theta)]}{2} = \norm{\grad^2 A(\tparam)}{2}$. 
Combining this result with Lemma~\ref{lem:gmom_concentration} we get that with probability at least $1-\delta$
\begin{equation*}
 \begin{array}{lll}
 \norm{g(\theta; D_{\tiln}, \tildelta) - \grad \mathcal{R}(\theta)}{2} &\leq& C\sqrt{\frac{\text{tr}(\text{Cov}(\nabla \poploss (\theta)))\log{1/\tildelta}}{\tiln}} \vspace{0.1in}\\
 &\leq& \underbrace{C\sqrt{\frac{\norm{\grad^2 A(\tparam)}{2} p\log{1/\tildelta}}{\tiln}}}_{\beta(\tiln,\tildelta)}.
 \end{array}
 \end{equation*}
Since $\alpha(\tiln,\tildelta) = 0$, the stability condition is always satisfied, as long as $\tau_l > 0$. Substituting $\beta(\tiln,\tildelta)$ into Theorem~\ref{thm:main} gives us the required result.
\section{Upper bound on Contamination Level}
 We  provide a complementary result, which gives an upper bound for the contamination level $\epsilon$ based on the initialization point $\theta^0$, above which, Algorithm~1 would not work.  The key idea is that the error incurred by any mean estimation oracle is lower bounded by the variance of the distribution, and that if the zero vector lies within that error ball, then any mean oracle can be forced to output $\mathbf{0}$ as the mean. For Algorithm~1, this implies that, in estimating the mean of the gradient, if the error is high, then one can force the mean to be $\mathbf{0}$ which forces the algorithm to converge. For the remainder of the section we consider the case of linear regression with $x \sim \calN(0,\calI_p)$ in the asymptotic regime of $n \rightarrow \infty$.

  \begin{lemma}\label{lem:lowerBoundGradient}
 Consider the model in equation\eqref{eq:linregress} with $x \sim \calN(0,\calI_p)$ and $w \sim \calN(0,1)$, then there exists a universal constant $C_1$ such that if $\epsilon > C_1 \frac{\norm{\theta^0 - \tparam}{2}}{\sqrt{1+2 \norm{\theta^0 - \tparam}{2}^2}}$, then for every gradient oracle, there exists a contamination distribution $Q$ such that, Algorithm~1 will converge to $\theta^0$ even when the number of samples $n \rightarrow \infty$.
 \end{lemma}
\begin{proof}
Using Lemma~\ref{lem:regressionDist_identity}, we know that for any point $\theta$,
\begin{align*}
\grad \poploss(\theta) = xx^T \Delta - x.w \\
\Exp_{\tparam} [\grad \poploss(\theta) ] =  (\theta - \tparam) =  \Delta \\
\norm{\cov(\grad \poploss(\theta)}{2} = 1 + 2 \norm{\Delta}{2}^2,
\end{align*}
where $\Delta = \theta - \tparam$. \\
Let $P_{\grad \poploss(\theta)}$ represent the distribution $\grad \poploss \theta$. Similarly, let $P_{\epsilon, \grad \poploss (\theta), Q}$ represent the corresponding $\epsilon$-contaminated distribution. Then, using Theorem 2.1 \cite{chen2015robust}, we know that the minimax rate for estimating the mean of the distribution of gradients is given by:
\[ \inf_{\widehat{\mu}} \sup_{\theta \in \real^p, Q } P_{\epsilon, \grad \poploss(\theta), Q} \left\{ \norm{\widehat{\mu} - \Exp_{\tparam} [ \grad \poploss(\theta) ]}{2}^2 \geq  C \epsilon^2  (1 + 2 \norm{\Delta}{2}^2) \right\} \geq c .\]

The above statement says that at any point $\theta$, any mean oracle $\Psi$ will always incur an error of $ \Omega(\sqrt{C \epsilon^2  (1 + 2 \norm{\Delta}{2}^2)}) $ in estimating the gradient $\Exp_{\tparam} [ \grad \poploss(\theta)]$.
\[ \norm{\Psi (\theta) - \Exp_{\tparam} [ \grad \poploss(\theta) ]}{2} \geq C \epsilon  \sqrt{(1 + 2 \norm{\Delta}{2}^2)}~~~  \forall ~~ \Psi \]
For any oracle $\Psi$, there exists some adversarial contamination $Q$, such that whenever $\norm{\Exp_{\tparam} [ \grad \poploss(\theta) ]}{2} < C \epsilon  \sqrt{(1 + 2 \norm{\Delta}{2}^2)}$, then $\norm{\Psi(\theta)}{2} = 0$.  \\ \\

Suppose that the contamination level $\epsilon$ is such that,
\[ \epsilon > \frac{1}{C} \frac{\norm{\Exp_{\tparam} [ \grad \poploss(\theta^0) ]}{2}}{\sqrt{(1 + 2 \norm{\theta^0 - \tparam}{2}^2)} },  \]
then for every oracle there exists a corresponding $Q$ such that Algorithm~1 will remain stuck at $\theta^0$.

Plugging $\Exp_{\tparam} [ \grad \poploss(\theta^0) ]= \theta^0 - \tparam$, we recover the statement of the lemma.
\end{proof}
 \citet{chen2015robust} provide a general minimax lower bound of $\Omega(\epsilon)$ for  $\epsilon$-contamination models in this setting. In contrast, using Algorithm 1 with \citep{lai2016agnostic} as oracle, we can only $O(\sqrt{\epsilon \log p})$ close to the true parameter even when the contamination is small, which implies that our procedure is not minimax optimal. Our approach is nonetheless the only practical algorithm for robust estimation of general statistical models.

\section{Proof of Lemma~\ref{lem:lai2016agnostic}}\label{sec:proof:lem:lai2016agnostic}
In this section we present a refined, non-asymptotic analysis of the robust mean estimator of~\citep{lai2016agnostic}, described in Algorithm~\ref{algo:huber_mean_estimation}. We begin by introducing some preliminaries. We subsequently analyze the algorithm
in 1-dimension and finally turn our attention to the general algorithm.

\subsection{Preliminaries}
Unless otherwise stated, we assume throughout that the random variable $X$ has bounded fourth moments, i.e. for every unit vector $v$,
\begin{align*}
\mathbb{E} \left[ \inprod{X - \mu}{v}^4  \right] \leq C_4 \left[ \mathbb{E}\left[\inprod{X-\mu}{v}^2\right] \right]^2.
\end{align*}
We summarize some useful results from \citep{lai2016agnostic}, which bound the deviation of the conditional mean/covariance from the true mean/covariance. 
\begin{lemma}\label{lem:meanshift}[Lemma 3.11 \citep{lai2016agnostic}]
Let $X$ be a univariate random variable with bounded fourth moments, and let $A$ be any with event with probability $\mathbb{P}(A) = 1 - \smconst \geq \half$. Then,
\begin{align*}
\left| \mathbb{E}(X|A) - \mathbb{E}(X) \right| \leq \sigma \sqrt[4]{8C_4 \smconst^3}.
\end{align*}
\end{lemma}

\begin{lemma}\label{lem:covarianceAuxillary}[Lemma 3.12 \citep{lai2016agnostic}]
Let $X$ be a univariate  random variable with $\mathbb{E}[X] = \mu$, $\mathbb{E} \left( (X - \mu)^2 \right) = \sigma^2$ and let $\mathbb{E}( (X - \mu)^4) \leq C_4 \sigma^4$. Let $A$ be any with event with probability $\mathbb{P}(A) = 1 - \smconst \geq \half$. Then,
\begin{align*} 
(1 - \sqrt{C_4 \smconst}) \sigma^2 \leq \mathbb{E} ((X - \mu)^2 | A ) \leq ( 1 + 2 \smconst) \sigma^2. 
\end{align*}
\end{lemma}

\begin{corollary}\label{cor:covarianceShift}[Corollary 3.13 \citep{lai2016agnostic}]
Let $A$ be any event with probability $\mathbb{P}(A) = 1 - \smconst \geq \half$, 
and let $X$ be a random variable with bounded fourth moments. We denote 
$\Sigma|_{A} = \mathbb{E} (XX^T | A) - (\mathbb{E}(X|A)) (\mathbb{E}(X|A))^T$ to be the conditional covariance matrix. We have that,
\begin{align*} 
(1 - \sqrt{C_4 \smconst} - \sqrt{ 8  C_4 \smconst^3}) \Sigma  \preceq \Sigma|_{A} \preceq ( 1 + 2 \smconst) \Sigma.
\end{align*}
\end{corollary}
\noindent For random variables with bounded fourth moments we can use Chebyshev's inequality to obtain tail bounds.
\begin{lemma}\label{lem:tail}[Lemma 3.14 \citep{lai2016agnostic}] Let $X$ have bounded fourth moments, then for every unit vector $v$ we have that,
\begin{align*}
\mathbb{P}(|\inprod{X}{v} - \mathbb{E}[\inprod{X}{v}]| \geq t \sqrt{\left[ \mathbb{E}\left[\inprod{X-\mu}{v}^2\right] \right]}) \leq \frac{C_4}{t^4}.
\end{align*}
\end{lemma}

\noindent Our proofs also use the matrix Bernstein inequality for rectangular matrices. As a preliminary, we consider a finite sequence $\{Z_k\}$ of independent, random matrices of size $d_1 \times d_2$. We assume that each random matrix satisfies $\mathbb{E}(Z_k) = 0$, and $ \norm{Z_k}{\op} \leq R$ almost surely. 
We define:
\begin{align*}
\sigma^2 := \max \left\{ \norm{\sum \limits_k \mathbb{E}(Z_k Z_k^T)} {\op}, \norm{\sum \limits_k \mathbb{E}(Z_k Z_k^T)}{\op}\right\}.
\end{align*}
With these preliminaries in place we use the following result from~\citep{tropp2012user}.
\begin{lemma}\label{lem:bernstein}
For all $t \geq 0$,
\begin{align*}
\mathbb{P} \Big( \Big\|\sum \limits_k Z_k\Big\|_{\text{op}} \geq t \Big) \leq (d_1 + d_2) \exp \left( \frac{-t^2/2}{\sigma^2 + Rt/3} \right).
\end{align*}
Equivalently, with probability at least $1 - \delta$, 
\begin{align*} \Big\|\sum \limits_k Z_k\Big\|_{\text{op}}  \leq \sqrt{2 \sigma^2 \log \left( \frac{d_1 + d_2}{\delta } \right)} +  \frac{2R}{3} \log \left( \frac{d_1 + d_2}{\delta} \right).
\end{align*}
\end{lemma}
\noindent We let $\mathcal{I}$ denote the set of all intervals in $\mathbb{R}$. The following is a standard uniform convergence result.
\begin{lemma}
\label{lem:VC}
Suppose $X_1,\ldots,X_n \sim \mathbb{P}$, then 
with probability at least $1 - \delta$,
\begin{align*}
\sup_{I \in \mathcal{I}} \left| \mathbb{P}(I) - \frac{1}{n} \sum_{i=1}^n \mathbb{I}(X_i \in I) \right| \leq 2 \sqrt{ \frac{ 4 \log (en) + 2 \log (2 /\delta) }{ n} }.
\end{align*}
\end{lemma}

\begin{algorithm}[H]
\centering
\caption{Huber Outlier Gradients Truncation}
        \begin{algorithmic}
\Function{HuberOutlierGradientTruncation(Sample Gradients $S$, Corruption Level $\epsilon$, Dimension $p$,$\delta$)}{}
\If{p=1} 
\State Let $[a,b]$ be smallest interval containing $\left(1 - \epsilon - C_5 \left(
\sqrt{\frac{1}{|S|} \log \left( \frac{|S|}{\delta} \right)} \right) \right)(1 - \epsilon)$ fraction of points.
\State $\widetilde{S} \leftarrow S \cap [a,b]$.
\State \Return $\widetilde{S}$
\Else
\State Let $[S]_i$ be the samples with the $i^{th}$ co-ordinates only,  $[S]_i = \{ \left\langle x , e_i \right\rangle | x \in S \}$
\For{$i=1$ to $p$} 
\State $a[i] = \text{{\sc HuberGradientEstimator}}([S]_i,\epsilon,1,\delta/p)$.
\EndFor
\State Let $B(r, a)$ be the ball of smallest radius centered at $a$  containing \small{$(1 - \epsilon - C_p \left(
\sqrt{\frac{p}{|S|} \log \left( \frac{|S|}{p \delta} \right)} \right) (1 - \epsilon)$} \normalsize { fraction of points in $S$.
\State $\widetilde{S} \leftarrow S \cap B(r,a)$.
\State \Return $\widetilde{S}$}
\EndIf
\EndFunction
\end{algorithmic}
\label{algo:huber_outlier_truncation}
\end{algorithm}

\noindent We now turn our attention to an analysis of Algorithm~\ref{algo:huber_mean_estimation} for the 1-dimensional case.
\subsection{The case when $p=1$}
Firstly, we analyze Algorithm~\ref{algo:huber_mean_estimation} when $p=1$.
\begin{lemma}\label{lem:vempala1D}
Suppose that, $P_\tparam$ is a distribution on $\real^1$ with mean $\mu$, variance
$\sigma^2$, and bounded fourth moments. There exist positive universal constants $C_1,C_2,C_8$ > 0, such that given $n$ samples from the distribution in \eqref{eqn:mixture}, the algorithm with  probability at least $1-\delta$, returns an estimate $\emu$ such that,
\[ \norm{\emu - \mu}{2} \leq C_1 C_4^{\frac{1}{4}}  \sigma \left(\epsilon +  \sqrt{\frac{\log 3 / \delta}{2n}} + t  \right)^{\frac{3}{4}} + C_2 \sigma \left(\epsilon +  \sqrt{\frac{\log 3 / \delta}{2n}} + t  \right)^{\half} \sqrt{\frac{\log(3 / \delta) }{{n}}} \]
where $t = C_8 \sqrt{\frac{1}{n} \log \left( \frac{n}{\delta} \right)} $.
which can be further simplified to,
\[ \norm{\emu - \mu}{2} \leq C_1 C_4^{\frac{1}{4}} \sigma \left(\epsilon +  C_8 \sqrt{\frac{1}{n} \log \left( \frac{n}{\delta} \right)} \right)^{\frac{3}{4}} + C_2 \sigma \left(\epsilon +  C_8 \sqrt{\frac{1}{n} \log \left( \frac{n}{\delta} \right)} \right)^{\half} \sqrt{\frac{\log(1 / \delta) }{{n}}} \]
\end{lemma}

\begin{proof}
By an application of Hoeffding's inequality we obtain that 
with probability at least $ 1 - \delta/3$, 
the fraction of corrupted samples (i.e. samples from the distribution $Q$) is less than $\epsilon + \sqrt{\frac{ \log (3/\delta)}{2n}}$. We condition on this event through the remainder of this proof. We let $\eta$ denote the fraction of corrupted samples.
Further, we let $\truesamp$ be the samples from the true distribution. Let $\ngood$ be the cardinality of this set, i.e. $\ngood := |\truesamp|$.

Let $I_{1 - \eta }$ be the interval around $\mu$ containing $1 - \eta$ mass of $P_\tparam$. Then, using Lemma~\ref{lem:tail}, we have that:
\begin{align*}
\len{I_{1 - \eta }} \leq  \frac{C_4^{\frac{1}{4}} \sigma}{\eta^{\frac{1}{4}}}.
\end{align*}
Using Lemma~\ref{lem:VC} we obtain that with probability at least $1 - \delta/3$ the number of samples from the distribution $P$ that fall in the interval $I_{1-\eta}$ is at least $1 - \eta - t$ where $t$ is upper bounded as: 
\begin{align*}
t \leq 2 \sqrt{ \frac{ 4 \log (en) + 2 \log (6 /\delta) }{ n} }.
\end{align*}
Now we let $\widetilde{S}$ be the set of points in the smallest interval containing $(1-\eta - t)(1-\eta)$ fraction of all the points. 

\begin{itemize}
\item Using VC theory, we know that for every interval $I \subset \real$, there exists some universal constant $C_3$ such that
\begin{align}\label{eqn:VC_interval}
\P  \left( \left| \left( P(x \in I | x \sim D) - P(x \in I | x \in_u S_D) \right) \right|  > t/2 \right)  \leq n_D^2  \exp( - n_D t^2 / 8)
\end{align}
This can be re-written as, that with probability at least $(1-\delta/3)$,  there exists a universal constant $C_0$ such that, 
\[ \sup_{I } \left| \left( P(x \in I | x \sim D) - P(x \in I | x \in_u S_D) \right) \right|  \leq C_0 \sqrt{\frac{1}{n_D} \log \left( \frac{n_D}{\delta} \right)} \leq \underbrace{C_5 \sqrt{\frac{1}{n} \log \left( \frac{n}{\delta} \right)}}_{t} \]
\item Using Equation~\eqref{eqn:VC_interval}, we know that $(1 - \eta - t)$ fraction of $S_D$ lie in $\calI_{1 - \eta}$. 

Let $\widetilde{S}$ be the set of points in the smallest interval containing $(1-\eta - t)(1-\eta)$ fraction of the points. 

\item We know that the length of minimum interval containing $(1 - \eta - t)(1 - \eta)$ fraction of the points of $S$ is less than length of smallest interval  containing $(1 - \eta - t)$ fraction of points of $S_D$, which in turn is less than length of $I_{1 - \eta }$.
\item Now, $I_{1 - \eta }$ and minimum interval containing  $(1 - \eta - t)$ fraction of points of $S_D$ need to overlap. This is because, $n$ is large enough such that $t < \half - \eta$ hence, the extreme points for such an interval can be atmost $2 \len {I_{1 - \eta}}$ away. 
\item Hence, the distance of all chosen noise-points from $\mu$ will be within the $ \len {I_{1 -\eta}}$.
\item Moreover, the interval of minimum length with $(1 - \eta - t) (1 - \eta)$ fraction of $S$ will contain at least $1 - 3\eta - t$ fraction of $S_D$. 
\item Hence, we can bound the error of $mean(\widetilde{S})$ by controlling the sources of error.
\begin{itemize}
\item All chosen noise points are within $\len {I_{1 -\eta}}$, and there are atmost $\eta$ of them, hence the maximum error can be $\eta  \len {I_{1 -\eta}}$.
\item Next, the mean of chosen good points will converge to the mean of the conditional distribution. \ie points sampled from $D$ but conditioned to lie in the minimum length interval. The variance of these random variables is upper bounded using Lemma~\ref{lem:covarianceAuxillary}.
\item To control the distance between the mean($E(X)$ and the conditional mean($E(X|A)$), where $A$ is the event that a sample $x$ is in the chosen interval. We know that $P(A) \geq 1 - 3\eta - t$, hence, using Lemma~3.11\citep{lai2016agnostic}, we get that there exists a constant $C_{13}$ such that,
\[ | E[X] - E[X|A] | \leq C_{13} C_4^{\frac{1}{4}} \sigma (\eta + t)^{\frac{3}{4}} \]
\end{itemize}
\item Hence, with probability at least $1 - \delta/3$,  the mean of $\widetilde{S}$ will be within
\[ \eta \times \len {I_{1 - \eta}}  + C_{13} C_4^{\frac{1}{4}} \sigma (\eta + t)^{\frac{3}{4}} + C_6 \sigma (1 + 2 \eta)^{\half} \sqrt{\frac{\log(3/\delta)}{n}}  \]

\item Taking union-bound over all conditioning statements, and upper bounding, $\eta$ with  $\epsilon + \sqrt{\frac{ \log (3/\delta)}{2n}}$, we recover the statement of the lemma.
\end{itemize}
\end{proof}

\subsection{The case when $p > 1$}
To prove the case for $p>1$, we use a series of lemmas. Lemma~\ref{lem:vempala-pD} proves that the outlier filtering constrains the points in a ball around the true mean. Lemma~\ref{lem:controlVar} controls the error in the mean and covariance the true distribution after outlier filtering ($\widetilde{D}$). Lemma~\ref{lem:bottomHalf} controls the error for the mean of $\widetilde{S}$ when projected onto the bottom span of the covariance matrix $\Sigma_{\widetilde{S}}$. 

\begin{lemma}\label{lem:vempala-pD}
Suppose that, $P_\tparam$ is a distribution on $\real^p$ with mean $\mu$, covariance
$\Sigma$, and bounded fourth moments. There exist positive universal constants $C_1,C_2,C_8$ > 0, such that given $n$ samples from the distribution in Equation~\eqref{eqn:mixture}, we can find a vector $a \in \real^p$ such that with probability at least $1 - \delta$, 
 \begin{align*} 
 \norm{a - \mu}{2} \leq & C_1 C_4^{\frac{1}{4}} \sqrt{\trace{\Sigma}} \left(\epsilon +  C_8 \sqrt{\frac{1}{n} \log \left( \frac{np}{\delta} \right)} \right)^{\frac{3}{4}} \\
 & + C_2  \left(\epsilon +  C_8 \sqrt{\frac{1}{n} \log \left( \frac{np}{\delta} \right)} \right)^{\half} \sqrt{\trace{\Sigma}} \sqrt{\frac{\log(p / \delta)}{{n}}}
 \end{align*}
 \end{lemma}
 \begin{proof}
 Pick $n$ orthogonal directions $v_1,v_2,\ldots,v_n$, and use method for one-dimensions, and using union bound, we can recover the result.
 \end{proof}
 
 Next, we prove the case when $p > 1$. Firstly, we prove that after the outlier step, 
 
 \begin{lemma}\label{lem:outlier}
 After the outlier removal step, there exists universal constants $C_{11} > 0$ such that with probability at least $1 - \delta$, every remaining point $x$ satisfies,
 \[  \norm{x - \mu}{2}  \leq r_1^* + 2r_2^*  \]
 where $r_1^* = C_{10}\frac{C_4^{\frac{1}{4}} \sqrt{p \norm{\Sigma}{2}}}{\eta^\frac{1}{4}}$ and $r_2^* = C_1C_4^{\frac{1}{4}} \sqrt{p \norm{\Sigma}{2}} (\eta + t)^{\frac{3}{4}} + C_2 \sqrt{p \norm{\Sigma}{2}} (\eta + t)^{ \frac{1}{2} } \sqrt{\frac{\log (1 / \delta)}{n}}$ and $t = C_8 \sqrt{\frac{1}{n} \log \left( \frac{np}{\delta} \right)}$. Here $\eta \leq \epsilon + \sqrt{\frac{\log(1/\delta)}{2n}}$ is the fraction of samples corrupted.
\end{lemma}

\noindent \begin{proof}

\begin{itemize}
\item Let $\widetilde{S}$ be the set of points chosen after the outlier filtering. Let $\widetilde{S}_D$ be set of good points chosen after the outlier filtering. Let $\widetilde{S}_N$ be the set of bad points chosen after the outlier filtering.
\item Using VC theory we know that for every closed ball $\calB(\mu,r) = \{ x | \norm{x - \mu}{2} \leq r \}$, there exists a constant $C_9$ such that with probability at least $1 - \delta$
\[ \sup \limits_{B}  \left|  P( x \in B | x \sim D) - P( x \in B | x \in_u S_D) \right|  \leq  \underbrace{C_9  \sqrt{\frac{p}{n} \log \left( \frac{n}{p \delta} \right)}}_{t_2}\] 

\item Let $B^* = B(\mu,r^*)$ for $r_1^* = C_{10} \frac{C_4^{\frac{1}{4}}}{(\eta )^{\frac{1}{4}}} \sqrt{p \norm{\Sigma}{2}}$. Then, we claim that 
 \[ P(x \in B^* | x \sim D) \geq 1 - \eta \]
 \begin{itemize}
 \item To see this, suppose we have some $x \in D$. Let $z = x - \mu$. Let $z_i = z^Tv_i$ for some orthogonal directions $v_1,v_2,\ldots,v_p$. Let $Z^2 = \sum z_i^2 = \norm{z}{2}^2$.
 \item \begin{align*}
 P \left(Z^2 \geq \frac{C_4^\half p \norm{\Sigma}{2}}{(\eta )^\half} \right) = P \left( Z^4 \geq \frac{C_4 p^2 \norm{\Sigma}{2}^2}{(\eta )} \right) \leq \frac{(\eta  ) E(Z^4)}{C_4 p^2 \norm{\Sigma}{2}^2}
 \end{align*}
 \item Now, $E(Z^4) \leq p^2 \max_i E(z_i^4) \leq C_4 p^2 \norm{\Sigma}{2}^2$. Plugging this in the above, we have that $P(x \in B^* | x \sim D) \geq 1 - \eta$.
 \end{itemize}
 \item Hence, we have that $P(x \in B^* | x \in_u S_D) \geq 1 - \eta - t_2$.
 \item Using Lemma~\ref{lem:vempala-pD}, we have that at least $(1 - \eta - t_2)$ fraction of good points are $r_1^* + r_2^*$ away from $a$. Hence, we have that the minimum radius of the ball containing all the $(1 - \eta - t_2) (1 - \eta)$ has a radius of atmost $r_1^* + r_2^*$, which when combined with the triangle inequality recovers the statement of lemma.
\end{itemize}
\end{proof}

As before, let $\tilS$ be the set of points after outlier filtering. Let $\mu_{\tilS} = mean(\tilS)$, $\mu_{\tilS_D} = mean(\tilS_D)$, $\mu_{\tilS_N} = mean(\tilS_N)$.
\begin{lemma}\label{lem:controlVar}
Let $\tilS_D$ be the set of clean points remaining after the outlier filtering. Then, with probability at least $1 - \delta$, we have that
\begin{align*}
\norm{\mu_{\tilS_D} - \mu}{2} \leq C_1 C_4^{\frac{1}{4}} (\eta + t_2)^{\frac{3}{4}} \sqrt{\norm{\Sigma}{2}} \left( 1 + \frac{\log(p/\delta)}{n} \right) & + \sqrt{\norm{\Sigma}{2}} \sqrt{(1 + 2(\eta + t_2)} \sqrt{\frac{1}{n} \log (p /\delta)} \\
& + C_{15} \frac{(r_1^*+2r_2^*)}{n} \log (p /\delta)
\end{align*}
and
\[  \norm{\Sigma_{\tilS_D}}{2} \leq   \beta(n,\delta) \norm{\Sigma}{2},    \]
where 
\[ \beta(n,\delta) = \left(1 + 2C(\eta + t_2) +  \left(1 + \frac{p\sqrt{C_4} }{\sqrt{\eta}} + \sqrt{C_4}p (\eta + t)^{\frac{3}{2}} + (\eta + t_2)^{\frac{3}{2}} \right) \left( \sqrt{\frac{\log(p/\delta}{n}} + \frac{\log (p/\delta)}{n} \right) \right) \]

\end{lemma}
\begin{proof}
We first prove the bounds on the mean shift.
\begin{align*}
\norm{\mu_{\tilS_D} - \mu}{2} \leq \underbrace{\norm{\mu_{\tilS_D} - \mu_{\widetilde{D}}}{2}}_{A} + \underbrace{\norm{\mu_{\widetilde{D}} - \mu}{2}}_{B} 
\end{align*}
\begin{itemize}
\item  \textbf{Control of B.} We use Lemma~\ref{lem:meanshift} on $X = x^T \frac{\mu_{\widetilde{D}} - \mu}{\norm{\mu_{\widetilde{D}} - \mu}{2}}$ for $x \sim D$, and $A$ be the event that $x$ is not removed by the outlier filtering.
\begin{align}
\norm{\mu_{\widetilde{D}} - \mu}{2} \leq C_1 C_4^{\frac{1}{4}} (\eta + t_2)^{\frac{3}{4}} \sqrt{\norm{\Sigma}{2}}
\end{align}
\item \textbf{Control of A}.  Using Lemma~\ref{lem:covarianceAuxillary},we have that $\norm{\Sigma_{\widetilde{D}}}{2} \leq (1 + 2(\eta + t_2)) \norm{\Sigma}{2}$. Now, we use Bernstein's inequality . Lemma~\ref{lem:bernstein} with $R = C(r_1^*+2r_2^* + B)$, we get that, with probability at least $1 - \delta$,
\begin{align}
\norm{\mu_{\tilS_D} - \mu_{\widetilde{D}}}{2} \leq C_{14} \sqrt{\norm{\Sigma}{2}} \sqrt{(1 + 2(\eta + t_2)} \sqrt{\frac{1}{n} \log (p /\delta)} + C_{15} \frac{(r_1^*+2r_2^* + B)}{n} \log (p /\delta)
\end{align}
\end{itemize}
Next, we prove the bound for covariance matrix.
\begin{align}
\norm{\Sigma_{\tilS_D}}{2} \leq \norm{\Sigma_{\tilS_D} - \Sigma_{\widetilde{D}}}{2} + \underbrace{\norm{\Sigma_{\widetilde{D}} - \Sigma}{2}}_{ \leq 2C (\eta + t_2) \norm{\Sigma}{2} (\text{By Corollary~\ref{cor:covarianceShift})})} + \norm{\Sigma}{2}
\end{align}
To control $\norm{\Sigma_{\tilS_D} - \Sigma_{\widetilde{D}}}{2}$, we use Bernstein's inequality, with $Z_k = \frac{(x_k - \mu_{\widetilde{D}})(x_k - \mu_{\widetilde{D}})^T - \Sigma_{\widetilde{D}}}{n}$. From, Lemma~\ref{lem:outlier}, we know that the points are constrained in a ball. Plugging this into Lemma~\ref{lem:bernstein},
\begin{align*}
\norm{\Sigma_{\tilS_D} - \Sigma_{\widetilde{D}}}{2} \leq C (\norm{\Sigma}{2} + R^2) \left( \sqrt{\frac{\log(p/\delta}{n}} + \frac{\log (p/\delta)}{n} \right)
\end{align*}
where $R^2 = C \left( r_1^{*^2} + r_2^{*^2} +B^2 \right)$.
\end{proof}
Plugging in the values, we get that,
\[ \norm{\Sigma_{\tilS_D} - \Sigma_{\widetilde{D}}}{2} \leq C \norm{\Sigma}{2} \left(1 + \frac{p\sqrt{C_4} }{\sqrt{\eta}} + \sqrt{C_4}p (\eta + t)^{\frac{3}{2}} + (\eta + t_2)^{\frac{3}{2}} \right) \left( \sqrt{\frac{\log(p/\delta}{n}} + \frac{\log (p/\delta)}{n} \right)
  \]
Finally, we have that,
\[ \norm{\Sigma_{\tilS_D}}{2} \leq \norm{\Sigma}{2} \underbrace{\left(1 + 2C(\eta + t_2) +  \left(1 + \frac{p\sqrt{C_4} }{\sqrt{\eta}} + \sqrt{C_4}p (\eta + t)^{\frac{3}{2}} + (\eta + t_2)^{\frac{3}{2}} \right) \left( \sqrt{\frac{\log(p/\delta}{n}} + \frac{\log (p/\delta)}{n} \right) \right)}_{\beta(n,\delta)} \]
\begin{lemma}\label{lem:bottomHalf}
Let $W$ be the bottom $p/2$ principal components of the covariance matrix after filtering $\Sigma_{\tilS}$. Then there exists a universal constant $C>0$ such that with probability at least $1 - \delta$, we have that
\[ \norm{ \eta P_W \delta_\mu }{2}^2 \leq C \eta \left( \beta(n,\delta) +  \gamma(n,\delta) C_4^{\half}) \norm{\Sigma}{2} \right), \]
\end{lemma}
where $\delta_\mu  = \mu_{\tilS_N} - \mu_{\tilS_D}$, $P_W$ is the projection matrix on the bottom $p/2$-span of $\Sigma_{\tilS}$, $\beta(n,\delta)$ is as defined in Lemma~\ref{lem:controlVar} and $\gamma(n,\delta) = \left( \eta^{\half} + (\eta + t)^{5/2} + \eta(\eta + t) \frac{\log (1/\delta)}{n} \right) $
\begin{proof}
We have
\begin{align}
\Sigma_{\tilS }= \underbrace{(1 - \eta) \Sigma_{\tilS_D}}_{E} + \underbrace{\eta\Sigma_{\tilS_N} + (\eta - \eta^2) \delta_{\mu} \delta_{\mu^T }}_{F} \\
\end{align}
By Weyl's inequality we have that, 
\[ \lambda_{p/2}(\Sigma_{\tilS })  \leq \lambda_1(E) + \lambda_{p/2}(F)\]
\begin{itemize}
\item \textbf{Control of $\lambda_{p/2}(F)$.} 
\begin{align*}
\lambda_{p/2}(F) & \leq \frac{\trace{F}}{p/2} \\
& \leq C_{15} \eta \frac{((r_1^*)^2 + (r_2^*)^2)+B^2}{p/2} \\ 
& \leq C_{16} C_4^{\half} \norm{\Sigma}{2} \underbrace{\left( \eta^{\half} + (\eta + t)^{5/2} + \eta(\eta + t) \frac{\log (1/\delta)}{n} \right)}_{\gamma(n,\delta)}
\end{align*}
where $t = C_8 \sqrt{\frac{1}{n} \log \left( \frac{np}{\delta} \right)}$.

\item \textbf{Control of $\lambda_{1}(E)$.} 
\begin{align*}
\lambda_1(E) \leq (1-\eta) \beta \norm{\Sigma}{2}
\end{align*}
\end{itemize}
Hence, we have that:
\[ \lambda_{p/2}(\Sigma_{\tilS })  \leq (1 - \eta) \beta \norm{\Sigma}{2} +  C_{16} \gamma \sqrt{C_4} \norm{\Sigma}{2} \]
Using that $W$ is the space spanned by the bottom $p/2$ eigenvectors of $\Sigma_{\tilS}$ and $P_W$ is corresponding projection operator, we have that:
\[ P_W^T \Sigma_{\tilS} P_W \preceq \left[ (1 - \eta) \beta + C_{16} \gamma \sqrt{C_4} \right] \norm{\Sigma}{2} I_{p} \]
Following some algebraic manipulation in \citep{lai2016agnostic}, we get that,
\[  \norm{ \eta P_W \delta_\mu }{2}^2 \leq \eta \left( (\beta(n,\delta) +  \gamma C_4^{\half}) \norm{\Sigma}{2}  \right) \]
\end{proof}
Having established all required results, we are now ready to prove Lemma~\ref{lem:lai2016agnostic}. We first present a result for general mean estimation. The proof of Lemma~\ref{lem:lai2016agnostic} then follows directly from this result.
\begin{lemma}
Suppose that, $P_\tparam$ is a distribution on $\real^p$ with mean $\mu$, covariance $\Sigma$ and bounded fourth moments. There exist positive universal constant $C$ > 0, such that given $n$ samples from the distribution in Equation~\eqref{eqn:mixture}, the algorithm with  probability at least $1-\delta$, returns an estimate $\emu$ such that,
\begin{align*}
\norm{\emu - \mu}{2} \leq C \norm{\Sigma}{2}^\half (1 +\sqrt{\log p}) \left( \sqrt{\eta} + C_4^{\frac{1}{4}}(\eta + t_2)^{\frac{3}{4}} +   \left( \sqrt{\eta} p C_4^{\half} \sqrt{\frac{\log p \log(p \log (p/\delta))}{n}} \right)^\half \right)
\end{align*}
where $\eta = \epsilon + \sqrt{\frac{\log(p) \log (\log p / \delta)}{2n}}$ and $t_2 = \sqrt{\frac{p \log (p) \log (n/(p \delta))}{n}}$.
\end{lemma}

\begin{proof}
We divide $n$ samples into $\floor{\log(p)}$ different sets. We choose the first set and keep that as our active set of samples. We run our outlier filtering on this set, and let the remaining samples after the outlier filtering be $\tilS_D$. By orthogonality of subspaces spanned by eigenvectors, coupled with triangle inequality and contraction of projection operators, we have that
\[ \norm{\emu - \mu}{2}^2 \leq 2\norm{P_W (\emu - \mu_{\tilS_D})}{2}^2 + 2\norm{P_W (\mu_{\tilS_D} - \mu)}{2}^2 + \norm{\emu_V - P_V \mu}{2}^2  \]
\[ \norm{\emu - \mu}{2}^2 \leq 2\norm{ P_W(\emu - \mu_{\tilS_D})}{2}^2 + 2\norm{(\mu_{\tilS_D} - \mu)}{2}^2 + \norm{\emu_V - P_V \mu}{2}^2  \]
where $V$ is the span of the top $p/2$ principal components of $\Sigma_{\tilS}$ and where $\emu_V$ is the mean vector of returned by the running the algorithm on the reduced dimensions $dim(V ) = p/2$. From Lemma~\ref{lem:bottomHalf}, both $\beta(n,\delta)$ and $\gamma(n,\delta)$ are monotonically increasing in the dimension; moreover the upper bound in Lemma~\ref{lem:controlVar} is also monotonically increasing in the dimension $p$, hence, the error at each step of the algorithm can be upper bounded by error incurred when running on dimension $p$, with $n/\log(p)$ samples, and probability of $\delta / \log p$. Hence, the overall error for the recursive algorithm can be upper bounded as,
\begin{align*}
\norm{\emu - \mu}{2}^2 \leq \left(2\norm{P_W(\emu - \mu_{\tilS_D})}{2}^2 + 2\norm{\mu_{\tilS_D} - \mu}{2}^2 \right) (1 + \log p) 
\end{align*}

Combining Lemma~\ref{lem:controlVar} and Lemma~\ref{lem:bottomHalf} which are instantiated for $ n/\log p$ samples and probability $\delta / \log p$, we get,
\[ \norm{\emu - \mu}{2} \leq C \norm{\Sigma}{2}^\half \sqrt{\log p} \left( \sqrt{\eta} + C_4^{\frac{1}{4}}(\eta + t_2)^{\frac{3}{4}} +   \left( \sqrt{\eta} p C_4^{\half} \sqrt{\frac{\log p \log(p \log (p/\delta))}{n}} \right)^\half \right) \]
\end{proof}